\documentclass[dvipsnames]{article}

     \usepackage{iclr2021_conference,times}

\usepackage{todonotes}
\usepackage[utf8]{inputenc} 
\usepackage[T1]{fontenc}    
\usepackage{hyperref}       
\usepackage{url}            
\usepackage{booktabs}       
\usepackage{amsfonts}       
\usepackage{nicefrac}       
\usepackage{microtype}      
\usepackage{amsmath,amsthm,amssymb}
\usepackage{multirow}
\usepackage{bm}
\usepackage{subfig}

\usepackage{wrapfig}
\usepackage{graphics}

\usepackage{amsmath,amsfonts,bm}

\def\eqref#1{equation~\ref{#1}}

\def\1{\bm{1}}

\DeclareMathAlphabet{\mathsfit}{\encodingdefault}{\sfdefault}{m}{sl}
\SetMathAlphabet{\mathsfit}{bold}{\encodingdefault}{\sfdefault}{bx}{n}

\usepackage{hyperref}
\hypersetup{colorlinks,linkcolor={blue},citecolor={blue},urlcolor={green}}
\usepackage{url}

\newcommand{\wg}[1]{{\color{black} #1}}
\newcommand{\wgr}[1]{{\color{black} #1}}

\newtheorem{definition}{Definition}[section]
\newtheorem{theorem}{Theorem}
\newtheorem{prop}{Proposition}
\newtheorem{corollary}{Corollary}[theorem]
\newtheorem{lemma}{Lemma}

\usepackage{tikz}
\newcommand*\circled[1]{\tikz[baseline=(char.base)]{
            \node[shape=circle,draw,inner sep=2pt] (char) {#1};}}

\theoremstyle{remark}
\newtheorem{rem}{Remark}

\usepackage{scalerel,stackengine}
\stackMath
\newcommand\reallywidehat[1]{%
\savestack{\tmpbox}{\stretchto{%
  \scaleto{%
    \scalerel*[\widthof{\ensuremath{#1}}]{\kern-.6pt\bigwedge\kern-.6pt}%
    {\rule[-\textheight/2]{1ex}{\textheight}}
  }{\textheight}%
}{0.5ex}}%
\stackon[1pt]{#1}{\tmpbox}%
}
\usepackage[ruled,vlined]{algorithm2e}

\title{Sliced Kernelized Stein Discrepancy}

\author{%
  Wenbo Gong \\
  University of Cambridge\\
  \texttt{wg242@cam.ac.uk} \\
   \And
  Yingzhen Li \thanks{Work done at Microsoft Research Cambridge} \\
  Imperial College London \\
  \texttt{yingzhen.li@imperial.ac.uk} \\
  \AND
  Jos\'e Miguel Hern\'andez-Lobato \\
  University of Cambridge \\
  The Alan Turing Institute\\
  \texttt{jmh233@cam.ac.uk} \\
}

\iclrfinalcopy
\begin{document}

\maketitle

\begin{abstract}
Kernelized Stein discrepancy (KSD), though being extensively used in goodness-of-fit tests and model learning, suffers from the curse-of-dimensionality. We address this issue by proposing the \emph{sliced Stein discrepancy} and its scalable and kernelized variants, which employ kernel-based test functions defined on the optimal one-dimensional projections. When applied to goodness-of-fit tests, extensive experiments show the proposed discrepancy significantly outperforms KSD and various baselines in high dimensions. {For model learning, we show its advantages over existing Stein discrepancy baselines by training independent component analysis models with different discrepancies.} We further propose a novel particle inference method called \textit{sliced Stein variational gradient descent} (S-SVGD) which alleviates the mode-collapse issue of SVGD in training variational autoencoders.

\end{abstract}

\section{Introduction}
\vspace{-6pt}
Discrepancy measures for quantifying differences between two probability distributions play key roles in statistics and machine learning. 
Among many existing discrepancy measures,
Stein discrepancy (SD) is unique in that it
only requires samples from one distribution and the 
score function (i.e. the gradient up to a multiplicative constant) from the other \citep{gorham2015measuring}. SD, a special case of \textit{integral probability metric} (IPM) \citep{sriperumbudur2009integral}, requires finding an optimal test function within a given function family. This optimum is analytic when a \textit{reproducing kernel Hilbert space} (RKHS) is used as the test function family, and the corresponding SD is named \textit{kernelized Stein discrepancy} (KSD) \citep{liu2016kernelized,chwialkowski2016kernel}. Variants of SDs have been widely used in both Goodness-of-fit (GOF) tests \citep{liu2016kernelized,chwialkowski2016kernel} and model learning \citep{liu2016two,grathwohl2020cutting,hu2018stein,liu2016stein}.

Although theoretically elegant, KSD, especially with RBF kernel, suffers from the ''curse-of-dimensionality'' issue, which leads to significant deterioration of test power in GOF tests \citep{chwialkowski2016kernel,huggins2018random} and mode collapse in particle inference \citep{zhuo2017message,wang2018stein}. 
%
%
A few attempts have been made to address this problem, however, they either are limited to specific applications with strong assumptions \citep{zhuo2017message,chen2020projected,wang2018stein} or require significant approximations \citep{singhal2019kernelized}. 
%
As an alternative, in this work we present our solution to this issue by adopting the idea of ``slicing''. 
{Here the key idea is to project the score function and test inputs onto multiple one dimensional slicing directions, resulting in a variant of SD that only requires to work with one-dimensional inputs for the 
test functions.}
Specifically, our contributions are as follows.


\vspace{-6pt}
\begin{itemize}
    \item We propose a novel theoretically validated family of discrepancies called \textit{sliced Stein discrepancy} (SSD), along with its scalable variant called \textit{max sliced kernelized Stein discrepancy} (maxSKSD) using kernel tricks and the \emph{optimal test directions}. 
    %
    %
    \item A GOF test is derived based on an unbiased estimator of maxSKSD with optimal test directions.
    MaxSKSD achieves superior performance on benchmark problems and \textit{restricted Boltzmann machine} models \citep{liu2016kernelized,huggins2018random}.
    %
    %
    \item We evaluate the maxSKSD in model learning by two schemes. First, we train an independent component analysis (ICA) model in high dimensions by directly minimising maxSKSD, which results in faster convergence compared to baselines \citep{grathwohl2020cutting}. Further, we propose a particle inference algorithm based on maxSKSD called the \emph{sliced Stein variational gradient descent} (S-SVGD) as a novel variant of the original SVGD \citep{liu2016stein}. It alleviates the posterior collapse of SVGD when applied to training variational autoencoders \citep{kingma2013auto,rezende2014stochastic}.
\end{itemize}


\vspace{-6pt}
\section{Background}
\vspace{-4pt}
\subsection{Kernelized Stein Discrepancy}
\vspace{-6pt}
\label{sub:KSD}
For two probability distributions $p$ and $q$ supported on $\mathcal{X}\subseteq \mathbb{R}^D$ with continuous differentiable densities ${p}(\bm{x})$ and $q(\bm{x})$, we define the score $\bm{s}_p(\bm{x})=\nabla_{\bm{x}}\log p(\bm{x})$ and $\bm{s}_q(\bm{x})$ accordingly. For a test function $f:\mathcal{X}\rightarrow\mathbb{R}^D$, the Stein operator is defined as
\begin{equation}
    \mathcal{A}_pf(\bm{x})=\bm{s}_p(\bm{x})^Tf(\bm{x})+\nabla_{\bm{x}}^Tf(\bm{x}).
    \label{eq: Stein Operator}
\end{equation}
{For a function $f_0:\mathbb{R}^D\rightarrow\mathbb{R}$, the \emph{Stein class} $\mathcal{F}_q$ of $q$ is defined as the set of functions satisfying Stein's identity \citep{stein1972bound}: $\mathbb{E}_{q}[\bm{s}_q(\bm{x})f_0(\bm{x})+\nabla_{\bm{x}}f_0(\bm{x})]=\bm{0}$. This can be generalized to a vector function $\bm{f}:\mathbb{R}^D\rightarrow\mathbb{R}^D$ where $\bm{f}=[f_1(\bm{x}),\ldots,f_D(\bm{x})]^T$ by letting $f_i$ belongs to the Stein class of $q$ for each $i\in D$}.
Then the Stein discrepancy \citep{liu2016kernelized,gorham2015measuring} is defined as
\begin{equation}
    D(q,p)=\sup_{f\in\mathcal{F}_q}\mathbb{E}_{q}[\mathcal{A}_pf(\bm{x})]=\sup_{f\in\mathcal{F}_q} \mathbb{E}_{q}[(\bm{s}_p(\bm{x})-\bm{s}_q(\bm{x}))^Tf(\bm{x})].
    \label{eq: Stein Discrepancy}
\end{equation}
When $\mathcal{F}_q$ is sufficiently rich, {and $q$ vanishes at the boundary of $\mathcal{X}$}, the supremum is obtained at $f^*(\bm{x})\propto\bm{s}_p(\bm{x})-\bm{s}_q(\bm{x})$ with some mild regularity conditions on $f$ \citep{hu2018stein}. Thus, the Stein discrepancy focuses on the score difference of $p$ and $q$.
\emph{Kernelized Stein discrepancy} (KSD) \citep{liu2016kernelized,chwialkowski2016kernel} restricts the test functions to be in a $D$-dimensional RKHS $\mathcal{H}_D$ with kernel $k$ to obtain an analytic form.
By {defining}
    $\allowdisplaybreaks u_p(\bm{x},\bm{x}')=\bm{s}_p(\bm{x})^T\bm{s}_p(\bm{x}')k(\bm{x},\bm{x}')+\bm{s}_p(\bm{x})^T\nabla_{\bm{x}'}k(\bm{x},\bm{x}')+\bm{s}_p(\bm{x}')^T\nabla_{\bm{x}}k(\bm{x},\bm{x}')+\text{Tr}(\nabla_{\bm{x},\bm{x}'}k(\bm{x},\bm{x}'))$
the analytic form of KSD is:
\begin{equation}
\begin{split}
    D^2(q,p)&= \left( \sup_{f\in\mathcal{H}_D,||f||_{\mathcal{H}_D}\leq 1}\mathbb{E}_{q}[\mathcal{A}_pf(\bm{x})] \right)^2
    =\mathbb{E}_{q(\bm{x}) q(\bm{x}')}[u_p(\bm{x},\bm{x}')].
    \end{split}
    \label{eq: KSD}
\end{equation}

\vspace{-5pt}
\subsection{Stein Variational Gradient Descent}
\label{sub:Background SVGD}
\vspace{-4pt}
Although SD and KSD can be directly minimized for variational inference (VI) \citep{ranganath2016operator,liu2016two,feng2017learning}, {\citet{liu2016stein}} alternatively proposed a novel particle inference algorithm called \emph{Stein variational gradient descent} (SVGD). It applies a sequence of deterministic transformations to a set of points such that each of mappings maximally decreases the Kullback-Leibler (KL) divergence from the particles' underlying distribution $q$ to the target $p$.

To be specific, we define the mapping $T(\bm{x}):\mathbb{R}^D\rightarrow\mathbb{R}^D$ as $T(\bm{x})=\bm{x}+\epsilon\bm{\phi}(\bm{x})$ where $\bm{\phi}$ characterises the perturbations. 
The result from \cite{liu2016stein} shows that the
optimal perturbation inside the RKHS is exactly the optimal test function in KSD. 
\begin{lemma}{\citep{liu2016stein}}
Let $T(\bm{x})=\bm{x}+\epsilon\bm{\phi}(\bm{x})$ and $q_{[T]}(\bm{z})$ be the density of $\bm{z}=T(\bm{x})$ when $\bm{x}\sim q(\bm{x})$. If the perturbation $\bm{\phi}$ is in the RKHS $\mathcal{H}_D$ and $||\bm{\phi}||_{\mathcal{H}_D}\leq D(q,p)$, then the steepest descent directions $\bm{\phi}^*_{q,p}$ is 
\begin{equation}
    \bm{\phi}^*_{q,p}(\cdot)=\mathbb{E}_q[\nabla_{\bm{x}}\log p(\bm{x})k(\bm{x},\cdot)+\nabla_{\bm{x}}k(\bm{x},\cdot)]
    \label{eq: SVGD Update}
\end{equation}
and $\nabla_{\epsilon}KL[q_{[T]}||p]|_{\epsilon=0}=-D^2(q,p)$.
\label{lem: SVGD and KSD}
\end{lemma}
The first term in {Eq.(\ref{eq: SVGD Update})} is called drift, which drives the particles towards a mode of $p$. The second term controls the repulsive force, which spreads the particles around the mode. When particles stop moving, the KL decrease magnitude $\epsilon D^2(q,p)$ is $0$, which means the KSD is zero and $p=q$ a.e. 
\vspace{-6pt}
\section{Sliced Kernelized Stein Discrepancy}
\begin{figure}
    \begin{minipage}[t]{0.5\textwidth}
    \includegraphics[scale=0.5]{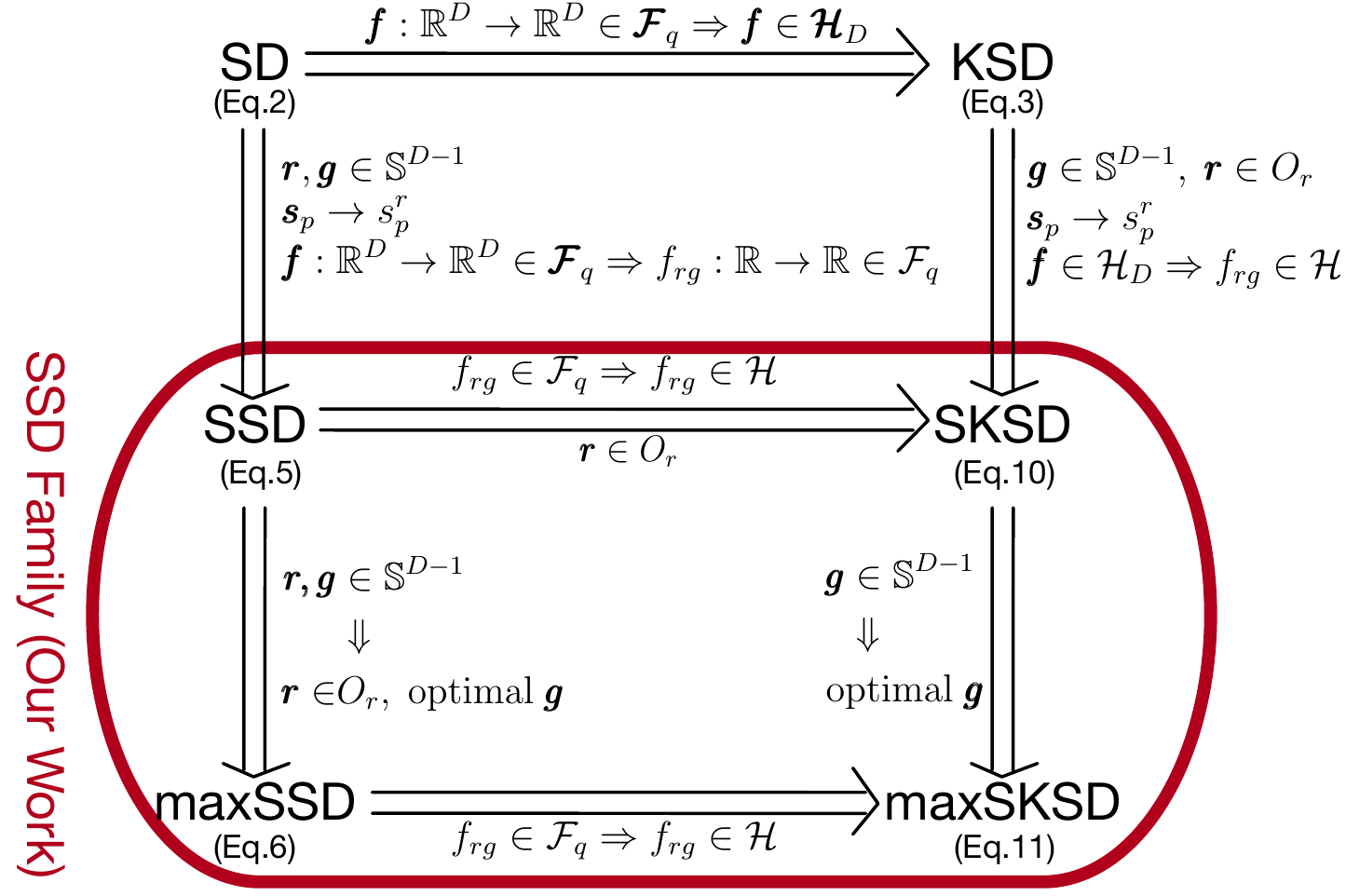}
    \label{fig:Relationship}
    \end{minipage}
    \begin{minipage}[t]{0.5\textwidth}
    \includegraphics[scale=0.45]{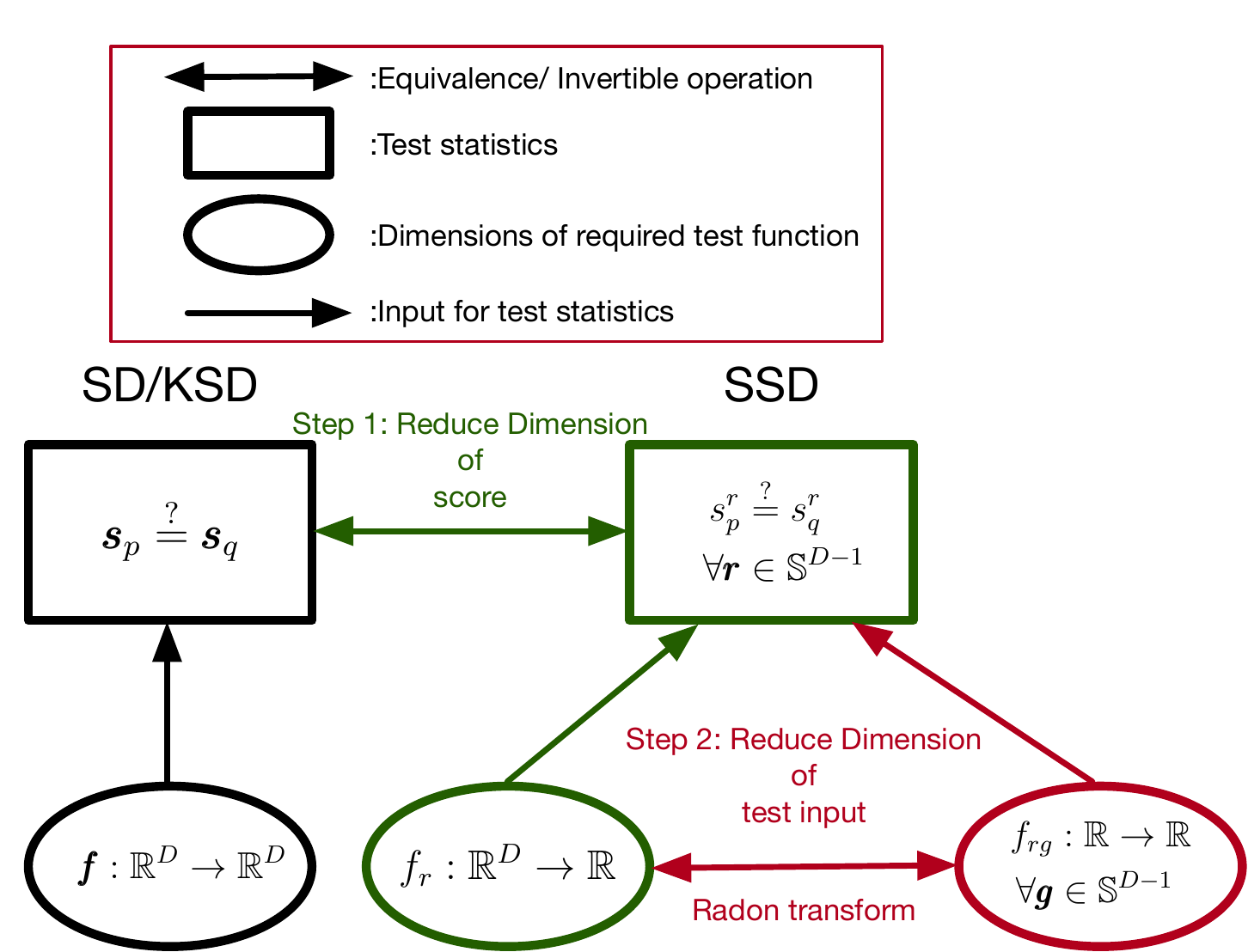}
    \label{fig:Intuition}
    \end{minipage}
    \caption{\textbf{(Left)} The connections between SD, KSD and the proposed SSD family. \textbf{(Right)} The intuition of SSD. The rectangular boxes indicate what statistics the discrepancy wants to test. The circle represents the dimension of the test function required for such test. The double arrow means equivalence relations or invertable operations.}
    \label{fig: intuition and relation}
\end{figure}
\vspace{-6pt}
We propose the \emph{sliced Stein discrepancy} (SSD) and kernelized version named maxSKSD. Theoretically, we prove their correctness as discrepancy measures. Methodology-wise, we apply maxSKSD to GOF tests, and develop two ways for model learning.






\vspace{-4pt}
\subsection{Sliced Stein Discrepancy}
\vspace{-6pt}

\wg{Before moving to the details, we give a brief overview of the intuition on how to tackle the curse-of-fimensionality issue of SD (The right figure of Figure \ref{fig: intuition and relation}). For detailed explanation, refer to appendix \ref{subsec: SSD Intuition}. This issue of Stein discrepancy (Eq.\ref{eq: Stein Discrepancy}) comes from two sources: the score function $s_p(\bm{x})$ and the test function $f(\bm{x})$ defined on $\mathcal{X} \subset \mathbb{R}^D$. First, we notice that comparing $\bm{s}_p$ and $\bm{s}_q$ is equivalent to comparing projected score $s_p^r=\bm{s}_p^T\bm{r}$ and $s_q^r$ for all $\bm{r}\in \mathbb{S}^{D-1}$ on an hyper-sphere ({\color[RGB]{0,81,30}Green square in Figure \ref{fig: intuition and relation} (Right)}). This operation reduces the test function's output from $\mathbb{R}^D$ to $\mathbb{R}$ ({\color[RGB]{0,81,30}Green circle in Figure \ref{fig: intuition and relation} (Right)}). However, its input dimension is not affected. Reducing the input dimension of test functions is non-trivial, as directly removing input dimensions results in the test power decrease. \wgr{This is because less information is accessed by the test function (see examples in appendix \ref{subsec: SSD Intuition}).} \wgr{Our solution to this problem uses \textit{Radon transform} which is inspired by CT-scans. It projects the original test function $f(\bm{x})$ in Stein discrepancy (Eq. \ref{eq: Stein Discrepancy}) (as an $\mathbb{R}^D\rightarrow \mathbb{R}$ mapping) to a group of $\mathbb{R}\rightarrow\mathbb{R}$ functions along a set of directions ({\color{red}{$\bm{g}\in\mathbb{S}^{D-1}$}}). Then, this group of functions are used as the new test functions to define the proposed discrepancy. The invertibility of \textit{Radon transform} ensures that testing with input in the original space $\mathbb{R}^D$ is equivalent to the test using a group of low dimensional functions with input in $\mathbb{R}$}. 
Thus, the above two steps not only reduce the dimensions of the test function's output and input, but also maintain the validity of the resulting discrepancy as each step is either equivalent or invertible.}

In detail, {assume two distributions $p$ and $q$ supported on $\mathbb{R}^D$} with differentiable densities $p(\bm{x})$ and $q(\bm{x})$, and define the test functions $f(\cdot;\bm{r},\bm{g}):\mathbb{R}^D\rightarrow\mathbb{R}$ such that $f(\bm{x};\bm{r},\bm{g})=f_{rg}\circ h_g(\bm{x})=f_{rg}(\bm{x}^T\bm{g})$, where $h_g(\cdot)$ is the inner product with $\bm{g}$ and $f_{rg}:\mathbb{R}\rightarrow \mathbb{R}$. \wgr{One should note that the $\bm{r}$ and $\bm{g}$ in $f(\cdot;\bm{r},\bm{g})$ should not just be treated as parameters in a test function $f$. In fact, they are more like the index to indicate that for each pair of $\bm{r}$, $\bm{g}$, we need a new $f(\cdot;\bm{r},\bm{g})$, i.e. new $f_{rg}$, which is completely independent to other test functions.} The proposed sliced Stein discrepancy (SSD), defined using two uniform distributions $p_r(\bm{r})$ and $p_g(\bm{g})$ over the hypersphere $\mathbb{S}^{D-1}$, is given by the following, with $f_{rg}\in \mathcal{F}_q$ meaning $f(\cdot;\bm{r},\bm{g})\in\mathcal{F}_q$:
\begin{equation}
    S(q,p)=\mathbb{E}_{p_r,p_g}\left[
    \sup_{f_{rg}\in\mathcal{F}_q}{\mathbb{E}_q[s^r_p(\bm{x})f_{rg}(\bm{x}^T\bm{g})+\bm{r}^T\bm{g}\nabla_{\bm{x}^T\bm{g}}f_{rg}(\bm{x}^T\bm{g})]}
    \right].
    \label{eq: SSD}
\end{equation}

We verify the proposed SSD is a valid discrepancy measure, namely, $S(q,p)=0$ iff.~$q=p$ a.e.

\begin{theorem}{(SSD Validity)} If assumptions 1-4 in appendix \ref{Sec:Basics} are satisfied, then for two probability distributions $p$ and $q$, $S(q,p)\geq 0$, and $S(q,p)=0$ if and only if $p=q$ a.e.
\label{thm: SSD non-convergence}
\end{theorem}

Despite this attractive theoretical result, SSD is difficult to compute in practice. Specifically, the expectations over $\bm{r}$ and $\bm{g}$ can be approximated by Monte Carlo but this typically requires a very large number of samples in high dimensions \citep{deshpande2019max}. We propose to relax such limitations by using only a finite number of slicing directions $\bm{r}$ from an orthogonal basis $O_r$ of $\mathbb{R}^D$, e.g.~the standard basis of one-hot vectors, {and the corresponding optimal test direction} $\bm{g}_r$ for each $\bm{r}$.
%
%
We call this variant maxSSD, which is defined as follows and validated in Corollary \ref{coro: maxSSD}:
%

\begin{equation}
    S_{max}(q,p)=\sum_{\bm{r}\in O_r}{\sup_{f_{rg_r}\in\mathcal{F}_q,\bm{g}_r\in\mathbb{S}^{D-1}}{\mathbb{E}_q[{s}^r_p(\bm{x})f_{rg_r}(\bm{x}^T\bm{g}_r)+\bm{r}^T\bm{g}_r\nabla_{\bm{x}^T\bm{g}_r}f_{rg_r}(\bm{x}^T\bm{g}_r)]}}.
    \label{eq:maxSSD}
\end{equation}

\begin{corollary}{(maxSSD)}
Assume the conditions in Theorem \ref{thm: SSD non-convergence}, then $S_{max}(q,p)=0$ iff. $p=q$ a.e.
\label{coro: maxSSD}
\end{corollary}

\vspace{-4pt}
\subsection{Closed form Solution with the Kernel Trick}
\vspace{-6pt}

The optimal test function given $\bm{r}$ and $\bm{g}$ is intractable without further assumptions on the test function families. This introduces another scalability issue as optimizing these test functions explicitly can be time consuming.
Fortunately, we can apply the kernel trick to obtain its analytic form. Assume for each test function $f_{rg}\in\mathcal{H}_{rg}$, where $\mathcal{H}_{rg}$ is a scalar-valued RKHS equipped with kernel $k(\bm{x},\bm{x}';\bm{r},\bm{g})=k_{rg}(\bm{x}^T\bm{g},\bm{x}'^T\bm{g})$ that satisfies assumption 5 in appendix \ref{Sec:Basics} and $f_{rg}(\bm{x}^T\bm{g})=\langle f_{rg},k_{rg}(\bm{x}^T\bm{g},\cdot)\rangle_{\mathcal{H}_{rg}}$. We define the following quantities:
\begin{align}
    \xi_{p,r,g}(\bm{x},\cdot)&=s_p^r(\bm{x})k_{rg}(\bm{x}^T\bm{g},\cdot)+\bm{r}^T\bm{g}\nabla_{\bm{x}^T\bm{g}}k_{rg}(\bm{x}^T\bm{g},\cdot),
    \label{eq: SKSD test function}\\
      h_{p,r,g}(\bm{x},\bm{y})&=s^r_p(\bm{x})k_{rg}(\bm{x}^T\bm{g},\bm{y}^T\bm{g})s^r_p(\bm{y})+\bm{r}^T\bm{g}s^r_p(\bm{y})\nabla_{\bm{x}^T\bm{g}}k_{rg}(\bm{x}^T\bm{g},\bm{y}^T\bm{g})+\nonumber\\
      &\quad\,\,\bm{r}^T\bm{g}s^r_p(\bm{x})\nabla_{\bm{y}^T\bm{g}}k_{rg}(\bm{x}^T\bm{g},\bm{y}^T\bm{g})+(\bm{r}^T\bm{g})^2\nabla^2_{\bm{x}^T\bm{g},\bm{y}^T\bm{g}}k_{rg}(\bm{x}^T\bm{g},\bm{y}^T\bm{g}).
\label{eq: SKSD inner}
\end{align}
The following theorem describes the optimal test function inside SSD (Eq.(\ref{eq: SSD})) and maxSSD (Eq.(\ref{eq:maxSSD})). 
\begin{theorem}{(Closed form solution)}
If $\mathbb{E}_q[h_{p,r,g}(\bm{x},\bm{x})]<\infty$, then
\begin{equation}
    \begin{split}
        D^2_{rg}(q,p)&=||\sup_{f_{rg}\in\mathcal{H}_{rg},||f_{rg}||\leq 1}\mathbb{E}_q[{s}^r_p(\bm{x})f_{rg}(\bm{x}^T\bm{g})+\bm{r}^T\bm{g}\nabla_{\bm{x}^T\bm{g}}f_{rg}(\bm{x}^T\bm{g})]||^2\\
        &=||\mathbb{E}_q[\xi_{p,r,g}(\bm{x})]||^2_{\mathcal{H}_{rg}}
        =\mathbb{E}_{q(\bm{x})q(\bm{x}')}[h_{p,r,g}(\bm{x},\bm{x}')].
    \end{split}
    \label{eq: D_rg}
\end{equation}
\label{thm: SKSD Analytic form}
\end{theorem}
Next, we propose the kernelized version of SSD with orthogonal basis $O_r$, called SKSD. 
\begin{theorem}{(SKSD as a discrepancy)}
For two probability distributions $p$ and $q$, given assumptions 1,2 and 5 in appendix \ref{Sec:Basics} and $\mathbb{E}_q[h_{p,r,g}(\bm{x},\bm{x})]<\infty$ for all $\bm{r}$ and $\bm{g}$, we define SKSD as
\begin{equation}
    SK_{o}(q,p)=\sum_{\bm{r}\in O_r}\int_{\mathbb{S}^{D-1}}{p_g(\bm{g})D^2_{rg}(q,p)d\bm{g}},
    \label{eq: SKSD}
\end{equation}
which is equal to 0 if and only if $p=q$ a.e.
\label{thm: Validity of SKSD}
\end{theorem}
Following the same idea of maxSSD (Eq.\ref{eq:maxSSD}), it suffices to use optimal slice direction $\bm{g}_r$ for each $\bm{r}\in O_r$, resulting in a slicing matrix $\bm{G}\in\mathbb{S}^{D\times (D-1)}$.
We name this discrepancy as maxSKSD, or \textit{maxSKSD-g} when we need to distinguish it from another variant described later.
\begin{corollary}{(maxSKSD)}
Assume the conditions in Theorem \ref{thm: Validity of SKSD} are satisfied. Then
\begin{equation}
    SK_{max}(q,p)=\sum_{\bm{r}\in O_r}{\sup_{\bm{g}_r}{D^2_{rg_r}(q,p)}}
    \label{eq: maxSKSD}
\end{equation}
is equal to 0 if and only if $p=q$ a.e.
\label{coro: max SKSD}
\end{corollary}
Figure \ref{fig: intuition and relation} (Left) clarifies the connections between the mentioned discrepancies. \wg{We emphasise that using a single projection $\bm{g}$ in maxSKSD may be insufficient when no single projected feature $\bm{x}^T\bm{g}$ is informative enough to describe the difference between $p$ and $q$. Instead, in maxSKSD, for each score projection $\bm{r}\in O_r$, we have a corresponding $\bm{g}_r$.}
\wg{One can also use the optimal $\bm{r}$ to replace the summation over $O_r$, which provides additional benefits in certain GOF tests. We call this discrepancy \textit{maxSKSD-{rg}}, and its validity can be proved accordingly. Interestingly, in appendix \ref{APP: Pathology},  we show under certain scenarios \textit{maxSKSD-g} can have inferior performance due to the noisy information provided by the redundant dimensions. Further, we show that such limitation can be efficiently addressed by using \textit{maxSKSD-{rg}}.}    

\vspace{-6pt}
\paragraph{Kernel choice and optimal $\bm{G}$}

\wg{RBF kernel with median heuristics is a common choice. However, better kernels, e.g.~deep kernels which evaluate a given kernel on the transformed input $\phi(\bm{x})$, might be preferred. It is non-trivial to directly use such kernel on SKSD or maxSKSD. 
We propose an adapted form of Eq.(\ref{eq: SKSD}) to incorporate such kernel and maintain its validity. We include the details in appendix \ref{App: Deep Kernel} and leave the experiments for future work.}

\wgr{The quality of sliced direction $\bm{G}$ is crucial for the performance of both \textit{maxSKSD-g} or \textit{maxSKSD-rg}. Indeed, it represents the projection directions that two distributions differ the most. The closed-form solutions of $\bm{G}$ is not analytic in general, in practice, finding the optimal $\bm{G}$ involves solving other difficult optimizations as well (projection $\bm{r}$ and test function $f_{rg}$). For the scope of this work, we obtained $\bm{G}$ by optimizing \textit{maxSKSD-g} or \textit{maxSKSD-rg} using standard gradient optimization, e.g. Adam, with random initialization. Still in some special cases (e.g.~$p$, $q$ are full-factorized), analytic solutions of optimal $\bm{G}$ exists, which is further discussed in appendix \ref{App: Closed form}.}


\vspace{-4pt}
\subsection{Application of maxSKSD}
\vspace{-4pt}
\paragraph{Goodness-of-fit Test} Assume the optimal test directions $\bm{g}_r\in\bm{G}$ are available, {maxSKSD (Eq.(\ref{eq: maxSKSD}))} can then be estimated using U-statistics \citep{hoeffding1992class,serfling2009approximation}. Given i.i.d.~samples $\{\bm{x}_i\}_{i=1}^N\sim q$, we have an unbiased minimum variance estimator:
\begin{equation}
    \reallywidehat{SK}_{max}(q,p)=\frac{1}{N(N-1)}\sum_{\bm{r}\in O_r}{\sum_{1\leq i\neq j\leq N}{h_{p,r,g_r}(\bm{x}_i,\bm{x}_j)}}.
    \label{eq: U maxSKSD}
\end{equation}
The asymptotic behavior of the estimator is analyzed in appendix \ref{APP: GOF Test}. We use bootstrap \citep{liu2016kernelized,huskova1993consistency,arcones1992bootstrap} to determine the threshold for rejecting the null hypothesis as indicated in algorithm \ref{alg: GOF Test}. The bootstrap samples can be calculated by 
\begin{equation}
    \reallywidehat{SK}_m^*=\sum_{1\leq i\neq j\leq N}{(w^m_i-\frac{1}{N})(w^m_j-\frac{1}{N})\sum_{\bm{r}\in O_r}{h_{p,r,g_r}(\bm{x}_i,\bm{x}_j)}}
    \label{eq: Bootstrap samples}
\end{equation}
where $(w^m_1,\ldots,w^m_N)_{m=1}^M$ are random weights drawn from multinomial distributions $\text{Multi}(N,\frac{1}{N},\ldots,\frac{1}{N})$. 

\begin{algorithm}[H]
\SetKwInOut{Input}{Input}
\SetKwInOut{Output}{Hypothesis}
\SetAlgoLined
\Input{Samples $\{\bm{x}_i\}_{i=1}^N \sim q(\bm{x})$, score function $\bm{s}_p(\bm{x})$, Orthogonal basis $O_r$, optimal test direction $\bm{g}_r$ for each $\bm{r}\in O_r$, kernel function $k_{rg}$, significant level $\alpha$, and bootstrap sample size $M$.}
\Output{$H_0$: $p=q$   v.s.   $H_1$: $q \neq p$}
Compute $\reallywidehat{SK}_{max}(q,p)$ using U-statistic Eq.(\ref{eq: U maxSKSD})\;
Generate $M$ bootstrap samples $\{\reallywidehat{SK}_m^*\}_{m=1}^M$ using Eq.(\ref{eq: Bootstrap samples})\;
Reject null hypothesis $H_0$ if the proportion $\reallywidehat{SK}_m^*>\reallywidehat{SK}_{max}(q,p)$ is less than $\alpha$.
 \caption{GOF Test with maxSKSD U-statistics}
 \label{alg: GOF Test}
\end{algorithm}

\begin{wrapfigure}[17]{r}{0.35\textwidth}
\centering
    \includegraphics[scale=0.16]{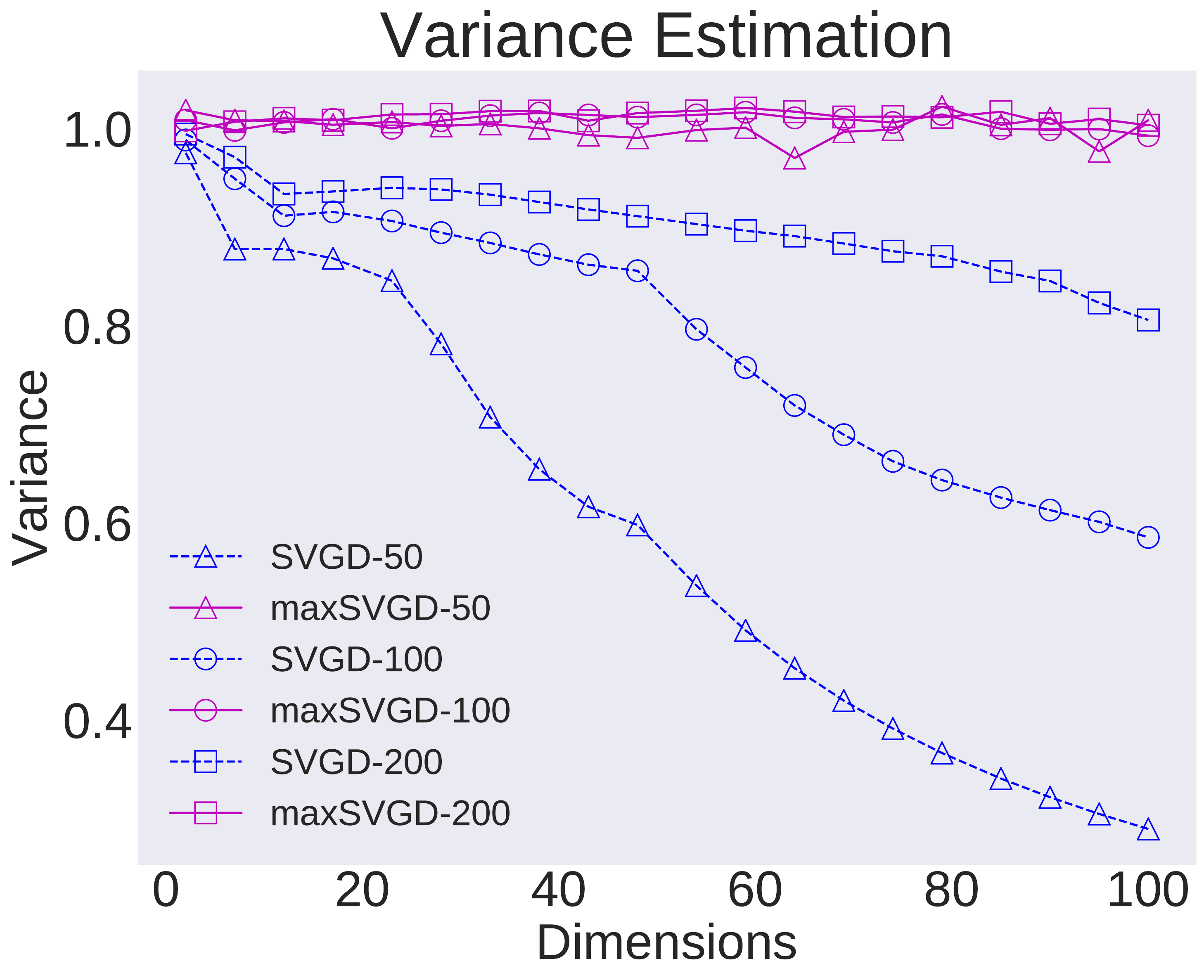}
\caption{Estimating the average variance of $p(\bm{x}) = \mathcal{N}(\bm{0},\bm{I})$ across dimensions using SVGD particles.
{SVGD-50 means the variance are estimated using 50 samples.} }
\label{fig: S_SVGD_Gaussian}
\end{wrapfigure}
\vspace{-6pt}
%
\paragraph{Model Learning} The proposed maxSKSD can be applied to model learning in two ways. First, it can be directly used as a training objective, in such case $q$ is the data distribution and $p$ is the model to be learned, and the learning algorithm performs $\min_{p} SK_{max}(q,p)$.
The second model learning scheme is to leverage the particle inference for latent variables and train the model parameters using an EM-like \citep{dempster1977maximum} algorithm. 
Similar to the relation between SVGD and KSD, we can derive a corresponding particle inference algorithm based on maxSKSD, called \textit{sliced-SVGD} (S-SVGD).
In short, we define a specific form of the perturbation as
$\bm{\phi}(\bm{x})=[\phi_{g_i}(\bm{x}^T\bm{g}_i),\ldots,\phi_{g_D}(\bm{x}^T\bm{g}_D)]^T$ and modify the proofs of Lemma \ref{lem: SVGD and KSD} accordingly. The resulting S-SVGD algorithm uses kernels defined on one dimensional projected samples,
which sidesteps the vanishing repulsive force problem of SVGD in high dimensions \citep{zhuo2017message,wang2018stein}. {We illustrate this in Figure \ref{fig: S_SVGD_Gaussian} by estimating the variance of a standard Gaussian with the particles obtained by SVGD or S-SVGD (see appendix \ref{App: S_SVGD_Gaussian}). We see that as the dimension increases, SVGD severely under-estimates the variance of $p$, while the S-SVGD remains robust.}
Furthermore, its validity is justified since in such case the KL gradient equals to maxSKSD which is a valid discrepancy. Readers are referred to appendix \ref{APP: Sliced SVGD} for the derivations. \wgr{We also give an analysis of their memory and computational cost for both GOF and model learning in appendix \ref{APP: computational memory cost}.}

\vspace{-6pt}
\section{Experiments}
\vspace{-4pt}
\subsection{Goodness of fit test}
\label{subsec: Experiment GOF}
\vspace{-6pt}
We evaluate maxSKSD (Eq.(\ref{eq: maxSKSD})) for GOF tests in high dimensional problems. First, we demonstrate its robustness to the increasing dimensionality using the Gaussian GOF benchmarks \citep{jitkrittum2017linear,huggins2018random,chwialkowski2016kernel}. Next, we show the advantage of our method for GOF tests on 50-dim \textit{Restricted Boltzmann Machine} (RBM) \citep{liu2016kernelized,huggins2018random,jitkrittum2017linear}. 
We included in comparison extensive baseline test statitics for GOF test: Gaussian or Cauchy random Fourier features (RFF) \citep{rahimi2008random}, KSD with RBF kernel \citep{liu2016kernelized,chwialkowski2016kernel}, finite set Stein discrepancy (FSSD) with random or optimized test locations \citep{jitkrittum2017linear}, random feature Stein discrepancy (RFSD) with L2 SechExp and L1 IMQ kernels \citep{huggins2018random}, and maximum mean discrepancy (MMD) \citep{gretton2012kernel} with RBF kernel.
Notice that we use gradient descent to obtain the test directions $\bm{g}_r$ (and potentially the slicing directions $\bm{r}$) for Eq.(\ref{eq: maxSKSD}).
\vspace{-5pt}
\subsubsection{GOF Tests with High dimensional Gaussian Benchmarks}
\label{subsub: benchmark GOF}
\vspace{-5pt}
We conduct 4 different benchmark tests with $p=\mathcal{N}(0,\bm{I})$: (1) \textbf{Null test}: $q=p$; (2) \textbf{Laplace}: $q(\bm{x})=\prod_{d=1}^D{\text{Lap}(x_d|0,1/\sqrt{2})}$ with mean/variance matched to $p$; (3) \textbf{Multivariate-t}: $q$ is fully factorized multivariate-t with $5$ degrees of freedom, \wgr{$0$ mean and scale $1$. In order to match the variance of $p$ and $q$, we change the variance of $p$ to $\frac{5}{5-2}$}; (4) \textbf{Diffusion}: $q(\bm{x})=\mathcal{N}(\bm{0},\bm{\Sigma}_1)$ where the variance of $1^{\text{st}}$-dim is 0.3 and the rest is the same as in $\bm{I}$. 
For the testing setup, we set the significance level $\alpha=0.05$. For FFSD and RFSD, we use the open-sourced code from the original publications. We only consider {maxSKSD-g} here as it already performs nearly optimally. We refer to appendix \ref{App: High dimensional benchmark GOF test} for details.

\begin{figure}
\vspace{-1em}
    \centering
    \includegraphics[scale=0.112]{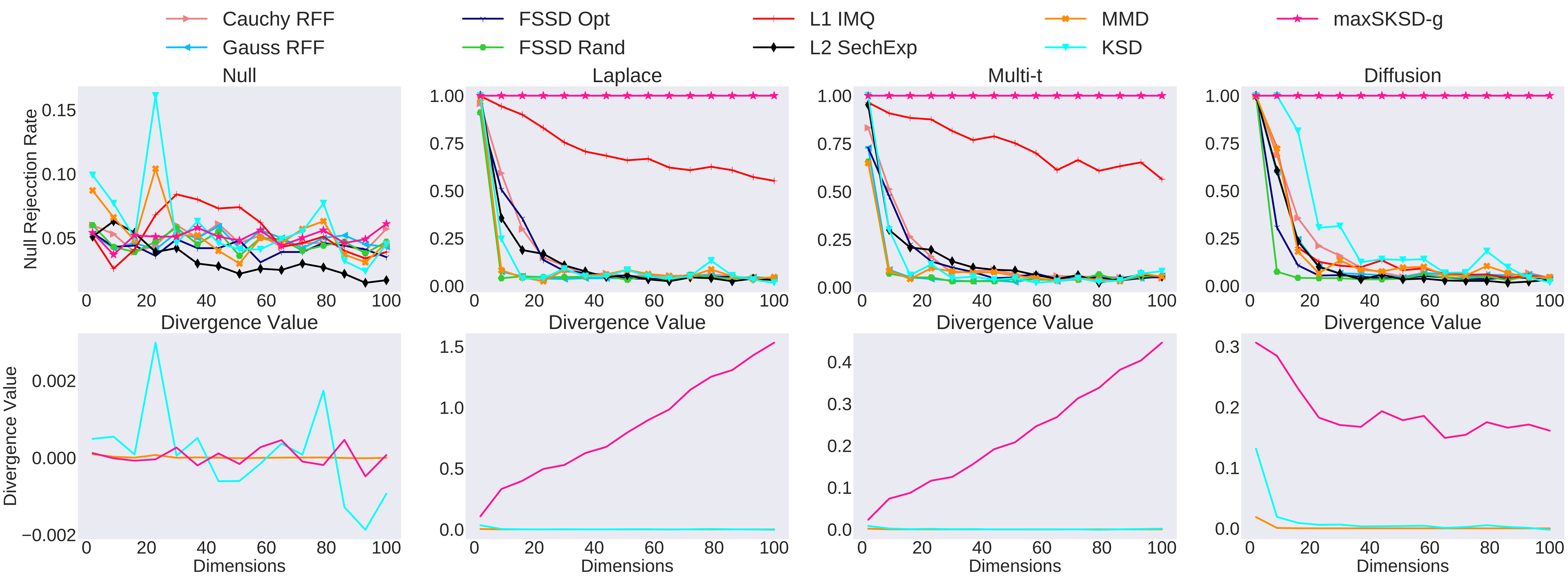}
    \caption{Each column reports GOF test results for a different alternative hypothesis, with the upper panel showing the rejection rate of the Null hypothesis and the lower panel showing the discrepancy value averaged over all trials. Both quantities are plotted w.r.t. the number of dimensions.}
    \label{fig:High dimensional benchmark GOF test}
    \vspace{-1em}
\end{figure}
{Figure} \ref{fig:High dimensional benchmark GOF test}
shows the GOF test performances and the corresponding discrepancy values. In summary, the proposed maxSKSD outperforms the baselines in all tests, where the result is robust to the increasing dimensions and the discrepancy values match the expected behaviours.
\vspace{-10pt}
\paragraph{Null} The left-most column in Figure \ref{fig:High dimensional benchmark GOF test} shows that all methods behave as expected as the rejection rate is closed to the significance level, except for RFSD with L2 SechExp kernel. All the discrepancy values oscillate around $0$, with the KSD being less stable.
\vspace{-10pt}
\paragraph{Laplace and Multivariate-t} The two middle columns of Figure \ref{fig:High dimensional benchmark GOF test} show that maxSKSD-g achieves a nearly perfect rejection rate consistently as the dimension increases, while the test power for all baselines decreases significantly. For the discrepancy values, {similar to the KL divergence between $q$ and $p$, maxSKSD-g linearly increases with dimensions due to the independence assumptions.}. 
%
\vspace{-10pt}
\paragraph{Diffusion} This is a more challenging setting since $p$ and $q$ only differ in one of their marginal distributions, which can be easily buried in high dimensions. As shown in the rightmost column of Figure \ref{fig:High dimensional benchmark GOF test}, all methods failed in high dimensions except maxSKSD-g, which still consistently achieves optimal performance. For the discrepancy values, we expect a positive constant due to the one marginal difference between $p$ and $q$. Only maxSKSD-g behaves as expected as the problem dimension increases. {The decreasing value at the beginning is probably due to the difficulty in finding the optimal direction $\bm{g}$ in high dimensions when the training set is small}.

\vspace{-3pt}
\subsubsection{RBM GOF test}
\label{subsub: RBM GOF}
\begin{wrapfigure}[14]{R}{0.46\textwidth}
\centering
\vspace{-16pt}
\hspace{-8pt}
    \includegraphics[scale=0.15]{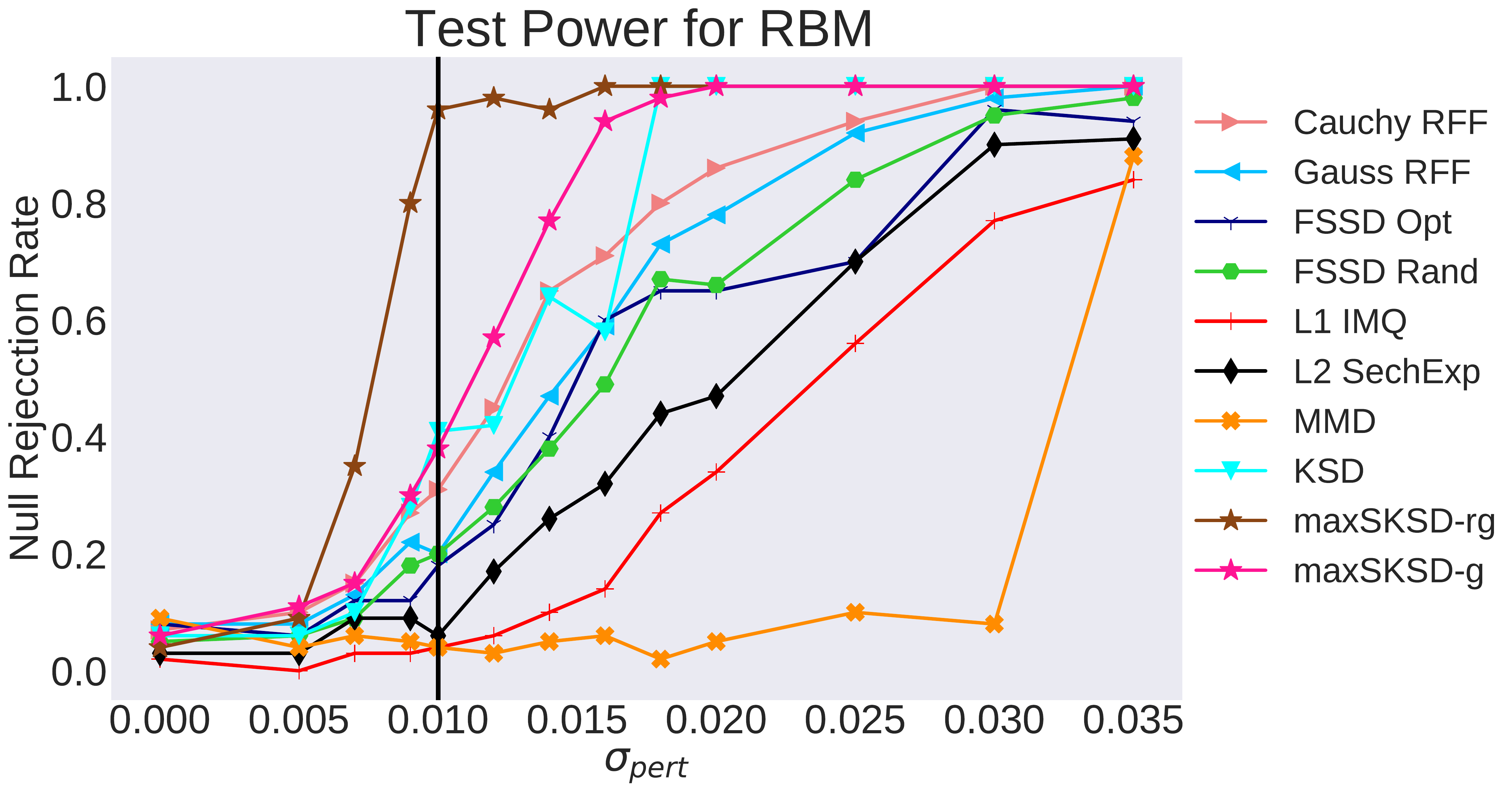}
\caption{RBM GOF Test with different levels of perturbation noise. The black vertical line indicates the perturbation level at $0.01$. }
\label{fig: RBM_All}
\end{wrapfigure}
We demonstrate the power maxSKSD for GOF tests on RBMs, but we now also include results for \textit{maxSKSD-rg}. 
We follow the test setups in \citet{liu2016kernelized,jitkrittum2017linear,huggins2018random} where different amounts of noise are injected into the weights to form the alternative hypothesis $q$. The samples are drawn using block Gibbs samplers. Refer to appendix {\ref{App: Setup RBM GOF test}} for details.
Figure \ref{fig: RBM_All} shows that maxSKSD based methods dominate the baselines, especially with maxSKSD-rg significantly outperforming the others. At perturbation level $0.01$, maxSKSD-rg achieves $0.96$ rejection rate, while others are all below $0.5$. This result shows the advantages of optimizing the slicing directions $\bm{r}$.
\vspace{-4pt}
\subsection{Model Learning}
\label{subsec: Experiment model training}
\vspace{-6pt}
We evaluate the efficiency of maxSKSD-based algorithms in training machine learning models. First, we use \textit{independent component analysis} (ICA) which is often used as a benchmark for evaluating training methods for energy-based model \citep{gutmann2010noise,hyvarinen2005estimation,ceylan2018conditional}. Our approach trains the ICA model by directly minimizing maxSKSD. Next, we evaluate the proposed S-SVGD particle inference algorithm, {when combined with \textit{amortization}} \citep{feng2017learning,pu2017vae}, in the training of a \textit{variational autoencoder} (VAE) \citep{kingma2013auto,rezende2014stochastic} on binarized MNIST. Appendix \ref{App: Bayesian NN} also 
shows superior results for S-SVGD when training a Bayesian neural network (BNN) on UCI datasets \citep{Dua:2019}.
\vspace{-3pt}

\subsubsection{ICA}
\vspace{-6pt}
ICA consists of a simple generative process $\bm{z}\sim\text{Lap}(0,1)$ and $\bm{x}=\bm{W}\bm{z}$, where the model parameters are a non-singular matrix $\bm{W}\in\mathbb{R}^{D\times D}$. The log density for $\bm{x}$ is $\log p(\bm{x})=\log p_z(\bm{W}^{-1}\bm{x})+C$, where the normalization constant $C$ can be ignored when training with Stein discrepancies. We train the models on data sampled from a randomly initialized ICA model and evaluate the corresponding test log likelihoods. We compare maxSKSD with KSD and the state-of-the-art LSD \citep{grathwohl2020cutting}.
For more details on the setup, we refer the reader to appendix \ref{App: ICA Setup}.

\begin{table}[]
\vspace{-1em}
\centering
\caption{Test NLL for different dimensional ICA with different objective functions. The above results are averaged over 5 independent runs of each methods.}
\label{tab: ICA NLL}
\vspace{-0.5em}
\begin{tabular}{c|lllllll}
\hline
Method                & \multicolumn{7}{c}{Dimension}                                                                                            \\
\multicolumn{1}{l|}{} & $D=10$              & $D=20$              & $D=40$              & $D=60$              & $D=80$              & $D=100$            & $D=200$            \\ \hline
KSD                   & \textbf{-10.23}          & -15.98          & -34.50         & -56.87         & -86.09         & -116.51        & -329.49       \\
LSD                   & {-10.42} & -14.54          & \textbf{-17.16} & \textbf{-15.05} & -12.39          & -5.49          & 46.63          \\
maxSKSD               & -10.45          & \textbf{-14.50} & -17.28          & -15.70          & \textbf{-11.91} & \textbf{-4.21} & \textbf{47.72} \\ \hline
\end{tabular}
\vspace{-0.7em}
\end{table}
Table \ref{tab: ICA NLL} shows that both maxSKSD and LSD are robust to increasing dimensions, with maxSKSD being better when $D$ is very large. Also at $D =200$, maxSKSD converges significantly faster than LSD (see Figure \ref{fig: ICA Training Curve} in appendix \ref{App: ICA Additional Plots}). This faster convergence is due to the closed-form solution for the optimal test functions,
whereas LSD requires adversarial training.
While KSD is also kernel-based, it suffers from the curse-of-dimensionality and fails to train the model properly for $D>20$. 
Instead the proposed maxSKSD can successfully avoid the problems of KSD with high dimensional
data. 
\vspace{-4pt}
\subsubsection{Amortized SVGD}
\vspace{-6pt}
Finally, we consider training VAEs with implicit encoders on dynamically binarized MNIST. The decoder is trained as in vanilla VAEs, but the encoder is trained by amortization \citep{feng2017learning,pu2017vae}, which minimizes the mean square error between the initial samples from the encoder, and the modified samples driven by the SVGD/S-SVGD dynamics (Algorithm \ref{alg: Amortized SVGD} in appendix \ref{App: Amortized_SVGD_Setup}).

We report performance in terms of test log-likelihood (LL). Furthermore we consider an imputation task, by removing the pixels in the lower half of the image and imputing the missing values using (approximate) posterior sampling from the VAE models. The performance is measured in terms of imputation diversity and correctness, {using label entropy and accuracy}.
For fair comparisons, we do not tune the coefficient of the repulsive force.
We refer to appendix \ref{App: Amortized_SVGD_Setup} for details.
\begin{table}[]
\begin{minipage}[t]{0.48\linewidth}\centering
\caption{Average log likelihood on first $5,000$ test images for different $D$ of latent dimensions.}
\label{tab: Amortized SVGD LL}
\vspace{-0.5em}
\begin{tabular}{c|llll}
\hline
Method                & \multicolumn{4}{c}{Latent Dim}                                                                           \\
\multicolumn{1}{l|}{} & \multicolumn{1}{c}{D=16} & \multicolumn{1}{c}{D=32} & \multicolumn{1}{c}{D=48} & \multicolumn{1}{c}{D=64} \\ \hline
Vanilla VAE           & -91.50                   & -90.39                   & -90.58                   & -91.50                   \\
SVGD VAE              & \textbf{-88.58}          & -90.43                   & -93.47                   & -94.88                   \\
S-SVGD VAE            & -89.17                   & \textbf{-87.55}          & \textbf{-87.74}          & \textbf{-87.78}          \\ \hline
\end{tabular}
\end{minipage}\hfill
\begin{minipage}[t]{0.4\linewidth}\centering
\caption{Label entropy and accuracy for imputed images. }
\label{tab: Label entropy}
\vspace{-0.5em}
\begin{tabular}{l|ll}
\hline
\multirow{2}{*}{Method} & \multirow{2}{*}{Entropy} & \multirow{2}{*}{Accuracy} \\
                        &                          &                           \\ \hline
Vanilla VAE             & 0.297                    & 0.718                     \\
SVGD VAE                & 0.538                    & 0.691                     \\
S-SVGD VAE              & \textbf{0.542}           & \textbf{0.728}            \\ \hline
\end{tabular}
\end{minipage}
\vspace{-1em}
\end{table}

Table \ref{tab: Amortized SVGD LL} reports the average test LL. We observe that S-SVGD is much more robust to the increasing latent dimensions compared to SVGD. To be specific, with $D=16$, SVGD performs the best where S-SVGD performs slightly worse than SVGD. However, when the dimension starts to increase, LL of SVGD drops significantly. For $D=64$, a common choice for latent space, it performs even significantly worse than vanilla VAE. On the other hand, S-SVGD is much more robust. {Notice that the purpose of this experiment is to show compare their robustness instead of achieving the state-of-the-art performance. Still the performance can be easily boosted, e.g. running longer S-SVGD steps before encoder update, we leave it for the future work.}

For the imputation task, we compute the label entropy and accuracy for the imputed images (Table \ref{tab: Label entropy}). We observe S-SVGD has higher label entropy compared to vanilla VAE and better accuracy compared to SVGD. This means both S-SVGD and SVGD capture the muli-modality nature of the posterior compared to uni-modal Gaussian distribution. However, high label entropy itself may not be a good indicator for the quality of the learned posterior. One can think of a counter-example that the imputed images are diverse but does not look like any digits. This may also gives a high label entropy but the quality of the posterior is poor. Thus, we use the accuracy to indicate the {``correctness''} of the imputed images, with higher label accuracy meaning the imputed images are closed to the original image. Together, a good model should give a higher label entropy along with the high label accuracy.  We observe S-SVGD has more diverse imputed images with high imputation accuracy. 

\subsection{summary of the experiments in appendix}
We present further empirical results on GOF tests and model learning in the appendix to demonstrate the advantages of the proposed maxSKSD. As a summary glance of the results:
\begin{itemize}
    \item In appendix \ref{APP: Pathology}, we analyse the potential limitations of maxSKSD-g and show that they can be mitigated by maxSKSD-rg, i.e.~optimising the slicing direction $\bm{r}$;
    \item In appendix \ref{App: Sampler Convergence}, we successfully apply maxSKSD to selecting the step size for \textit{stochastic gradient Hamiltonian Monte Carlo} (SGHMC) \citep{chen2014stochastic};
    \item In appendix \ref{App: Bayesian NN}, we show that the proposed S-SVGD approach out-performs the original SVGD on Bayesian neural network regression tasks.
\end{itemize}

\vspace{-3pt}
\section{Related Work}
\vspace{-6pt}
\paragraph{Stein Discrepancy}
SD \citep{gorham2015measuring} and KSD \citep{liu2016kernelized,chwialkowski2016kernel} are originally proposed for GOF tests. Since then research progress has been made to improve these two discrepancies.
For SD, LSD \citep{grathwohl2020cutting,hu2018stein} is proposed to increase the capacity of test functions using neural networks with $L_2$ regularization. 
On the other hand, FSSD \citep{jitkrittum2017linear} and RFSD \citep{huggins2018random} aim to reduce the computation cost of KSD from $O(n^2)$ to $O(n)$ where $n$ is the number of samples. Still the curse-of-dimensionality issue remains to be addressed in KSD, and the only attempt so far (to the best of our knowledge) is the \textit{kernelized complete conditional Stein discrepancy} (KCC-SD \citep{singhal2019kernelized}), which share our idea of avoiding kernel evaluations on high dimensional inputs but through comparing conditional distributions. 
KCC-SD requires the sampling from $q(x_d|\bm{x}_{-d})$, which often needs significant approximations in practice due to its intractability. This makes KCC-SD less suited for GOF test due to estimation quality in high dimensions. On the other hand, our approach does not require this approximation, and the corresponding estimator is well-behaved asymptotically. 
\vspace{-6pt}
\paragraph{Wasserstein Distance and Score matching}
Sliced Wasserstein distance (SWD) \citep{kolouri2016sliced} and sliced score matching (SSM) \citep{song2019sliced} also uses the ``slicing'' idea. However, their motivation is to address the computational issues rather than statistical difficulties in high dimensions. SWD leveraged the closed-form solution of 1D Wasserstein distance by projecting distributions onto 1D slices. SSM uses Hutchson's trick \citep{hutchinson1990stochastic} to approximate the trace of Hessian. 
\vspace{-6pt}
\paragraph{Particle Inference}
{\citet{zhuo2017message,wang2018stein}} proposed \textit{message passing SVGD} to tackle the well-known mode collapse problem of SVGD using local kernels in the graphical model. 
\wg{However, our work differs significantly in both theory and applications. Theoretically, the discrepancy behind their work is only valid if $p$ and $q$ have the same Markov blanket structure (refer to Section 3 in \cite{wang2018stein} for detailed discussion).  Thus, unlike our method, no GOF test and practical inference algorithm can be derived for generic cases. Empirically, the Markov blanket structure information is often unavailable, whereas our method only requires projections that can be easily obtained using optimizations. }
\textit{Projected SVGD} (pSVGD) is a very recent attempt \citep{chen2020projected} which updates the particles in an adaptively constructed low dimensional space, resulting in a biased inference algorithm. The major difference compared to S-SVGD is that our work still updates the particles in the original space with kernel being evaluated in 1D projections. Furthermore, S-SVGD can theoretically recover the correct target distribution. There is no real-world experiments provided in \citep{chen2020projected}, and a stable implementation of pSVGD is non-trivial, so we did not consider pSVGD when selecting the baselines.
\vspace{-3pt}
\section{Conclusion}
\vspace{-6pt}
We proposed sliced Stein discrepancy (SSD), as well as its scalable and kernelized version maxSKSD, to address the curse-of-dimensionality issues in Stein discrepancy. The key idea is to project the score function on one-dimensional slices and define (kernel-based) test functions on one-dimensional projections.
We also theoretically prove their validity as a discrepancy measure. We conduct extensive experiments including GOF tests and model learning to show maxSKSD's improved performance and robustness in high dimensions. There are three exciting avenues of future research. First, although validated by our theoretical study in appendix \ref{App: Deep Kernel}, practical approaches to incorporate deep kernels into SSD remains an open question. Second, the performance of maxSKSD crucially depends on the optimal projection direction, so better optimization methods to efficiently construct this direction is needed. {Lastly, we believe ``slicing'' is a promising direction for kernel design to increase the robustness to high dimensional problems in general. For example, MMD can be easily extended to high dimensional two-sample tests using this kernel design trick.}

\bibliography{iclr2021_conference}
\bibliographystyle{iclr2021_conference}

\clearpage

\appendix
\section{Definitions and Assumptions}
\label{Sec:Basics}
\begin{definition}{(\textbf{Stein Class} \citep{liu2016kernelized})}
Assume distribution $q$ has continuous and differentiable density $q(\bm{x})$. A function $f$ defined on the domain $\mathcal{X}\subset \mathbb{R}^D$, $f:\mathcal{X}\rightarrow \mathbb{R}$ is in the \textbf{{Stein class} of $q$} if {$f$ is smooth} and satisfies 
\begin{equation}
    \int_{\mathcal{X}}{\nabla_{x}(f(\bm{x})q(\bm{x}))d\bm{x}}=0
\end{equation}
\end{definition}
We can easily see that the above holds true for $\mathcal{X}=\mathbb{R}^D$ if 
\begin{equation}
    \lim_{||\bm{x}||\rightarrow\infty}{q(\bm{x})f(\bm{x})=0}
\end{equation}
This can be verified using integration by parts or divergence theorem. Specifically, if $q(\bm{x})$ vanishes at infinity, then it only requires the test function $f$ to be bounded. This definition can be generalized to a vector valued function $\bm{f}:\mathcal{X}\rightarrow \mathbb{R}^D$. We say such function $\bm{f}$ is in Stein class of $q$ if the member of $\bm{f}$, $f_i$, belongs to the Stein class of $q$ for all $i\in D$.

\begin{definition}{\citep{liu2016kernelized}}
A kernel $k(\bm{x},\bm{x}')$ is said to be in the {Stein class} of $q$ if $k(\bm{x},\bm{x}')$ has continuous second order partial derivatives, and both $k(\bm{x},\cdot)$ and $k(\cdot,\bm{x})$ are in the Stein class of $q$ for any fixed $\bm{x}$
\label{def: Kernel stein class}
\end{definition}

\paragraph{Radon Transform}
In machine learning literature, Radon transform has been used as the primary tool to derive sliced Wasserstein distance \citep{kolouri2019generalized,deshpande2019max,deshpande2018generative}. To be specific,
the standard Radon transform, denoted as $\mathcal{R}$, is a map from $L^1$ integrable functions $I\in L^1(\mathbb{R}^D)$ to the infinite set of its integrals over the hyperplane of $\mathbb{R}^D$. Specifically, for $L^1$ integrable functions:
\begin{equation}
    L^1(\mathbb{R}^D)=\{I:\mathbb{R}^D\rightarrow\mathbb{R}\ | \ \int_{\mathbb{R}^D}{|I(\bm{x})|d\bm{x}}<\infty\},
\end{equation}
the Radon transform is defined by
\begin{equation}
    \mathcal{R}[I](l,\bm{g})=\int_{\mathbb{R}^D}{I(\bm{x})\delta(l-\langle\bm{x},\bm{g}\rangle)d\bm{x}}
    \label{eq: Radon Transform}
\end{equation}
for $(l,\bm{g})\in \mathbb{R}\times \mathbb{S}^{D-1}$ where $\mathbb{S}^{D-1}\subset \mathbb{R}^D$ stands for a unit sphere in $\mathbb{R}^D$. For fixed $\bm{g}$, this defines a continuous function $\mathcal{R}[I](\cdot,\bm{g}):\mathbb{R}\rightarrow\mathbb{R}$ which is the projection of function $I$ on to the hyper-plane with its normal vector defined by $\bm{g}$ and offset defined by $l$. 

In the following we state the assumptions that we used to prove our main results.

\paragraph{Assumption 1}\label{assumption 1}(Properties of densities) Assume the two probability distributions $p$, $q$ has continuous differentiable density $p(\bm{x})$, $q(\bm{x})$ supported on $\mathbb{R}^D$.  Density $q$ satisfies: 
$\lim_{||\bm{x}||\rightarrow \infty}{q(\bm{x})= 0}$.
\paragraph{Assumption 2}\label{assumption 2}(Regularity of score functions) Denote the score function of $p(\bm{x})$ as $\bm{s}_p(\bm{x})=\nabla_{\bm{x}}\log p(\bm{x}) \in \mathbb{R}^D$ and score function of $q(\bm{x})$ accordingly. Assume the score functions satisfy 
\begin{equation}
    \begin{split}
        &\int_{\mathbb{R}^D}{q(\bm{x})|(s_p(\bm{x})-s_q(\bm{x}))^T\bm{r}|d\bm{x}}<\infty\\
        &\int_{\mathbb{R}^D}{q(\bm{x})||(s_p(\bm{x})-s_q(\bm{x}))^T\bm{r}||^2d\bm{x}}<\infty
    \end{split}
    \label{eq: assumption 2}
\end{equation}
for all $\bm{r}$
where $\bm{r}\in \mathbb{S}^{D-1}$ is a vector sampled from a uniform distribution over a unit ball $\mathbb{S}^{D-1}$. In other words, the score difference, when projected on the $\bm{r}$ direction, is both $L^1$ and $L^2$ integrable with respect to the probability measure defined by $q(\bm{x}) d\bm{x}$. These conditions are used to ensure both the Radon transform and the proposed divergence are well defined. 

\paragraph{Assumption 3}{\label{assumption 3}}(Stein Class of test functions) Assume the test function $f(\cdot;\bm{r},\bm{g}):\mathbb{R}^D\rightarrow \mathbb{R}$ is smooth and belongs to the Stein class of $q$.


\paragraph{Assumption 4}{\label{assumption 4}}(Bounded Radon transformed functions) Define 
\begin{equation}
    I_{q,p}=q(\bm{x})(s_p(\bm{x})-s_q(\bm{x}))^T\bm{r}
\end{equation}
We assume the Radon transformation of $I_{q,p}$, $\mathcal{R}[I_{q,p}](l,\bm{g})$ is bounded for all $\bm{g}$, where $\bm{g}$ is sampled from a uniform distribution over a unit ball $\mathbb{S}^{D-1}$. Namely, $||\mathcal{R}[I_{q,p}](l,\bm{g})||_{\infty}<\infty$

\paragraph{Assumption 5}{\label{assumption 5}}(Properties of kernels) For the RKHS $\mathcal{H}_{r,g}$ equipped with kernel function $k(\cdot,\cdot;\bm{r},\bm{g})$ defined as $k(\bm{x},\bm{x}';\bm{r},\bm{g})=k_{rg}(\bm{x}^T\bm{g},\bm{x}'^T\bm{g})$, we assume the kernel $k_{rg}$ is $C_0$-universal and $k$ belongs to the Stein class of $q$ for all $\bm{r}$ and $\bm{g}$. We further assume the kernel $k_{rg}$ is uniformly bounded such that $\sup_{\bm{x}}k_{rg}(\bm{x}^T\bm{g},\bm{x}^T\bm{g})<\infty$ for all $\bm{g}$ and $\bm{r}$. For example, the RBF kernel is a suitable choice for $k_{rg}$. 
\section{SSD Related propositions and theorems }
\subsection{Understanding the proposed Slice Stein Discrepancies}
\label{subsec: SSD Intuition}
We provide an explanation on the roles played by $\bm{r}$ and $\bm{g}$. We follow the same notations used in defining SSD (Eq.(\ref{eq: SSD})). The key idea to tackle the curse-of-dimensionality is to project both the score function $\bm{s}_p(\bm{x})\in\mathbb{R}^D$ and test function input $\bm{x}\in\mathbb{R}^D$. First, the \emph{slicing direction} $\bm{r}$ is introduced to project the score function, i.e. $s_p^r(\bm{x})=\bm{s}_p(\bm{x})^T\bm{r}$. By doing so, if $s_p^r(\bm{x})=s_q^r(\bm{x})$ for all $\bm{r}\in \mathbb{S}^{D-1}$, then $\bm{s}_p(\bm{s})=\bm{s}_q(\bm{x})$ and $p=q$ a.e. These equality conditions can be checked using Stein discrepancy (section \ref{sub:KSD}) by replacing $\bm{s}_p(\bm{x})$ with $s_p^r(\bm{x})$. Now it remains to address the scalability issue for the test functions as the score projection operation does not reduce the dimensionality of the test function input $\bm{x}$. In fact, using similar ideas from section \ref{sub:KSD}, the optimal test function to describe the difference between the projected score is proportional to $s_p^r(\bm{x})-s_q^r(\bm{x})$, which is an $\mathbb{R}^D\rightarrow\mathbb{R}$ function and thus it still utilizes the information in the original space $\bm{x} \in \mathbb{R}^D$. 

To resolve the high dimensionality of $\bm{x}$, it is preferred to use a test function that is defined on the one-dimensional input $\mathbb{R}\rightarrow \mathbb{R}$. However, using the projected input along the slicing direction $\bm{r}$ for the test function is insufficient to tell differences between the projected scores due to the information loss, as shown in the pathological example at the end of this section. Therefore, we need to find a way to express the projected score difference using a wide range of one-dimensional representations. 

{Our solution takes inspiration from the idea of CT-scans. To be precise, we test the difference of the projected score along a \emph{test direction} $\bm{g}$, by projecting $q(\bm{x})(s_p^r(\bm{x})-s_q^r(\bm{x}))$ to direction $\bm{g}$. {This is exactly the Radon transform of function $q(\bm{x})(s_p^r(\bm{x})-s_q^r(\bm{x}))$ in the direction $\bm{g}$, which is an $\mathbb{R}\rightarrow\mathbb{R}$ function with input $\bm{x}^{T}\bm{g}$.} Importantly, if the Radon transformed projected score difference is zero for all $\bm{g}$, the invertibility of Radon transform tell us the projected score difference is zero, and then $p=q$ a.e. if it holds true for all $\bm{r}$. Again this equality condition for the Radon transformed projected score difference can be checked in a similar way as in Stein discrepancy by defining test functions with input $\bm{x}^T \bm{g}$.}



To see why using a test direction $\bm{g}$ is necessary, we provide a counter-example in the case of using orthonormal slicing basis $O_r$. That is, if we set $\bm{g}=\bm{r}$ for $\bm{r} \in O_r$, then there exists a pair of distributions $p\neq q$ such that the following discrepancy equals to zero: 
\begin{equation}
    D(q,p)=\sum_{\bm{r}\in O_r}\sup_{f_r\in\mathcal{F}_q}\mathbb{E}_q[s_p^r(\bm{x})f_r(\bm{x}^T\bm{r})+\nabla_{\bm{x}^T\bm{r}}f_r(\bm{x}^T\bm{r})].
\end{equation}
To see this, we first select $O_r$ to be the standard orthonormal basis of $\mathbb{R}^D$ (i.e. the basis formed by one-hot vectors) w.l.o.g., as all the orthonormal basis in $\mathbb{R}^D$ are equivalent up to rotations or reflections. Now consider two probability distributions $p$ and $q$ supported on $\mathbb{R}^D$, where $p(\bm{x})=\prod_{i}^D{p(x_i)}$ and $q(x_i)=p(x_i)$. Importantly, $q$ distribution might not be factorized. Then we have
\begin{equation}
    \begin{split}
        D(q,p)&=\sum_{i}^D{\sup_{f_i\in\mathcal{F}_q}\mathbb{E}_{q_i}[s_p^i(x_i)f_i(x_i)+\nabla_{x_i}f_i(x_i)]} =0
    \end{split}
\end{equation}
where $s_p^i(x_i)=\nabla_{x_i}\log p(x_i)$ and $q_i=q(x_i)$. The second equality is from Stein identity due to the matching marginal of $p$ and $q$. However, it is not necessary that $p=q$, e.g. each dimensions in $q$ is correlated.  The main reason for this counter-example is that the test function only observes the marginal input $x_i$ and ignores any correlations that may exist in $q$.


\label{App:Theorem SSD non convergence}
\subsection{Proof of Theorem \ref{thm: SSD non-convergence}}
We split the proof of theorem \ref{thm: SSD non-convergence} into two parts. First, we prove the `if' part by the following proposition.
\begin{prop}{(SSD Detect Convergence)} If two distributions $p=q$ a.e., and assumption 1-4 are satisfied, then 
$S(q,p)=0$.
\label{prop: SSD convergence}
\end{prop}

\begin{proof}
To prove SSD can detect convergence of $q$ and $p$, we first introduce the Stein identity \citep{stein2004use,liu2016kernelized}.
\begin{lemma}{(Stein Identity)}
Assume $q$ is a smooth density satisfied assumption 1, then we have 
\[
\mathbb{E}_{q}[\bm{s}_q(\bm{x})f(\bm{x})^T+\nabla f(\bm{x})]=\bm{0}
\]
for any functions $f:\mathbb{R}^D\rightarrow\mathbb{R}^D$ in Stein class of $q$.
\end{lemma}
\noindent From the Stein identity, and $p=q$ a.e., we can take the trace of the Stein identity:
\[
\int{q(\bm{x})[\bm{s}_q(\bm{x})^TF(\bm{x})+\nabla^T_{\bm{x}}F(\bm{x})]d\bm{x}}=0
\]
where $F(\bm{x}):\mathbb{R}^D\rightarrow\mathbb{R}^D$ and it belongs to the Stein class of $q$.\\
Next, we choose a special form for $F(\bm{x})$. For particular sliced direction pair $\bm{r}$ and $\bm{g}$, we define
\[
F(\bm{x})=\left[
\begin{array}{c}
     r_1f_{r,g}(\bm{x}^T\bm{g})\\
     r_2f_{r,g}(\bm{x}^T\bm{g})\\
     \vdots\\
     r_Df_{r,g}(\bm{x}^T\bm{g})
\end{array}
\right]
\]
where $\bm{r}=[r_1,r_2,\ldots,r_D]^T$.

From the assumption 3 and definition of Stein class of $q$ for vector functions in section \ref{sub:KSD}, it is trivial that $F(\bm{x})$ belongs to the Stein class of $q$.
Substitute this $F(\bm{x})$ into Stein discrepancy Eq.(\ref{eq: Stein Discrepancy}), we have 
\[
\begin{split}
    &\int{q(\bm{x})[\bm{s}_q(\bm{x})^TF(\bm{x})+\nabla^T_{\bm{x}}F(\bm{x})]d\bm{x}}=0\\
    \Rightarrow&\int{q(\bm{x})[\bm{s}_q(\bm{x})^T\bm{r}f_{r,g}(\bm{x}^T\bm{g})+\bm{r}^T\bm{g}\nabla_{\bm{x}^T\bm{g}}f_{r,g}(\bm{x}^T\bm{g})]d\bm{x}}=0\\
\end{split}
\]
for all test functions $f_{r,g}$ that belongs to Stein class of $q$. 
Therefore, Eq.(\ref{eq: SSD}) is 0 if $p=q$ a.e.
\end{proof}

The 'only if' part of theorem \ref{thm: SSD non-convergence} is less direct to prove.
Before we start this journey, we need to introduce some properties relating to Radon transform.
\begin{lemma}{(Fourier Slice Theorem\citep{bracewell1956strip})}
For a particular smooth function $f(\bm{x}):\mathbb{R}^D\rightarrow\mathbb{R}$ that satisfies assumptions of Radon transforms, we define $\mathcal{F}_D$ as the D dimensional Fourier transform operator, $\mathcal{S}_1$ as a slice operator which extracts 1 dimensional central slice of a function and $\mathcal{R}$ as the Radon transform operator. Thus, for a slice direction $\bm{g}$, we have the following equivalence
\begin{equation}
    \mathcal{S}_1[\mathcal{F}_D[f]](\omega,\bm{g})=\mathcal{F}_1[\mathcal{R}[f](l,\bm{g})](\omega).
\end{equation}
\label{thm:Fourier Slice Theorem}
\end{lemma}
\noindent This theorem implies the following two operations are equivalent.
\begin{itemize}
    \item First apply $D$ dimensional Fourier transform to a function $f$ and then take a slice that goes through the origin with direction $\bm{g}$ from the transformed function.
    \item First apply the Radon transform with direction $\bm{g}$ to the function $f$ and then apply one dimensional Fourier transform to the projected function.
\end{itemize}
Next we show some properties related to the rotated or reflected distributions.
\begin{lemma}{(Marginalization Invariance of rotated distribution)}
Assume we have a probability distribution $q$ supported on $\mathbb{R}^D$, a orthogonal matrix $\bm{G}\in\mathbb{R}^{D\times D}$ and a test function $f:\mathbb{R}^D\rightarrow\mathbb{R}$, we can define the corresponding rotated distribution $q_G$ after applying the rotation matrix $\bm{G}$. Thus, we have the following identity
\begin{equation}
    \int{q_G(\bm{x})f(\bm{G}^{-1}\bm{x})d\bm{x}}=\int{q(\bm{x})f(\bm{x})d\bm{x}}.
\end{equation}
\label{lemma: Marginalization Invariance}
\end{lemma}
\begin{proof}
By the definition of rotation and change of variable formula, we define $\bm{y}=\bm{G}\bm{x}$, we can show
\[
\begin{split}
    q_G(\bm{y})&=q(\bm{x})|\bm{G}^{-1}|\\
    &=q(\bm{G}^{-1}\bm{y})\times 1\\
    &=q(\bm{G}^{-1}\bm{y})
\end{split}
\]
where $|\bm{G}^{-1}|$ represents the determinant of the inverse rotation matrix. Thus, by change of variable formula, we have
\[
\begin{split}
    &\int{q_{G}(\bm{y})f(\bm{G}^{-1}\bm{y})d\bm{y}}\\
    &=\int{q(\bm{G}^{-1}\bm{y})f(\bm{G}^{-1}\bm{y})d\bm{y}}\\
    &=\int{q(\bm{x})f(\bm{x})|\bm{G}|d\bm{x}}\\
    &=\int{q(\bm{y})f(\bm{y})d\bm{y}}.
\end{split}
\]
\end{proof}
\noindent This identity is useful when dealing with the rotated distributions. 
Next, we introduce the generalization of change-of-variable formula, which is often used in differential geometry.
\begin{lemma}{(Change of Variable Formula using Matrix Volume \citep{ben1999change})}
If $\mathcal{U}$ and $\mathcal{V}$ are sets in spaces with different dimensions, say $\mathcal{U}\in\mathbb{R}^n$ and $\mathcal{V}\in\mathbb{R}^m$ with $n>m$, and $\phi:\mathcal{U}\rightarrow\mathcal{V}$ is a continuously differentiable injective function and $f:\mathbb{R}^m\rightarrow\mathbb{R}$ is integrable on $\mathcal{V}$, we have the following change of variable formula:
\begin{equation}
    \int_{\mathcal{V}}{f(\bm{v})d\bm{v}}=\int_{\mathcal{U}}{(f\circ\phi)(\bm{u})\text{vol}J_{\phi}(\bm{u})d\bm{u}}
\end{equation}
where $\text{vol}J_{\phi}(\bm{u})$ is the matrix volume of the Jacobian matrix $J_{\phi}(\bm{u})=\partial(v_1,\ldots,v_m)/\partial(u_1,\ldots,u_n)$. Particularly, if $J_{\phi}(\bm{u})$ is of full column rank, then $\text{vol} J_{\phi}=\sqrt{\det J_{\phi}^TJ_{\phi}}$.
\label{lemma: Change of Variable Matrix Volume}
\end{lemma}
Next, we derive the key lemma that establishes the relationship between the conditional expectation of rotated distribution and Radon transform of the original distribution. 

\begin{lemma}{(Conditional Expectation = Radon Transform)}
For a particular test direction $\bm{g}_d \in \mathbb{S}^{D-1}$, we can define an arbitrary rotation matrix $\bm{G}\in\mathbb{R}^{D\times D}$ that the $d^{\text{th}}$ entry is the test direction $\bm{g}_d$. We assume the probability distribution $q(\bm{x})$ is supported on $\mathbb{R}^{D}$, and $\bm{x}_{-d}$ represents $\bm{x}\backslash x_{d}$ (all elements of $\bm{x}$ except $x_d$). Further, let define the mapping $\bm{x}=\bm{G}\bm{u}$ for $\bm{u}\in\mathbb{R}^{D}$ and $x_d$ is a constant $x_d=p$. Thus, with the smooth test function $f:\mathbb{R}^D\rightarrow\mathbb{R}$ and the assumptions in Radon transformation being true, we have the following identity:
\begin{equation}
    \int_{\mathcal{X}_d}{q_G(x_d,\bm{x}_{-d})f(\bm{G}^{-1}\bm{x})d\bm{x}_{-d}}=\int{q(\bm{u})f(\bm{u})\delta(p-\bm{u}^T\bm{g}_d)d\bm{u}}
\end{equation}
where $\mathcal{X}_d=\{\bm{x}\in\mathbb{R}^D|x_d=p\}$.
\label{lemma: Conditional marginalization}
\end{lemma} 
\begin{proof}
\noindent From the definition of $\bm{x}=\bm{G}\bm{u}$, we can define the rotation matrix $\bm{G}$ as following:
\[
\bm{G}=\left[
\begin{array}{c}
     \bm{g}_1^T  \\
     \vdots \\
     \bm{g}_d^T\\
     \vdots\\
     \bm{g}_D^T
\end{array}
\right]
\]
where $\bm{g}_d=[g_{d,1},\ldots,g_{d,D}]^T$. Thus, assume $x_d=p$, we can write down 
\[
\bm{x}=\bm{G}\bm{u}=\left[
\begin{array}{c}
     \sum_{k=1}^D{g_{1,k}u_{k}}\\
     \vdots\\
     p\\
     \vdots\\
     \sum_{k=1}^D{g_{D,k}u_k}
\end{array}
\right].
\]
Thus, the Jacobian matrix can be written as 
\[
J=\frac{\partial(\bm{G}\bm{u})}{\partial\bm{u}}=\left[
\begin{array}{cccccc}
     g_{1,1}&\ldots&g_{d-1,1}&g_{d+1,1}&\ldots&g_{D,1}  \\
     \vdots&\vdots &\vdots&\vdots&\vdots&\vdots\\
     g_{1,D}&\ldots&g_{d-1,D}&g_{d+1,D}&\ldots&g_{D,D}
\end{array}
\right]=\left[
\begin{array}{cccccc}
     \bm{g}_1&\ldots&\bm{g}_{d-1}&\bm{g}_{d+1}&\ldots&\bm{g}_{D}
\end{array}
\right].
\]
By the definition of rotation matrix, the Jacobian matrix is clearly full column rank. Thus, from Lemma \ref{lemma: Change of Variable Matrix Volume}, we have
\[
\begin{split}
    \text{vol} J&=\sqrt{\det J^TJ}=\left[
\begin{array}{c}
     \bm{g}^T_{1}\\
     \ldots\\
     \bm{g}^T_{d-1}\\
     \bm{g}^T_{d+1}\\
     \ldots\\
     \bm{g}_{D}^T
\end{array}
\right]\left[
\begin{array}{cccccc}
     \bm{g}_1&\ldots&\bm{g}_{d-1}&\bm{g}_{d+1}&\ldots&\bm{g}_{D}
\end{array}
\right] =\bm{I},
\end{split}
\]
where $\bm{I}\in\mathbb{R}^{(D-1)\times(D-1)}$ is the identity matrix. Then we directly apply the results in Lemma \ref{lemma: Change of Variable Matrix Volume} and Lemma \ref{lemma: Marginalization Invariance}, we have
\[
\begin{split}
    &\int_{\mathcal{X}_d}{q_{G}(x_d,\bm{x}_{-d})f(\bm{G}^{-1}\bm{x})d\bm{x}_{-d}}\\
    &=\int_{\mathcal{X}_d}{q(\bm{G}^{-1}\bm{x})f(\bm{G}^{-1}\bm{x})\times1d\bm{x}_{-d}}\\
    &\overset{\text{Lemma \ref{lemma: Change of Variable Matrix Volume}}}{=} \int_{\mathcal{U}}{q(\bm{u})f(\bm{u})\text{vol} \bm{J}d\bm{u}}
\end{split}
\]
where $\mathcal{U}=\{\bm{u}\in\mathbb{R}^D|\bm{g}_d^T\bm{u}=p\}$.
Thus, we have
\[
\int_{\mathcal{X}_d}{q_G(x_d,\bm{x}_{-d})f(\bm{G}^{-1}\bm{x})d\bm{x}_{-d}}=\int{q(\bm{u})f(\bm{u})\delta(p-\bm{u}^T\bm{g}_d)d\bm{u}}
\].

\end{proof}

Now, we can prove the 'only if' part of Theorem \ref{thm: SSD non-convergence} using the above lemmas.
\begin{proof}

In order to prove equation (Eq.(\ref{eq: SSD})) being 0 implies $p=q$ a.e., the strategy is to construct a lower bound for Eq.(\ref{eq: SSD}) by choosing a particular test function. We also need to make sure this lower bound is greater or equal to 0 and is 0 only if $p=q$ a.e. Thus, if the Eq.(\ref{eq: SSD}) is 0, it implies the lower bound is 0 and $q=p$ a.e.\\
Consider the inner supreme inside the Eq.(\ref{eq: SSD}), by Proposition \ref{prop: SSD convergence}, we have
\begin{equation}
    \begin{split}
&\mathbb{E}_{q}[\bm{s}_p(\bm{x})^T\bm{r}f_{r,g}(\bm{x}^T\bm{g})+\bm{r}^T\bm{g}\nabla_{\bm{x}^T\bm{g}}f_{r,g}(\bm{x}^T\bm{g})]\\
&=\mathbb{E}_{q}[(\bm{s}_p(\bm{x})-\bm{s}_q(\bm{x}))^T\bm{r}f_{r,g}(\bm{x}^T\bm{g})].
\end{split}
\label{eq:steps 1}
\end{equation}

Now we apply the Lemma \ref{lemma: Marginalization Invariance}, and assume $\bm{y}=\bm{G}\bm{x}$ and $y_d=\bm{g}^T\bm{x}$, then, Eq.(\ref{eq:steps 1}) can be rewritten as 
\begin{equation}
    \begin{split}
        &\int{q(\bm{x})[\bm{s}_p(\bm{x})-\bm{s}_q(\bm{x})]^T\bm{r}f_{r,g}(\bm{x}^T\bm{g})d\bm{x}}\\
        &=\int{q_G(y_d,\bm{y_{-d}})[\nabla_{\bm{G}^{-1}\bm{y}}\log{\frac{p(\bm{G}^{-1}\bm{y})}{q(\bm{G}^{-1}\bm{y})}}]^T\bm{r}f_{r,g}(y_d)d\bm{y}_{-d}dy_d}.
    \end{split}
    \label{eq:steps 2}
\end{equation}
The next step is to choose a specific form for the test function $f_{r,g}(y_d)$. Define 
\begin{equation}
    f_{r,g}(y_d)=\int{q_G(y_d,\bm{y}_{-d})[\nabla_{\bm{G}^{-1}\bm{y}}\log{\frac{p(\bm{G}^{-1}\bm{y})}{q(\bm{G}^{-1}\bm{y})}}]^T\bm{r}d\bm{y}_{-d}}.
    \label{eq:Special test function}
\end{equation}
First, we need to make sure this selected test function indeed satisfies assumption 3, namely, it needs to be in the Stein class of $q$. By Lemma \ref{lemma: Conditional marginalization}, this selected test function can be re-written into 
\[
\begin{split}
    &\int{q_G(y_d,\bm{y}_{-d})[\nabla_{\bm{G}^{-1}\bm{y}}\log{\frac{p(\bm{G}^{-1}\bm{y})}{q(\bm{G}^{-1}\bm{y})}}]^T\bm{r}d\bm{y}_{-d}}=\int{q(\bm{x})[\nabla_{\bm{x}}{\log{\frac{p(\bm{x})}{q(\bm{x})}}}]^T\bm{r}\delta(y_d-\bm{x}^T\bm{g})d\bm{x}}\\
    &=\mathcal{R}[I_{q,p}](y_d,\bm{g}).
\end{split}
\]
This is exactly the Radon transform of the function $I_{q,p}=q(\bm{x})(\bm{s}_p(\bm{x})-\bm{s}_q(\bm{x}))^T\bm{r}$. By assumption 4, this Radon transform is bounded. Thus, together with assumption 1, we can show this Radon transformed function indeed belongs to the Stein class of $q$ \citep{liu2016kernelized}. 

Now by substituting this specific test function Eq.(\ref{eq:Special test function}) into Eq.(\ref{eq:steps 2}), and defining $\bm{u}=[u_1,\ldots,y_d,\ldots,u_D]^T$, we have
\begin{equation}
    \begin{split}
    &\int{q_G(y_d,\bm{y}_{-d})[\nabla_{\bm{G}^{-1}\bm{y}}\log{\frac{p(\bm{G}^{-1}\bm{y})}{q(\bm{G}^{-1}\bm{y})}}]^T\bm{r}\int{q_G(y_d,\bm{u}_{-d})[\nabla_{\bm{G}^{-1}\bm{u}}\log{\frac{p(\bm{G}^{-1}\bm{u})}{q(\bm{G}^{-1}\bm{u})}}]^T\bm{r}d\bm{u}_{-d}}d\bm{y}_{-d}dy_d}\\
    &=\int\left\{\int{q_G(y_d,\bm{y}_{-d})[\nabla_{\bm{G}^{-1}\bm{y}}\log{\frac{p(\bm{G}^{-1}\bm{y})}{q(\bm{G}^{-1}\bm{y})}}]^T\bm{r}d\bm{y}_{-d}}\right\}\\
    &\left\{\int{q_G(y_d,\bm{u}_{-d})[\nabla_{\bm{G}^{-1}\bm{u}}\log{\frac{p(\bm{G}^{-1}\bm{u})}{q(\bm{G}^{-1}\bm{u})}}]^T\bm{r}d\bm{u}_{-d}}\right\}dy_d\\
    &=\int{f_{r,g}^2(y_d)dy_d}=\circled{1}\\
    &\geq0.
\end{split}
\label{eq:step 3}
\end{equation}
Thus we have constructed a lower bound (Eq.(\ref{eq:step 3})) for the supremum in Eq.(\ref{eq: SSD}) and it is greater than 0. Next, we show the expectation of this lower bound over $p_g$ and $p_r$ is 0 only if $p=q$ a.e.. ,If so then Eq.(\ref{eq: SSD}) is $0$ only if $p=q$ a.e..

First, it is clearly that $\circled{1}=0$ iff. $f_{r,g}(y_d)=0$ a.e.
By Lemma \ref{lemma: Conditional marginalization}, we have $f_{r,g}(y_d)=\mathcal{R}[I_{q,p}](y_d,\bm{g})$. 
Thus, we have
\[
\circled{1}=0 \quad \Rightarrow \quad \mathcal{R}[I_{q}](y_d,\bm{g})=\mathcal{R}[I_{p}](y_d,\bm{g}) \ a.e.
\]
where $I_q=q(\bm{x})\bm{s}_q(\bm{x})^T\bm{r}$ and $I_p=q(\bm{x})\bm{s}_p(\bm{x})^T\bm{r}$.

Now we define the $D$ dimensional Fourier transform operator $\mathcal{F}_D$, slice operator $\mathcal{S}_1$ as in Theorem \ref{thm:Fourier Slice Theorem}. Based on Fourier sliced lemma \ref{thm:Fourier Slice Theorem}, we have
\begin{equation}
\begin{split}
    &\mathcal{R}[I_{q}](y_d,\bm{g})=\mathcal{R}[I_{p}](y_d,\bm{g})\\
    \Rightarrow&\mathcal{F}_1[\mathcal{R}[I_{q}](y_d,\bm{g})]=\mathcal{F}_1[\mathcal{R}[I_{p}](y_d,\bm{g})]\\
    \Rightarrow&\mathcal{S}_1[\mathcal{F}_D[I_q]](\cdot,\bm{g})=\mathcal{S}_1[\mathcal{F}_D(I_p)](\cdot,\bm{g}).
\end{split}
\label{eq:step 4}
\end{equation}
This means the one dimensional slice at direction $\bm{g}$ for Fourier transform $\mathcal{F}_D(I_{q})$ and $\mathcal{F}_D(I_{p})$ are the same. Also note that the discrepancy (Eq.(\ref{eq: SSD})) is defined by integrating over test directions $\bm{g}$ with a uniform distribution $p_g(\bm{g})$ over $\mathbb{S}^{D-1}$. This means if the discrepancy is zero, then Eq.(\ref{eq:step 4}) must hold true for $\bm{g}$ a.e. over the hyper-sphere. Thus, we can show 
\begin{equation}
    \mathcal{F}_D(I_q)=\mathcal{F}_D(I_p) \;\;\;\; \text{a.e.}
    \label{eq:fourier equivalence}
\end{equation}
It is well-known that the Fourier transform is injective, thus, for any direction $\bm{r}$, we have
\begin{equation}
\begin{split}
    &\mathcal{F}_D(I_q)=\mathcal{F}_D(I_p)\\
    \Rightarrow&I_q=I_p\\
    \Rightarrow&q(\bm{x})\bm{s}_q(\bm{x})^T\bm{r}=q(\bm{x})\bm{s}_p(\bm{x})^T\bm{r}\\
    \Rightarrow&\bm{s}_q(\bm{x})^T\bm{r}=\bm{s}_p(\bm{x})^T\bm{r}
\end{split}
    \label{eq:marginalization equivalence}
\end{equation}
The $\mathcal{S}(q,p)$ (Eq.(\ref{eq: SSD})) also integrates over sliced directions $\bm{r} \in \mathbb{S}^{D-1}$, thus, we have
\[
   \bm{s}_q(\bm{x})^T\bm{r}=\bm{s}_p(\bm{x})^T\bm{r} \text{ for all }\bm{r} \quad \Rightarrow \quad \bm{s}_q(\bm{x})=\bm{s}_p(\bm{x}) \quad \Rightarrow \quad p=q \;\;\;\text{a.e.}
\]
This finishes the proof of the ``only if'' part: $\mathcal{S}(q,p)\geq0$ and is $0$ only if $q=p$ a.e.
\end{proof}

\subsection{Proof of Corollory \ref{coro: maxSSD}}
To prove the corollory \ref{coro: maxSSD}, we first propose a variant of SSD (Eq.(\ref{eq: SSD})) by relaxing the score projection $\bm{r}$. We call it \textit{orthogonal basis SSD}. 

\begin{rem}{(Orthogonal basis for SSD)}
It is not necessary to integrate over all possible $\bm{r}\in\mathbb{S}^{D-1}$ for Theorem \ref{thm: SSD non-convergence} to hold true. In fact, it suffices to use a set of projections that forms the orthogonal basis $O_r$ of $\mathbb{R}^D$. In such case we have
\begin{equation}
    S_o(q,p)=\sum_{\bm{r}\in O_r}{\int_{\mathbb{S}^{D-1}}{p_g(\bm{g})\sup_{f_{rg}\in\mathcal{F}_q}{\mathbb{E}_q[{s}^r_p(\bm{x})f_{rg}(\bm{x}^T\bm{g})+\bm{r}^T\bm{g}\nabla_{\bm{x}^T\bm{g}}f_{rg}(\bm{x}^T\bm{g})]d\bm{g}}}}
    \label{eq:orthogonal SSD}
\end{equation}
is zero if and only if $p=q$ a.e. One simple choice for $O_r$ can be $O_r=\{\bm{r}_1,\ldots,\bm{r}_D\}$ where $\bm{r}_d$ is one-hot vector with value 1 in d$^{\text{th}}$ component.
\label{rem: orthogonal SSD}
\end{rem}
To prove Remark \ref{rem: orthogonal SSD}, we only need to slightly modify the last few steps in the proof of Theorem \ref{thm: SSD non-convergence}.

\begin{proof}
We focus on the `only if' part as the other part is trivial. Without loss of generality, we set $O_r=\{\bm{r}_1,\ldots,\bm{r}_D\}$ where $\bm{r}_d$ is one-hot vector with value 1 in $i^{\text{th}}$ component. For general $O_r$, we can simply apply a inverse rotation matrix $\bm{R}^{-1}$ to recover this special case. 

From Eq.(\ref{eq:marginalization equivalence}), we have for direction $\bm{r}_d$,
\[
    \begin{split}
        &\bm{s}_q(\bm{x})^T\bm{r}_d=\bm{s}_p(\bm{x})^T\bm{r}_d\\
        \Rightarrow&\nabla_{x_d}\log q(x_d,\bm{x}_{-d})=\nabla_{x_d}\log p(x_d,\bm{x}_{-d})\\
        \Rightarrow&\nabla_{x_d}\log q(x_d|\bm{x}_{-d})=\nabla_{x_d}\log p(x_d|\bm{x}_{-d}).\\
    \end{split}
\]
If the above holds true for all directions $\bm{r}_d \in O_r$, then the score of the complete conditional for $q$ and $p$ are equal. Then from Lemma 1 in \citep{singhal2019kernelized}, we have $p=q$ a.e.
\end{proof}
Now we can prove Corollory \ref{coro: maxSSD} using Remark \ref{rem: orthogonal SSD}.
\begin{proof}
It is trivial to show $S_{max}(q,p)=0$ if $p=q$ a.e. (Stein Identity). Now assume $S_{max}(q,p)=0$, this means for any direction $\bm{r} \in O_r$, and $\bm{g}\in\mathbb{S}^{D-1}$, we have
\[
\sup_{f_{rg}\in\mathcal{F}_q}{\mathbb{E}_q[{s}^r_p(\bm{x})f_{rg}(\bm{x}^T\bm{g})+\bm{r}^T\bm{g}_r\nabla_{\bm{x}^T\bm{g}_r}f_{rg}(\bm{x}^T\bm{g})]}=0
\]
This is because we have show in the proof of Theorem \ref{thm: SSD non-convergence} that the above term is greater or equal to 0. Then we can directly use Remark \ref{rem: orthogonal SSD} to show $S_{max}(q,p)=0$ only if $q=p$ a.e.
\end{proof}

\section{SKSD Related Theorems}
\subsection{Proof of Theorem \ref{thm: SKSD Analytic form}}
\label{App:SKSD Analytic form}
\begin{proof}
First, we can verify the following equality using the proof techniques in \citep{liu2016kernelized,chwialkowski2016kernel}:
\begin{equation}
    h_{p,r,g}(\bm{x},\bm{y})=\langle\xi_{p,r,g}(\bm{x},\cdot),\xi_{p,r,g}(\bm{y},\cdot)\rangle_{\mathcal{H}_{rg}}.
\end{equation}
Next, we show that $\xi_{p,r,g}(\bm{x},\cdot)$ is Bochner integrable \citep{christmann2008support}, i.e.
\begin{equation}
    \mathbb{E}_q||\xi_{p,r,g}(\bm{x})||_{\mathcal{H}_{rg}}\leq\sqrt{\mathbb{E}_q||\xi_{p,r,g}(\bm{x})||^2_{\mathcal{H}_{rg}}}=\sqrt{\mathbb{E}_q[h_{p,r,g}(\bm{x},\bm{x})]}\leq \infty.
\end{equation}
Thus, we can interchange the expectation and the inner product. Finally we finish the proof by re-writing the supremum in $S_o(q,p)$: (Eq.(\ref{eq:orthogonal SSD}))
\begin{equation}
    \begin{split}
        &||\sup_{f_{rg}\in\mathcal{H}_{rg},||f_{rg}||\leq 1}\mathbb{E}_q[{s}^r_p(\bm{x})f_{rg}(\bm{x}^T\bm{g})+\bm{r}^T\bm{g}\nabla_{\bm{x}^T\bm{g}}f_{rg}(\bm{x}^T\bm{g})]||^2\\
        =&||\sup_{f_{rg}\in\mathcal{H}_{rg},||f_{rg}||\leq 1}{\mathbb{E}_q[\langle{s}^r_p(\bm{x})k_{rg}(\bm{x}^T\bm{g},\cdot)+\bm{r}^T\bm{g}\nabla_{\bm{x}^T\bm{g}}k_{rg}(\bm{x}^T\bm{g},\cdot),f_{rg}\rangle_{\mathcal{H}_{rg}}}]||^2\\
        =&||\sup_{f_{rg}\in\mathcal{H}_{rg},||f_{rg}||\leq 1}\langle f_{rg},\mathbb{E}_q[{s}^r_p(\bm{x})^T\bm{r}k_{rg}(\bm{x}^T\bm{g},\cdot)+\bm{r}^T\bm{g}\nabla_{\bm{x}^T\bm{g}}k_{rg}(\bm{x}^T\bm{g},\cdot)]\rangle_{\mathcal{H}_{rg}}||^2\\
        =&||\mathbb{E}_q[\xi_{p,r,g}(\bm{x})]||^2_{\mathcal{H}_{rg}}\\
        =&\langle\mathbb{E}_q[\xi_{p,r,g}(\bm{x},\cdot)],\mathbb{E}_q[\xi_{p,r,g}(\bm{x}',\cdot)]\rangle_{\mathcal{H}_{rg}}\\
        =&\mathbb{E}_{\bm{x},\bm{x}'\sim q}[h_{p,r,g}(\bm{x},\bm{x}')].
    \end{split}
\end{equation}
\end{proof}
\subsection{Proof of Theorem \ref{thm: Validity of SKSD}}
\label{App: Validity of SKSD}
\begin{proof}
First, we assume $p=q$ a.e. To show $SK_o(q,p)=0$, we only need to show $D^2_{rg}(q,p)=0$ for all $\bm{r}$ and $\bm{g}$. From Theorem \ref{thm: SKSD Analytic form}, we have
\[
D^2_{r,g}(q,p)=\langle\mathbb{E}_q[\xi_{p,r,g}(\bm{x},\cdot)],\mathbb{E}_q[\xi_{p,r,g}(\bm{x}',\cdot)]\rangle.
\]
From Assumption 5, we know $k_{rg}(\bm{x}^T\bm{g},\cdot)$ belongs to the Stein class of $q$. Then we follow the same proof technique in Proposition \ref{prop: SSD convergence} but replace the test function $f_{rg}(\bm{x}^T\bm{g})$ with $k_{rg}(\bm{x}^T\bm{g},\cdot)$. This gives 
\begin{equation}
    \mathbb{E}_q[{s}^r_q(\bm{x})^T\bm{r}k_{rg}(\bm{x}^T\bm{g},\cdot)+\bm{r}^T\bm{g}\nabla_{\bm{x}^T\bm{g}}k_{rg}(\bm{x}^T\bm{g},\cdot)]=0,
    \label{eq: Stein identity in kernel}
\end{equation}
i.e. $\mathbb{E}_q[\xi_{p,r,g}(\bm{x},\cdot)]=0$. Thus, $D^2_{rg}(q,p)=0$.

Next, we prove that it can detect the non-convergence of $p$ and $q$.  We know $SK_o(q,p)=0$ if and only if $D^2_{rg}(q,p)=0$. This means
\[
\begin{split}
    &D_{rg}(q,p)=0\\
    \Rightarrow&||\mathbb{E}_q[\xi_{p,r,g}(\bm{x})]||_{\mathcal{H}_{rg}}=0\\
    \Rightarrow&\mathbb{E}_q[\xi_{p,r,g}(\bm{x},\cdot)]=0
\end{split}
\]
where the second equality is from theorem \ref{thm: SKSD Analytic form}. From Eq.(\ref{eq: Stein identity in kernel}), we can re-write 
\[
\mathbb{E}_q[\xi_{p,r,g}(\bm{x},\cdot)]=\mathbb{E}_q[({s}^r_p(\bm{x})-{s}^r_q(\bm{x}))k_{rg}(\bm{x}^T\bm{g},\cdot)].
\]
Next, we denote $\bm{G}$ as an arbitrary rotation with the $d$th entry as the test direction $\bm{g}$, and $\bm{y}=\bm{G}\bm{x}$ with $y_d=\bm{x}^T\bm{g}$. Then from Lemma \ref{lemma: Marginalization Invariance}, we have
\[
\begin{split}
    &\int{q(\bm{x})\nabla_{\bm{x}}\log\frac{p(\bm{x})}{q(\bm{x})}^T\bm{r}k_{rg}(\bm{x}^T\bm{g},\cdot)d\bm{x}}\\
    =&\int{q_{G}(y_d,\bm{y}_{-d})\nabla_{\bm{G}^{-1}\bm{y}}\log\frac{q(\bm{G}^{-1}\bm{y})}{p(\bm{G}^{-1}\bm{y})}^T\bm{r}k_{rg}(y_d,\cdot)d\bm{y}_{-d}dy_d}\\
    =&\int{q_{G}(y_d)k_{rg}(y_d,\cdot)\int{q_G(\bm{y}_{-d}|y_d)\nabla_{\bm{G}^{-1}\bm{y}}\log\frac{q(\bm{G}^{-1}\bm{y})}{p(\bm{G}^{-1}\bm{y})}^T\bm{r}d\bm{y}_{-d}}dy_d}\\
    =&\int{q_{G}(y_d)k_{rg}(y_d,\cdot)H_r(y_d)dy_d}
    \end{split}
\]
where $H_r(y_d)=\int{q_G(\bm{y}_{-d}|y_d)\nabla_{\bm{G}^{-1}\bm{y}}\log\frac{q(\bm{G}^{-1}\bm{y})}{p(\bm{G}^{-1}\bm{y})}^T\bm{r}d\bm{y}_{-d}}$.
The above equation is exactly the mean embedding of the function $H_r(y_d)$ w.r.t. measure $q_G$. By assumption 5 that the kernel is $C_0$-universal, and by \cite{carmeli2010vector}, its embedding is zero if and only if $H_r(\cdot)=0$. This implies 
\[
\begin{split}
    &H_r(y_d)=\int{q_G(\bm{y}_{-d}|y_d)\nabla_{\bm{G}^{-1}\bm{y}}\log\frac{q(\bm{G}^{-1}\bm{y})}{p(\bm{G}^{-1}\bm{y})}^T\bm{r}d\bm{y}_{-d}}=0\\
    \Rightarrow&\int{q_G(y_d,\bm{y}_{-d})\nabla_{\bm{G}^{-1}\bm{y}}\log\frac{q(\bm{G}^{-1}\bm{y})}{p(\bm{G}^{-1}\bm{y})}^T\bm{r}d\bm{y}_{-d}}=0\\
    \Rightarrow&\int{q(\bm{x})({s}^r_p(\bm{x})-{s}^r_q(\bm{x}))\delta(y_d-\bm{x}^T\bm{g})d\bm{x}}=0
\end{split}
\]
where the third equality is from Lemma \ref{lemma: Conditional marginalization}. Then we can follow the same proof technique in Theorem \ref{thm: SSD non-convergence} and remark \ref{rem: orthogonal SSD} to show $SK_o(q,p)=0$ only if $p=q$ a.e.
\end{proof}
\section{Deep Kernel}
\label{App: Deep Kernel}
Using deep kernels for KSD is straight-forward and it only requires the deep kernel to be characteristic. But a naive application of deep kernels to SKSD or maxSKSD would result in a kernel evaluated on $\phi(\bm{x}^T\bm{g})$, which is less desirable. To make the kernel evaluated on the transformed input $\phi(\bm{x})^T\bm{g}$, we need to slightly adapt the form of SKSD (Eq.\ref{eq: SKSD}). 
Assume we have a smooth injective mapping $\phi$, we define the following term 
\begin{equation}
    \xi_{p,r,g,\phi}(\bm{x},\cdot)=s_p^r(\bm{x})k_{rg}(\phi^g(\bm{x}),\cdot)+C_{\phi}(\bm{x})\nabla_{\phi^g(\bm{x})}k_{rg}(\phi^g(\bm{x}),\cdot)
    \label{eq: Deep kernel SKSD test function}
\end{equation}
and 
\begin{equation}
    \begin{split}
        h_{p,r,g,\phi}&(\bm{x},\bm{y})=\\
        &s_p^r(\bm{x})k_{rg}(\phi^g(\bm{x}),\phi^g(\bm{y}))s_p^r(\bm{y})+C_{r,g,\phi}(\bm{x})s_p^r(\bm{y})\nabla_{\phi^g(\bm{x})}k_{rg}(\phi^g(\bm{x}),\phi^g(\bm{y}))\\
        &+C_{r,g,\phi}(\bm{y})s_p^r(\bm{x})\nabla_{\phi^g(\bm{y})}k_{rg}(\phi^g(\bm{x}),\phi^g(\bm{y}))\\
        &+C_{r,g,\phi}(\bm{x})C_{r,g,\phi}(\bm{y})\nabla^2_{\phi^g(\bm{x}),\phi^g(\bm{y})}k_{rg}(\phi^g(\bm{x}),\phi^g(\bm{y}))
    \end{split}
    \label{eq: Deep kernel SKSD inner}
\end{equation}
where $C_{r,g,\phi}(\bm{x})=\bm{r}^T\frac{\partial \phi(\bm{x})}{\partial \bm{x}}\bm{g}$ and $\phi^g(\bm{x})=\phi(\bm{x})^T\bm{g}$. We provide the following theorem to prove the validity of the corresponding SKSD discrepancy measure.
\begin{theorem}{(Deep Kernel SKSD)}
For two probability distributions $p$ and $q$, assume we have a smooth injective mapping $\phi(\bm{x})$, {such that} Assumptions 1,2 and 5 are satisfied, and $\mathbb{E}_q[h_{p,r,g,\phi}(\bm{x},\bm{x})]<\infty$ for all $\bm{r}$ and $\bm{g}$, then we propose deep kernel SKSD (Deep-SKSD) as
\begin{equation}
    DSK_o(q,p)=\sum_{\bm{r}\in O_r}{\int_{\mathbb{S}^{D-1}}{p_g(\bm{g})D^2_{r,g,\phi}(q,p)d\bm{g}}}, \quad D^2_{r,g,\phi}(q,p)=\mathbb{E}_q[h_{p,r,g,\phi}(\bm{x},\bm{x}')],
    \label{eq: deep SKSD}
\end{equation}
and it is 0 if and only if $p=q$ a.e..
\label{thm: Deep Kernel}
\end{theorem}

Deep-SKSD (Eq.(\ref{eq: deep SKSD})) can be viewed as a generalization of SKSD (Eq.(\ref{eq: SKSD})). Specifically, SKSD can be recovered using Deep-SKSD with $\phi$ as the identity mapping.
\subsection{Theorem \ref{thm: Deep Kernel}}
\label{App: Proof of deep kernel}
\begin{proof}
We follow the proof of Theorem \ref{thm: SKSD Analytic form} to show 
\[
h_{p,r,g,\phi}(\bm{x},\bm{y})=\langle\xi_{p,r,g,\phi}(\bm{x},\cdot),\xi_{p,r,g,\phi}(\bm{y},\cdot)\rangle_{\mathcal{H}_{rg}}.
\]
By Assumption 5, $k_{rg}$ belongs to the Stein class of $q$. Thus, $k_{rg}(\phi^g(\bm{x}),\cdot)$ belongs to the Stein class of $q$. This can be easily verified by using the definition of Stein class of $q$, and the facts that the kernel function is bounded and $q$ vanishes at boundary. Now we follow the proof in proposition \ref{prop: SSD convergence} to define 
\[
F(\bm{x})=\left[
\begin{array}{c}
     r_1k_{rg}(\phi^g(\bm{x}),\cdot)\\
     r_2k_{rg}(\phi^g(\bm{x}),\cdot)\\
     \vdots\\
     r_Dk_{rg}(\phi^g(\bm{x}),\cdot)\\
\end{array}
\right]
\]
\noindent and substitute this into Stein identity. This returns
\[
\begin{split}
    &\int{q(\bm{x})[{s}^r_q(\bm{x})k_{rg}(\phi^g(\bm{x}),\cdot)+C_{r,g,\phi}(\bm{x})\nabla_{\phi^g(\bm{x})}k_{rg}(\phi^g(\bm{x}),\cdot)]}d\bm{x}=0\\
    \Rightarrow&\mathbb{E}_q[\xi_{q,r,g,\phi}(\bm{x},\cdot)]=0\\
    \Rightarrow&D^2_{r,g,\phi}(q,q)=\mathbb{E}_q[h_{q,r,g,\phi}(\bm{x},\bm{x}')]=\langle\mathbb{E}_q[\xi_{q,r,g,\phi}(\bm{x},\cdot)],\mathbb{E}_q[\xi_{q,r,g,\phi}(\bm{x}',\cdot)]\rangle_{\mathcal{H}_{rg}}=0.
\end{split}
\]
Therefore, if $p=q$ a.e., then $DSK_o(q,p)=0$.

Now we prove $DSK_o(q,p)=0$ only if $p=q$ a.e.. It is trivial that $DSK_o(q,p)=0$ if and only if $D^2_{r,g,\phi}(q,p)=0$. In other words, 
\[
DSK_o(q,p)=0 \quad \Rightarrow \quad \mathbb{E}_q[\xi_{p,r,g,\phi}(\bm{x},\cdot)]=0.
\]
Similar to the proof in Theorem \ref{thm: Validity of SKSD}, the RHS term above can be re-written as 
\[
\mathbb{E}_q[\xi_{p,r,g,\phi}(\bm{x},\cdot)]=\mathbb{E}_q[({s}^r_p(\bm{x})-{s}^r_q(\bm{x}))k_{rg}(\phi^g(\bm{x}),\cdot)].
\]
We denote $\bm{G}$ as an arbitrary rotation with the $d^{\text{th}}$ entry as the test direction $\bm{g}$. We also define $\bm{y}=\phi(\bm{x})$ and $\bm{u}=\bm{G}\bm{y}$ with $u_d=\bm{y}^T\bm{g}$. Thus, by the change of variable formula, we have 
\begin{equation}
    \begin{split}
        q_{\phi}(\bm{y})&=q(\bm{x})|\bm{J}|^{-1},\\
        q_{G\phi}(\bm{u})&=q_{\phi}(\bm{y})|\bm{G}|^{-1}=q_{\phi}(\bm{y}),
    \end{split}
    \label{eq: deep SKSD change of variable}
\end{equation}
where $\bm{J}$ is the Jacobian matrix $\frac{\partial \phi(\bm{x})}{\partial \bm{x}}$ and $|\cdot|$ is the determinant. Thus, we have
\[
\begin{split}
    &\int{q(\bm{x})[\nabla_{\bm{x}}\log\frac{p(\bm{x})}{q(\bm{x})}]^T\bm{r}k_{rg}(\phi^g(\bm{x}),\cdot)}d\bm{x}\\
    =&\int{q_{\phi}(\bm{y})|\bm{J}|[\nabla_{\phi^{-1}(\bm{y})}\log\frac{p(\phi^{-1}(\bm{y}))}{q(\phi^{-1}(\bm{y}))}]^T\bm{r}k_{rg}(\bm{y}^T\bm{g},\cdot)|\bm{J}|^{-1}d\bm{y}}\\
    =&\int{q_{G\phi}(u_d,\bm{u}_{-d})[\nabla_{\phi^{-1}(\bm{G}^{-1}\bm{u})}\log\frac{p(\phi^{-1}(\bm{G}^{-1}\bm{u}))}{q(\phi^{-1}(\bm{G}^{-1}\bm{u}))}]^T\bm{r}k_{rg}(u_d,\cdot)|\bm{G}^{-1}|du_dd\bm{u}_{-d}}\\
    =&\int{q_{G\phi}(u_d)k_{rg}(u_d,\cdot)\int{q_{G\phi}(\bm{u}_{-d}|u_d)[\nabla_{\phi^{-1}(\bm{G}^{-1}\bm{u})}\log\frac{p(\phi^{-1}(\bm{G}^{-1}\bm{u}))}{q(\phi^{-1}(\bm{G}^{-1}\bm{u}))}]^T\bm{r}d\bm{u}_{-d}}du_d}\\
    =&\int{q_{G\phi}(u_d)k_{rg}(u_d,\cdot)H(u_d)du_d}.
\end{split}
\]
Following the proof steps in Theorem \ref{thm: Validity of SKSD}, we have $H(u_d)=0$. Then by Lemma \ref{lemma: Conditional marginalization}, we have
\[
\begin{split}
    &\int{q_{G\phi}(\bm{u}_{-d},u_d)[\nabla_{\phi^{-1}(\bm{G}^{-1}\bm{u})}\log\frac{p(\phi^{-1}(\bm{G}^{-1}\bm{u}))}{q(\phi^{-1}(\bm{G}^{-1}\bm{u}))}]^T\bm{r}d\bm{u}_{-d}}\\
    =&\int{q_\phi(\bm{y})[\nabla_{\phi^{-1}(\bm{y})}\log\frac{p(\phi^{-1}(\bm{y}))}{q(\phi^{-1}(\bm{y}))}]^T\bm{r}\delta(u_d-\bm{y}^T\bm{g})d\bm{y}}.
\end{split}
\]
Finally using similar proof techniques in Theorem \ref{thm: SSD non-convergence}, we have 
\[
p(\phi^{-1}(\bm{y}))=q(\phi^{-1}(\bm{y})).
\]
As $\phi$ is injective, we have $p=q$ a.e.
\end{proof}
\section{Closed-form solutions for $\pmb{G}$}
\label{App: Closed form}
\wgr{In general, such closed-from solutions of $\bm{G}$ is difficult to find, and we have to resort to gradient-based optimization for such task. However, the closed-form solutions exists under certain conditions. In the following, we give the closed-form solution of $\bm{G}$ under conditions that $p$ and $q$ are full-factorized.\\
Let define two distributions $p$, $q$ with support $\mathbb{R}^D$, such that $\log p(\bm{x})=\sum_{d=1}^D{\log p_d(x_d)}$ and $\log q(\bm{x})=\sum_{d=1}^D{\log q_d(x_d)}$. Now we consider the \textit{maxSSD-g} (Eq.\ref{eq:maxSSD}) with $O_r$ to be a group of one-hot vectors. Thus, eq.\ref{eq:maxSSD} becomes 
\begin{equation}
    S_{max}(q,p)=\sum_{d=1}^D{\max_{f_d\in\mathcal{F}_q,\bm{g}_d\in \mathbb{S}^{D-1}}{\mathbb{E}_q[S_{p,d}(x_d)f_d(\bm{x}^T\bm{g}_d)+g_{d,d}\nabla_{\bm{x}^T\bm{g}_d}f_d(\bm{x}^T\bm{g}_d)]}}
    \label{eq: Closed form maxSSD g}
\end{equation}
where $S_{p,d}(x_d)=\nabla_{x_d}\log p_d(x_d)$ and $g_{d,d}$ is the $d^\text{th}$ element of $\bm{g}_d$. It is difficult to directly solve this optimization. Instead, we can find its upper bound, and show that such upper bound can be recovered by choosing a specific form of $\bm{G}$. Thus, such $\bm{G}$ will be the optimal one. \\
Let's consider the Stein divergence with test function $H(\bm{x}):\mathbb{R}^D\rightarrow \mathbb{R}^D$ and $H(\bm{x})=[h_1(\bm{x}),\ldots,h_D(\bm{x})]^T$. We have 
\begin{equation}
    \begin{split}
        SD(q,p)&=\max_{H\in\mathcal{F}_q}\mathbb{E}_q[\nabla_{\bm{x}} \log p(\bm{x})^TH(\bm{x})+\nabla_{\bm{x}}^TH(\bm{x})]\\
        &=\max_{H\in\mathcal{F}_q}\mathbb{E}_q[\sum_{d=1}^D{\nabla_{x_d}\log p_d(x_d)h_d(\bm{x})+\nabla_{x_d}h_d(\bm{x})}]\\
        &=\max_{H\in\mathcal{F}_q}{\sum_{d=1}^D \mathbb{E}_q[S_{p,d}(x_d)h_d(\bm{x})+\nabla_{x_d}h_d(\bm{x})]}\\
        &=\sum_{d=1}^D{\max_{h_d\in\mathcal{F}_q}{\mathbb{E}_q[S_{p,d}(x_d)h_d(\bm{x})+\nabla_{x_d}h_d(\bm{x})]}}
    \end{split}
    \label{eq: Closed form SD}
\end{equation}
where $S_{p,d}(x_d)=\nabla_{x_d}\log p_d(x_d)$. It is trivial that $SD(q,p)\geq S_{max}(q,p)$ because $h_d(\bm{x})$ is less restrictive than $f_d(\bm{x}^T\bm{g}_d)$. From \cite{hu2018stein}, we can obtain the optimal form for $H^*(\bm{x})\propto S_p(\bm{x})-S_q(\bm{x})$ where $S_p(\bm{x})=\nabla_{\bm{x}}\log p(\bm{x})$. Therefore, $h^*_d(\bm{x})=h^*_d(x_d)\propto S_{p,d}(x_d)-S_{q,d}(x_d)$. By substitution into eq.\ref{eq: Closed form SD}, we have
\begin{equation}
    SD(q,p)=\sum_{d=1}^D{{\mathbb{E}_q[S_{p,d}(x_d)h^*_d(x_d)+\nabla_{x_d}h^*_d(x_d)]}}
    \label{eq: Closed form SD final}
\end{equation}
We note that eq.\ref{eq: Closed form SD final} can be recovered by \textit{maxSSD-g} (eq.\ref{eq: Closed form maxSSD g}) with $\bm{g}_d=[0,\ldots,1_d,\ldots,0]^T$ where $d^{\text{th}}$ element $1_d=1$, and $f_d(x_d)\propto S_{p,d}(x_d)-S_{q,d}(x_d)$. Thus, the optimal $\bm{G}=\bm{I}$, which is an identity matrix.  }

\section{Applications of maxSKSD}
\label{APP: applications of maxSKSD}
\subsection{Goodness-of-fit test}
\label{APP: GOF Test}
We propose a Goodness-of-fit test method based on the U-statistics of maxSKSD (Eq.(\ref{eq: U maxSKSD})) given the optimal test direction $\bm{g}_r$. In the following we analyze the asymptotic behavior of the proposed statistic.
\begin{theorem}
Assume the conditions in Theorem \ref{thm: Validity of SKSD} are satisfied, we have the following:
\begin{enumerate}
    \item If $q\neq p$, then $\reallywidehat{SK}_{max}(q,p)$ is asymptotically normal. Particularly,
    \begin{equation}
        \sqrt{N}(\reallywidehat{SK}_{max}(q,p)-SK_{max}(q,p))\stackrel{d}{\rightarrow}\mathcal{N}(0,\sigma_{h}^2)
        \label{eq: Asymptotic maxSKSD q neq q}
    \end{equation}
    where $\sigma_{h}^2=\text{var}_{\bm{x}\sim q}(\sum_{\bm{r}\in O_r}{\mathbb{E}_{\bm{x}'\sim q}[h_{p,r,g_r}(\bm{x},\bm{x}')]})$ and $\sigma_h\neq 0$
    \item If $q=p$, we have a degenerated U-statistics with $\sigma_h=0$ and 
    \begin{equation}
        N\reallywidehat{SK}_{max}(q,p)\stackrel{d}{\rightarrow}\sum_{j=1}^\infty{c_j(Z_j^2-1)}
        \label{eq: Asymptotic maxSKSD q=p}
    \end{equation}
    where $\{Z_j\}$ are i.i.d standard Gaussian variables, and $\{c_j\}$ are the eigenvalues of the kernel $\sum_{\bm{r}\in O_r}{h_{p,r,g_r}(\bm{x},\bm{x}')}$ under $q(\bm{x})$. In other words, they are the solutions of $c_j\phi_j(\bm{x})=\int_{\bm{x}'}{\sum_{\bm{r}\in O_r}{h_{p,r,g_r}(\bm{x},\bm{x}')}\phi_j(\bm{x}')q(\bm{x}')d\bm{x}'}$.
\end{enumerate}
\end{theorem}
\begin{proof}
We can directly use the results in Section 5.5 of \citep{serfling2009approximation}. We only need to check the conditions $\sigma_h\neq 0$ when $p\neq q$ and $\sigma_h=0$ when $p=q$.

When $p=q$, we re-write $\mathbb{E}_{\bm{x}'\sim q}[h_{p,r,g_r}]$ as 
\[
\mathbb{E}_{\bm{x}'\sim q}[h_{p,r,g_r}(\bm{x},\bm{x}')]=\langle \xi_{p,r,g_r}(\bm{x},\cdot),\mathbb{E}_{\bm{x}'\sim q}[\xi_{p,r,g_r}(\bm{x}',\cdot)]\rangle_{\mathcal{H}_{rg_r}}
\]
From the Eq.(\ref{eq: Stein identity in kernel}) in theorem \ref{thm: Validity of SKSD}, we have $\mathbb{E}_{\bm{x}'\sim q}[\xi_{p,r,g}(\bm{x}',\cdot)]=0$. Thus, $\mathbb{E}_{\bm{x}'\sim q}[h_{p,r,g_r}]=0$ for all $\bm{r}\in O_r$. Thus, we have $\sigma_h=0$ when $q=p$.

We assume when $p\neq q$, $\sigma_h=0$. This means $\mathbb{E}_{\bm{x}'\sim q}[h_{p,r,g}(\bm{x},\bm{x}')]=c_r$ where $c_r$ is a constant. Thus,
\[
\begin{split}
    &c_r=\mathbb{E}_{\bm{x}\sim p}[\mathbb{E}_{\bm{x}'\sim q}[h_{p,r,g}(\bm{x},\bm{x}')]]\\
    \Rightarrow&c_r=\mathbb{E}_{\bm{x}'\sim q}[\mathbb{E}_{\bm{x}\sim p}[h_{p,r,g}(\bm{x},\bm{x}')]]
\end{split}
\]
From the Eq.(\ref{eq: Stein identity in kernel}) in Theorem \ref{thm: Validity of SKSD}, we have $c_r=0$ for all $\bm{r}\in O_r$. Thus, $\mathbb{E}_{\bm{x},\bm{x}'\sim q}[h_{p,r,g}(\bm{x},\bm{x}')]=c_r=0$ which contradict $p\neq q$
\end{proof}
This theorem indicates a well-defined limit distribution for maxSKSD U-statistics. Next, similar to the previous work \citep{liu2016kernelized}, we adopt the bootstrap method \citep{arcones1992bootstrap,huskova1993consistency}. 
The quantile computed by the bootstrap samples (Eq.\ref{eq: Bootstrap samples}) is consistent to the one using degenerated U-statistics. This consistence is established in \citep{huskova1993consistency,arcones1992bootstrap}.


\subsection{Sliced SVGD}
\label{APP: Sliced SVGD}
First, we introduce one result from \cite{liu2016stein}, that shows the connections between the SD and KL divergence between the particle's underlying distribution $q$ and target $p$.
\begin{lemma}{\citep{liu2016stein}}
Let $T(\bm{x})=\bm{x}+\epsilon\bm{\phi}(\bm{x})$ and $q_{[T]}(\bm{z})$ be the density of $\bm{z}=T(\bm{x})$ when $\bm{x}\sim q(\bm{x})$. With $\mathcal{A}_p$ the Stein operator defined in Eq.(\ref{eq: Stein Operator}), we have
\begin{equation}
    \nabla_{\epsilon}KL[q_{[T]}||p]|_{\epsilon=0}=-\mathbb{E}_q[\mathcal{A}_p\bm{\phi}(\bm{x})].
    \label{eq: SVGD Magnitude}
\end{equation}
\label{lem: SVGD and SD}
\end{lemma}%

To derive the sliced version of SVGD, we follow the similar recipe of \cite{liu2016stein} by first connecting the SSD with KL divergence (like Lemma \ref{lem: SVGD and SD}), and then derive the optimal perturbation directions (like Lemma \ref{lem: SVGD and KSD} in background section \ref{sub:Background SVGD}). To achieve this, we modify the flow mapping to $T_G(\bm{x}):\mathbb{R}^D\rightarrow\mathbb{R}^D$ as $T_G(\bm{x})=\bm{x}+\epsilon\bm{\phi}_G(\bm{x})$. Specifically for $\bm{\phi}_G(\bm{x})$, we adopt $D$ univariate perturbations instead of one multivariate perturbation:
\begin{equation}
    \bm{\phi}_G(\bm{x})=\left[
    \begin{array}{c}
         \phi_{g_1}(\bm{x}^T\bm{g}_1)\\
         \vdots\\
         \phi_{g_D}(\bm{x}^T\bm{g}_D)
    \end{array}
    \right]
    \label{eq: S_SVGD perturbation}
\end{equation}
where $G=[\bm{g}_1,\ldots,\bm{g}_D]\in \mathbb{R}^{D\times D}$ represents slice matrix. For this specific mapping we have the following result analogous to Lemma \ref{lem: SVGD and SD}.
\begin{lemma}
Let $T_G(\bm{x})=\bm{x}+\epsilon\bm{\phi}_G(\bm{x})$ where $\bm{\phi}_G$ is defined as Eq.(\ref{eq: S_SVGD perturbation}). Define $q_{[T_G]}(\bm{z})$ as the density of $\bm{z}=T_G(\bm{x})$ when $\bm{x}\sim q(\bm{x})$, with slice matrix $\bm{G}$, we have
\begin{equation}
    \nabla_{\epsilon}[q_{[T_G]}||p]|_{\epsilon=0}=-\sum_{d=1}^D{\mathbb{E}_q[s_p^d(\bm{x})\phi_{g_d}(\bm{x}^T\bm{g}_d)+g_{d,d}\nabla_{\bm{x}^T\bm{g}_d}\phi_{g_d}(\bm{x}^T\bm{g}_d)]}
    \label{eq: S_SVGD and SSD}
\end{equation}
where $s_p^d(\bm{x})=\nabla_{x_d}\log p(\bm{x})$ and $g_{d,d}$ is the d$^{\text{th}}$ element in $\bm{g}_d$.
\label{lem: S_SVGD and SSD}
\end{lemma}
\begin{proof}
This can be easily verified by substituting Eq.(\ref{eq: S_SVGD perturbation}) into Eq.(\ref{eq: SVGD Magnitude}).
\end{proof}

{Eq.(\ref{eq: S_SVGD and SSD}) is similar to maxSSD (Eq.(\ref{eq:maxSSD})) where the optimal test directions and test functions are replaced with matrix $\bm{G}$ and perturbation $\phi_{g_d}(\bm{x})$. $O_r$ takes the values one-hot vectors. The main difference between this decrease magnitude and maxSSD is that we do not assume $\bm{G}$ is optimal. Next, we show how to obtain an analytic descent directions that maximize the decrease magnitude.}

By restricting each perturbation $\phi_{g_d}\in\mathcal{H}_{rg_d}$ where $\mathcal{H}_{rg_d}$ is an RKHS equipped with kernel, we have the following result.
\begin{lemma}
Assume the conditions in lemma \ref{lem: S_SVGD and SSD}. If for each perturbation $\phi_{g_d}\in \mathcal{H}_{rg_d}$ where $\mathcal{H}_{rg_d}$ is an RKHS equipped with kernel $k_{rg_d}$ and $||\phi_{g_d}||_{\mathcal{H}_{rg_d}}\leq D_{rg_d}(q,p)$, then the steepest descent direction for d$^{\text{th}}$ perturbation is 
\begin{equation}
    \phi^*_{g_d}(\cdot)=\mathbb{E}_{q}[\xi_{p,r_d,g_d}(\bm{x},\cdot)],
    \label{eq:S_SVGD update direction}
\end{equation}
and 
\begin{equation}
    \nabla_{\epsilon}KL[q_{[T_G]}||p]|_{\epsilon=0}=-\sum_{d=1}^D{D_{dg_d}^2(q,p)},
    \label{eq: S_SVGD decrease magnitude}
\end{equation}
where $D_{dg_d}^2(q,p)=\mathbb{E}_q[h_{p,r_d,g_d}(\bm{x},\bm{x}')]$ with one-hot vector $\bm{r}_d$.
\label{lem:S_SVGD and SKSD}
\end{lemma}
\begin{proof}
We show this result using the reproducing property of RKHS $\mathcal{H}_{rg_d}$.
The supremum of Eq.(\ref{eq: S_SVGD and SSD}) can be re-written as 
\begin{equation}
    \begin{split}
        &\sup_{\bm{\phi}_G}{\sum_{d=1}^D{\mathbb{E}_q[s^d_p(\bm{x})\phi_{g_d}(\bm{x}^T\bm{g}_d)+\bm{r}_d^T\bm{g}_d\nabla_{\bm{x}^T\bm{g}_d}\phi_{g_d}(\bm{x}^T\bm{g}_d)]}}\\
        =&\sum_{d=1}^D{\sup_{\phi_{g_d}}{{\mathbb{E}_q[s^d_p(\bm{x})\phi_{g_d}(\bm{x}^T\bm{g}_d)+\bm{r}_d^T\bm{g}_d\nabla_{\bm{x}^T\bm{g}_d}\phi_{g_d}(\bm{x}^T\bm{g}_d)]}}}\\
        =&\sum_{d=1}^D{\sup_{\substack{\phi_{g_d}\in\mathcal{H}_{rg_d}\\ ||\phi_{g_d}||_{\mathcal{H}_{rg_d}}\leq D_{rg_d}(q,p)}}{\mathbb{E}_q[\langle s_p^d(\bm{x})k_{rg_d}(\bm{x}^T\bm{g}_d,\cdot)+\bm{r}_d^T\bm{g}_d\nabla_{\bm{x}^T\bm{g}_d}k_{rg_d}(\bm{x}^T\bm{g}_d,\cdot),\phi_{g_d} \rangle_{\mathcal{H}_{rg_d}}]}}\\
        =&\sum_{d=1}^D{\sup_{\substack{\phi_{g_d}\in\mathcal{H}_{rg_d}\\ ||\phi_{g_d}||_{\mathcal{H}_{rg_d}}\leq D_{rg_d}(q,p)}}{\langle \mathbb{E}_q[\xi_{p,r_d,g_d}(\bm{x},\cdot)],\phi_{g_d} \rangle_{\mathcal{H}_{rg_d}}}}\\
        =&\sum_{d=1}^D{\mathbb{E}_q[h_{p,r_d,g_d}(\bm{x},\bm{x}')]}=\sum_{d=1}^D{D_{dg_d}^2(q,p)},
    \end{split}
\end{equation}
where the third equality is because of the Bochner integrability of $\xi_{p,r_d,g_d}(\bm{x},\cdot)$ shown in Theorem \ref{thm: SKSD Analytic form}. And the optimal perturbation for d$^{\text{th}}$ dimension is 
\begin{equation}
    \phi^*_{g_d}(\cdot)=\mathbb{E}_{q}[\xi_{p,r_d,g_d}(\bm{x},\cdot)].
\end{equation}
\end{proof}
Note that in Lemmas \ref{lem: S_SVGD and SSD} and \ref{lem:S_SVGD and SKSD}, we assume an arbitrary projection matrix $\bm{G}$. To find the steepest descent direction, one can maximize Eq.(\ref{eq: S_SVGD decrease magnitude}) w.r.t.~$\bm{G}$. In this case this decrease magnitude Eq.(\ref{eq: S_SVGD decrease magnitude}) is identical to maxSKSD Eq.(\ref{eq: maxSKSD}) with orthogonal basis $O_r$ and optimal test directions $\bm{G}$.

{The name \emph{sliced} SVGD comes from that for each perturbation $\phi_{g_d}(\cdot)$, the kernel $k_{rg_d}$ and the repulsive force $\bm{r}_d^T\bm{g}\nabla_{\bm{x}^T\bm{g}_d}k_{rg_d}(\bm{x}^T\bm{g}_d,\cdot)$ are evaluated on $\bm{x}^T\bm{g}_d$ instead of $\bm{x}$ in SVGD.} Although S-SVGD only uses one-dimensional projection of $\bm{x}$, it is still a valid inference method as long as the optimality of $\bm{G}$ is ensured, because maxSKSD is a valid discrepancy measure.\footnote{Note that maximizing Eq.(\ref{eq: S_SVGD decrease magnitude}) w.r.t. sliced matrix $\bm{G}$ is necessary, otherwise Eq.(\ref{eq: S_SVGD decrease magnitude}) is not a valid discrepancy measure, and a zero value does not imply $p=q$. In such case the resulting particle inference method is not asymptotically exact.} The S-SVGD method is summarised in Algorithm \ref{alg: S_SVGD}. In practice this algorithm may violate the optimality condition of $\bm{G}$ (due to estimation error using finite samples and local optimum found by gradient-based optimization), which is a common issue in many adversarial training procedure.
\begin{algorithm}[H]
\SetKwInOut{Input}{Input}
\SetKwInOut{Output}{Output}

\SetAlgoLined
\Input{Initial samples $\{\bm{x}_i\}_{i=1}^N$, target score function $\bm{s}_p(\bm{x})$, Orthogonal basis $O_r$, initial slice matrix $\bm{G}$, kernel function $k_{rg}$, iteration number $L$ and step size $\epsilon$.}
\Output{Set of particles $\{\bm{x}_i\}_{i=1}^N$ that approximates $p$}
\For{l $\leq$ L}{
Update each particles $\bm{x}_i^{l+1}=\bm{x}_i^l+\epsilon \bm{\phi}^*_{G}(\bm{x}_i^l)$ where $\bm{\phi}^*_G(\bm{x}_i^l)$ is computed using Eq.(\ref{eq:S_SVGD update direction})\;

Find the optimal slice matrix $\bm{G}$ by maximizing Eq.(\ref{eq: S_SVGD decrease magnitude}) using $\{\bm{x}^{l+1}_i\}_{i=1}^N$
}
 \caption{S-SVGD for variational inference}
 \label{alg: S_SVGD}
\end{algorithm}

\section{Limitations of maxSKSD-g}
\label{APP: Pathology}
\begin{figure}
    \centering
    \includegraphics[scale=0.2]{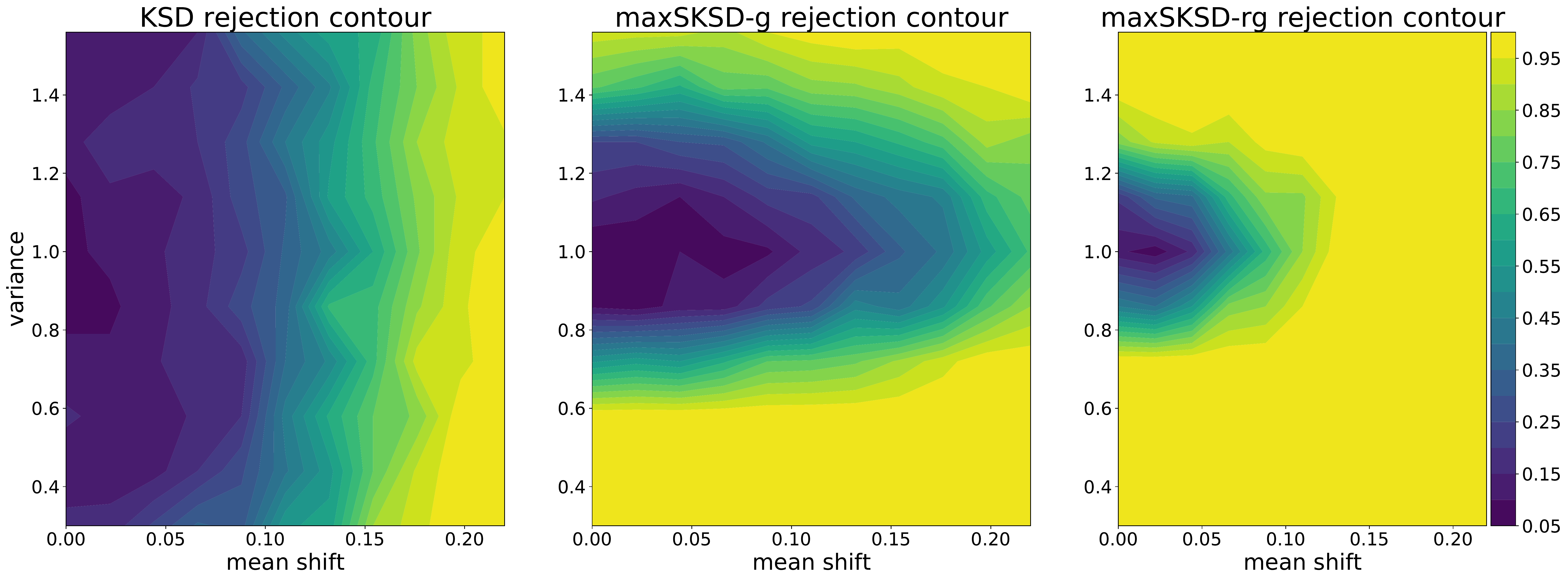}
    \caption{This rejection rate contour demonstrates the sensitivity of each GOF method to the change of mean/variance of two Gaussians under SDS. (\textbf{Left}) The rejection sensitivity of KSD. (\textbf{Middle}) Sensitivity of \textit{maxSKSD-g}. (\textbf{Right}) Sensitivity of \textit{maxSKSD-rg}.}
    \label{fig:sensitivity}
\end{figure}
In this section, we discuss the potential limitations of \textit{maxSKSD-g} and motivate the use of \textit{maxSKSD-rg} to address such issue. We begin this journey by proposing a sensitivity test on the rejection rate between two Gaussians with different mean and variances. Then, we analyze the issue of \textit{maxSKSD-g}, and why \textit{maxSKSD-rg} can potentially address such pathology. In the end, we empirically show \textit{maxSKSD-rg} indeed solves the problem under the same sensitivity test, and achieves the best performance. 
\subsection{Sensitivity test of two Gaussians}
Assume we have two fully factorized $50$ dimensional Gaussians $p$, $q$, we propose to test the sensitivity of the GOF test method to the change of mean and variance of $q$. In the following experiment, we only change the mean and variance of the $1^{\text{st}}$ dimension and keep the remaining $49$ dimensions to be the same as $p$. We call this scheme \textit{single dimension shift} (SDS). We include KSD as the baseline for comparison.

From the left and middle panel of figure.\ref{fig:sensitivity}, we notice that KSD is sensitive to the mean shift but less sensitive to the variance change. Specifically, KSD successfully detects the difference between $q$ and $p$ when the mean shift is beyond $0.18$ but fails if the only changed term is the variance. On the other hand, \textit{maxSKSD-g} is capable of detecting small variance changes but less sensitive to the mean change compared to KSD. This is consistent with the conclusion in section \ref{subsub: benchmark GOF} that \textit{maxSKSD-g} obtains nearly optimal rejection rate in Gaussian diffusion test where KSD completely fails.
However, this still shows the potential limitations of \textit{maxSKSD-g} that it may perform worse than the baselines under certain circumstances.
\subsection{Analysis of the pathology}
In this section, we give a detailed analysis on the potential reasons behind such limitation.
From the setup of SDS mean shift, we know the mean of $q$, $\bm{\mu}_q$, differs from $\bm{\mu}_p$ only in the $1^\text{st}$ dimension. We write down $q_d(x_d)$ to be the marginal distribution for dimension $d$, and $p_d$ accordingly. \textit{maxSKSD-g} (Eq.\ref{eq: maxSKSD}) requires the orthogonal basis $O_r$ and corresponding optimal slice matrix $\bm{G}$. We set $O_r$ to be the standard basis of $\mathbb{R}^D$, i.e. a set of one-hot vectors. Due to the fully factorized property of $p$, $q$ and SDS setup, the optimal slice matrix $\bm{G}$ is an identity matrix. Therefore, we can re-write the \textit{maxSKSD-g} into a summation over a set of KSD between their marginals $q_d$ and $p_d$. We first inspect the optimal test function in Eq.(\ref{eq: SKSD test function}) for dimension d. Assume we use the same kernel $k$ for all $\bm{r}$, $\bm{g}$, i.e. $k=k_{rg}$, we have
\begin{equation}
    \begin{split}
        \xi_{p,r,g}(\bm{x},\cdot)&=s_p^r(\bm{x})k(\bm{x}^T\bm{g}_r,\cdot)+\bm{r}^T\bm{g}_r\nabla_{\bm{x}^T\bm{g}_r}k(\bm{x}^T\bm{g}_r,\cdot)\\
        &=s_{p_d}(x_d)k(x_d,\cdot)+\nabla_{x_d}k(x_d,\cdot)
    \end{split}
\end{equation}
Substituting it into Eq.(\ref{eq: D_rg}), we obtain
\begin{equation}
    \begin{split}
        D_{rg}(q,p)&=||\mathbb{E}_{q}[\xi_{p,r,g}(\bm{x})]||^2_{\mathcal{H}_{rg}}=||\mathbb{E}_{q_d}[\xi_{p,r,g}(x_d)]||^2_{\mathcal{H}_{rg}}\\
        &=\mathbb{E}_{q_d(x_d)q_d(x'_d)}[u_p(x_d,x'_d)]\\
        &=D^2(q_d,p_d)
    \end{split}
\end{equation}
This is exactly the KSD between the marginal $q_d$ and $p_d$. Therefore, the \textit{maxSKSD-g} is written as 
\begin{equation}
    \begin{split}
        SK_{max}(q,p)&=\sum_{\bm{r}\in O_r}{\sup_{\bm{g}_r}{D_{rg}^2(q,p)}}\\
        &=\sum_{d=1}^D{D_{rg}(q,p)}\\
        &=\sum_{d=1}^D{D^2(q_d,p_d)}
    \end{split}
\end{equation}
This is the summation of the KSD between their marginals across all dimensions.

As the mean only differs in the first dimension, the dominant value for \textit{maxSKSD-g} is the KSD between the first marginal, $q_1$ and $p_1$. However, in practice, variances exists in other KSD term inside the summation of \textit{maxSKSD-g} and they provides nothing but noise. Therefore the overall variance of \textit{maxSKSD-g} increases with the dimensions. For a high dimensional problem, the important signal can be easily buried due to the increasing variance. 
\begin{figure}
    \centering
    \includegraphics[scale=0.25]{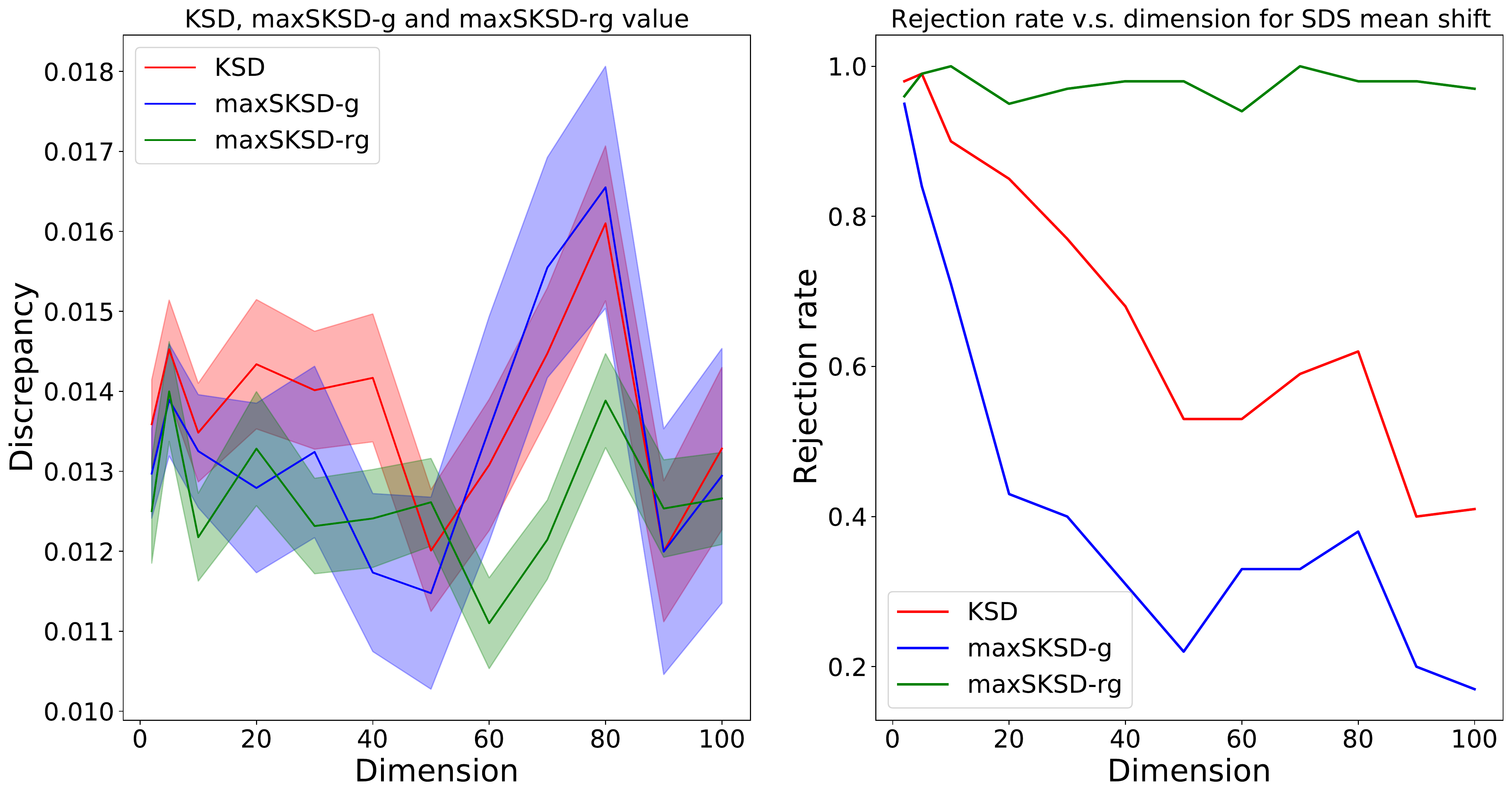}
    \caption{(Left) The discrepancy value of KSD, maxSKSD-g and maxSKSD-rg with standard error for SDS mean shift. Each plot is obtained by averaging 100 runs. (Right) The rejection rate for SDS mean shift with different dimensions.}
    \label{fig:variance dim}
\end{figure}
Figure \ref{fig:variance dim} (Left) shows the discrepancy value against dimensions. We can easily observe that the variance of \textit{maxSKSD-g} increases with the dimensions. This is consistent with the above analysis. The variance of KSD also increases but is less obvious compared to \textit{maxSKSD-g}. Figure \ref{fig:variance dim} (Right) shows that \textit{maxSKSD-g} also suffers from curse-of-dimensionality under the SDS mean shift due to such increasing variance. 
\subsection{maxSKSD-rg}
Based on the above analysis, the underlying reason behind \textit{maxSKSD-g} pathology is the noise provided by the irrelevant dimensions. This is the fundamental problem caused by choosing a orthogonal basis $O_r$ for \textit{maxSKSD}. In fact, for any orthogonal basis $O_r$, we can always create a such pathology where \textit{maxSKSD-g} suffers from the extra noise provided by irrelevant basis. 
One potential way to address such problem is to introduce a metric to avoid such irrelevant basis. In here, we choose \textit{maxSKSD} itself as the metric. Thus, instead of selecting a orthogonal basis $O_r$, we try to optimize the projections $O_r\in\mathbb{R}^{D\times D}$ (not necessarily orthogonal). We write \textit{maxSKSD-rg} as
\begin{equation}
    SK_{max-rg}(q,p)=\sup_{O_r,\bm{G}}{\sum_{\bm{r}\in O_r}{D^2_{rg_r}(q,p)}}
    \label{eq: maxSKSD rg}
\end{equation}
We test the performance of \textit{maxSKSD-rg} in the same SDS mean shift setup. Due to the fully factorized property of two Gaussians, the optimal $O_r$ consists of repeated one-hot vector $\bm{r}=[1,0,0,\ldots]$. The optimal slice matrix $\bm{G}$ share the same structure as $O_r$. Thus, \textit{maxSKSD-rg} (Eq.\ref{eq: maxSKSD rg}) becomes 
\begin{equation}
    SK_{max-rg}(q,p)=D\times D^2(q_1,p_1)
\end{equation}
which is $D$ times KSD value between the marginals of the $1^{\text{st}}$ dimensions. We notice that $D$ is just a constant, thus can be removed without changing its performance. We call this effective \textit(maxSKSD-rg). As it only considers the marginal of the most important dimension, its variance does not increase with dimension of the problem. From the \textit{maxSKSD-rg} value in figure.\ref{fig:variance dim} (left), we observe it has the lowest variance and does not change across the dimensions as expected. In terms of the rejection rate (the right of figure.\ref{fig:variance dim}), \textit{maxSKSD-rg} does not suffers from the curse-of-dimensionality, and consistently achieves nearly optimal rejection rate.

We also conduct the sensitivity test for \textit{maxSKSD-rg} (the right in figure.\ref{fig:sensitivity}). We observe \textit{maxSKSD-rg} not only addresses the mean shift pathology of \textit{maxSKSD-g}, it is also more sensitive compared to KSD.
Additionally, it is even more sensitive to the variance change compared to \textit{maxSKSD-g}. The reason is the same as the mean shift case.
\subsection{Is maxSKSD-rg always better?}
\wgr{Based on the above analysis, \textit{maxSKSD-rg} is superior compared to \textit{maxSKSD-g} theoretically. Indeed, it is trivial that \textit{maxSKSD-g} is a lower bound for \textit{maxSKSD-rg} which indicates a weaker discriminative power. However, this theoretical advantage relies on the assumption of the optimality of $\bm{r}$, which is also the key gap between theory and application. 

\textit{maxSKSD-rg} often gives superior performance compared to \textit{maxSKSD-g} in terms of GOF test. This is because GOF test only focus on the difference between two distributions, i.e. it focuses on finding a direction that gives higher discrepancy value. But this is not the case for model learning especially when they are used as training objectives. 
Instead, model learning focuses on the fact that the model approximates the target distribution in every directions. In theory, \textit{maxSKSD-rg} can still give good performance as at each training iteration, this objective tries to minimize the largest difference between two distributions. However, this is not true in practice. We suspect the reasons are two fold: (1) optimal $\bm{r}$ can not be guaranteed; (2) if the true optimal direction is drastically changing between iterations, gradient-based optimization may takes long time to move away from the current $\bm{r}$.
Therefore, instead of relying on one slicing direction that might be sub-optimal in practice, in \textit{maxSKSD-g} slicing along the directions in the orthogonal basis $O_r$ provides better coverage of the difference between $p$ and $q$.
Figure \ref{fig:ICA g rg comparison} shows the comparison between \textit{maxSKSD-rg} and \textit{maxSKSD-g} for training 200 dimensional ICA. }
\begin{figure}
    \centering
    \includegraphics[scale=0.15]{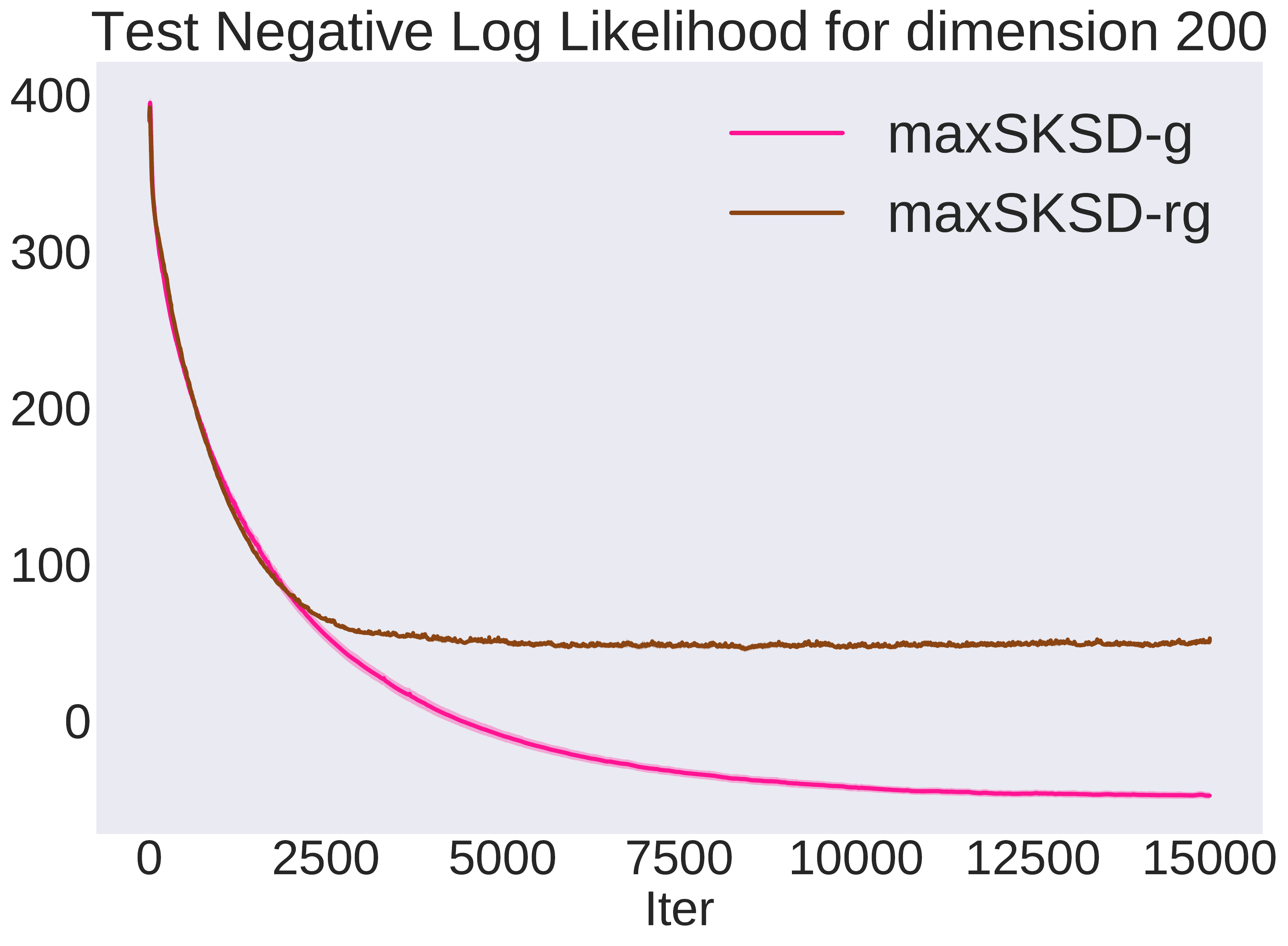}
    \caption{The ICA training curve of test negative LL with different training objectives. y-axis indicates the negative test log-likelihood.}
    \label{fig:ICA g rg comparison}
\end{figure}
\section{Computational and memory cost}
\label{APP: computational memory cost}
\wgr{In this section, we give a brief analysis on the computational and memory cost for performing GOF test and S-SVGD. }
\subsection{Memory cost}
\wgr{First, let's consider the GOT test. If \textit{maxSKSD-g} is used, we only need to store the slicing matrix $\bm{G}$, which cost $O(D^2)$. However, if \textit{maxSKSD-rg} is used, the actual memory cost can be much reduced. For GOF test, we only care about whether it can detect the differences between distributions. Thus, we do not need $D$ number of basis $\bm{r}$, instead, in theory, we only need $1$ basis  which maximizes the discriminative power. In practice, we may need $m$ basis where $1<m\ll D$. Thus, the memory cost will be $O(mD)+O(mD)=O(mD)$.

In the scope of this work, the memory cost for S-SVGD is similar to GOF test with \textit{maxSKSD-g}, where slicing matrix $\bm{G}\in\mathbb{R}^{D\times D}$ need to be stored. However, it remains a question that whether \textit{maxSKSD-rg} can be used to derive a corresponding variant of S-SVGD. If that's the case, one only need to select 'important' basis $\bm{r}$ to drive the particles towards the target distributions. Thus, the memory cost is reduced to $O(mD)$. For the BNN experiment in appendix \ref{App: Bayesian NN}, the actual memory cost for SVGD with Boston housing data set is $1003$MB and $1203$MB for S-SVGD with NVIDIA Quadro P6000.
}
\subsection{Time Complexity}
\wgr{The computational cost of computing the kernel $K(\bm{x},\bm{x}')$ with each pair of $\bm{x}$,$\bm{x}'$ is $O(D)$. Thus, evaluating KSD requires the $O(N^2D)$ where $N$ is the number of samples used for Monte Carlo estimations. Similarly, the cost of each kernel evaluation $K(\bm{x}^T\bm{g},\bm{x}'^T\bm{g})$ is $O(D)$, but \textit{maxSKSD-g} requires such evaluation for each $\bm{g}$. Thus, the total cost for \textit{maxSKSD-g} is $O(N^2D^2)$. On the other hand, \textit{maxSKSD-rg} can reduce such cost to $O(N^2Dm)$ where $m$ is the number of 'important' basis selected by maximizing Eq.\ref{eq: maxSKSD}.

For SVGD, the kernel is also evaluated on each pair of $\bm{x}$,$\bm{x}'$. Thus, the computational cost is $O(N^2D)$ where $N$ is the number of particles. Similarly, the computational cost of S-SVGD is $O(N^2D^2)$. In practice, first we compare their time consumption on the above BNN experiment. The S-SVGD uses about 0.073s per epoch whereas SVGD only uses 0.032s per epoch, which is about 2.5 times larger than SVGD with P6000. For amortized SVGD experiment with MNIST data set, the time consumption of amortized SVGD is 0.112s per iteration, and S-SVGD is 0.122s per iteration, which is almost the same due to the smaller latent dimension (32). 

We argue that there is no free lunch and every method has its compromises. In our case, the significant advantages of the proposed methods compared to KSD and SVGD come with the cost of higher computational and memory consumption. Even with this extra cost, the proposed framework is still a significant improvement for KSD and SVGD, as they fail even at very low dimensions (around 30), where the cost of our method is not much higher. Especially for GOF test, even the dimension is huge, one can always adopt \textit{maxSKSD-rg} to reduce the memory and computational cost by selecting important basis.}
\section{GOF test}
\subsection{Setup for High dimensional benchmark GOF test}
\label{App: High dimensional benchmark GOF test}
For each GOF test, we draw 1000 samples from alternative hypothesis $q$. These samples are directly used for GOF test methods that do not require any training procedure, like KSD, MMD, RFSD and FSSD Rand. However, for methods that require training like maxSKSD and FSSD Opt, we split the entire samples into 200 training data and 800 GOF test data as \citep{jitkrittum2017linear,huggins2018random}. For maxSKSD, we initially draw the slice matrix $\bm{G}$ from a normal distribution before normalizing the magnitude of each vectors in $\bm{G}$ to $1$ and use Adam with learning rate $0.001$ to update it (maximizing Eq.(\ref{eq: maxSKSD})). For FSSD-Opt, we use the default settings in the original publication. During the GOF test, only the test data are used for FSSD-Opt and maxSKSD. We set the significant level $\alpha=0.05$ and the dimension of the distribution grows from 2 to 100. 
We use 1000 bootstrap samples for all tests, and 1000 trials for Gaussian Null test, 500 trials for Gaussian Laplace test, 250 trials for Gaussian Multivariatet-t test and 500 trials for Gaussian diffusion test. 
\subsection{Setup for RBM GOF test}
\label{App: Setup RBM GOF test}
For the RBM, we use $50$ dimension for observable variable and $40$ dimension for hidden variable. We run 100 trials with $1000$ bootstrap samples for each method. $1000$ test samples are used for methods like KSD, MMD, RFSD and FFSD Rand. For maxSKSD\_g, maxSKSD\_rg and FFSD Opt that require training, we use $800$ samples for test. Parallel block Gibbs sampler with $2000$ burn-in is used to draw samples from $q$. To avoid the over-fitting to small training samples, we use 200 samples to update the slice matrix $\bm{G}$ (or $\bm{G}$ and $\bm{r}$ for maxSKSD\_rg) in each Gibbs step during the burning. However, it should be noted that these intermediate samples from the burn-in should not be used as the test samples for other methods because they are not from $q$. This setup is slightly different from the most general GOF test where only test samples and target density are given. However, it is still useful for some applications such as detecting convergence/selecting hyper-parameter of MCMC sampler (appendix \ref{App: Sampler Convergence}). Finding relatively good directions for maxSKSD with fewer training samples is a good direction for future work. 

\subsection{Selecting hyperparameter of a biased sampler}
We use the proposed methods to select the step size of a biased sampler. Particularly, we consider using SGHMC here which is a biased sampler without Metropolis-Hasting step. The bias is mainly caused by the discretization error, namely, the step size. For smaller step size, the bias is small but the mixing speed is slow. Larger step size results in higher bias with fast mixing. 

Selecting the step size is essentially a GOF test problem, where alternative hypothesis is the invariant underlying distribution of SGHMC, and the bias is quantified by the discrepancy value. The best step size is the one corresponding to the lowest discrepancy value. We compare our proposed maxSKSD based methods with KSD. The target distribution is a 15 dimensional correlated Gaussian distribution with zero mean and randomly generated co-variance matrix. We also include a strong baseline using KL divergence where the $q$ is a Gaussian distribution, with parameters estimated by samples from SGHMC. 

\paragraph{Setup} We run 100 parallel SGHMC chains with 2000 burn-in period. During each step in burn-in, we update the sliced matrix $\bm{G}$ (and $\bm{r}$) using such 100 samples. After burn-in, we fix the $\bm{G}$ and $\bm{r}$ and continue to run SGHMC with thinning $5$ until $1500$ test samples are collected. We run this experiment using 3 different seeds. For KSD and maxSKSD discrepancy value, we use U-statistics due to its unbiasedness. 
\begin{figure}
    \subfloat[Random seed 0]{\includegraphics[scale=0.145]{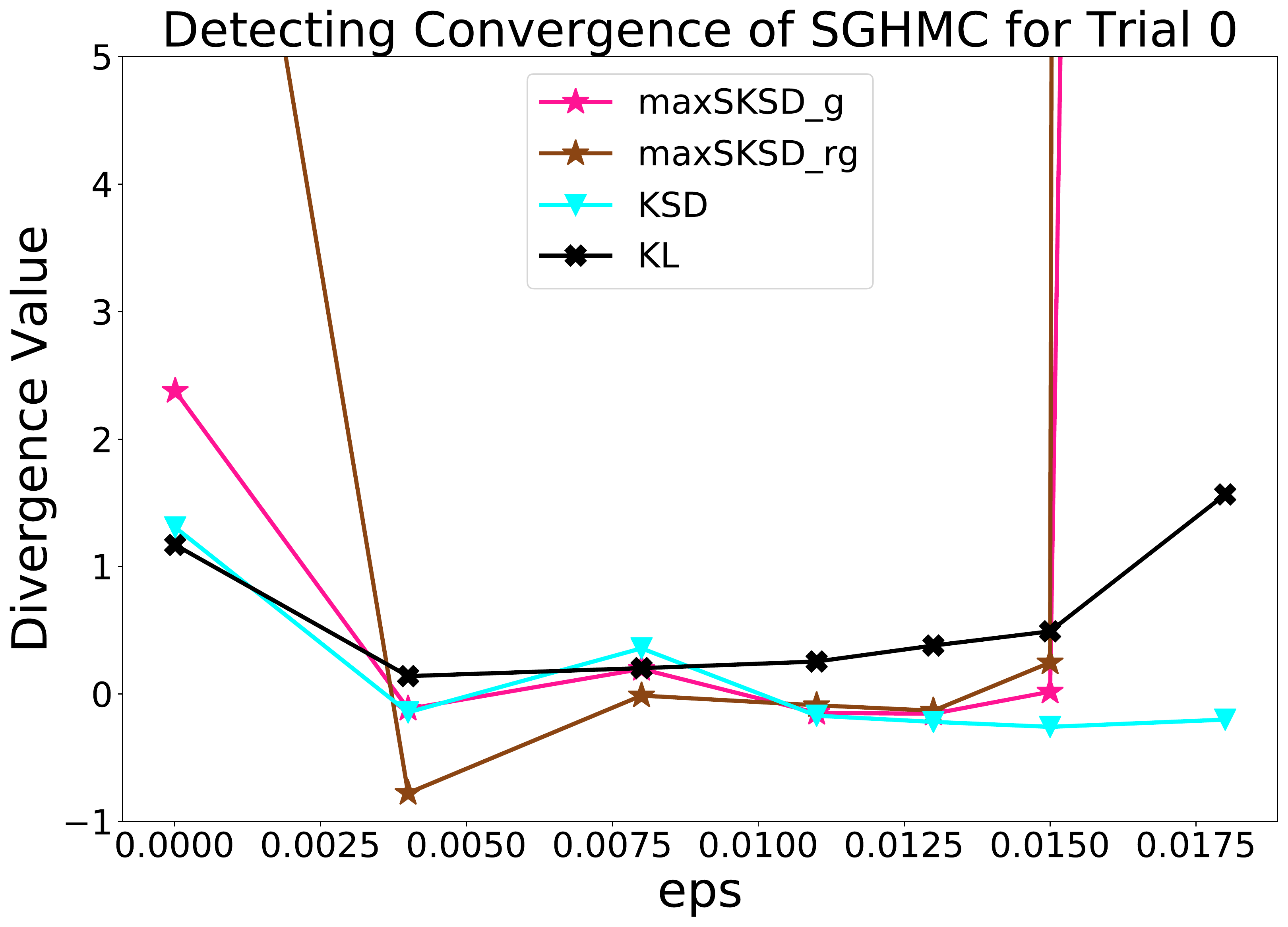}}
    \subfloat[Random seed 1]{\includegraphics[scale=0.145]{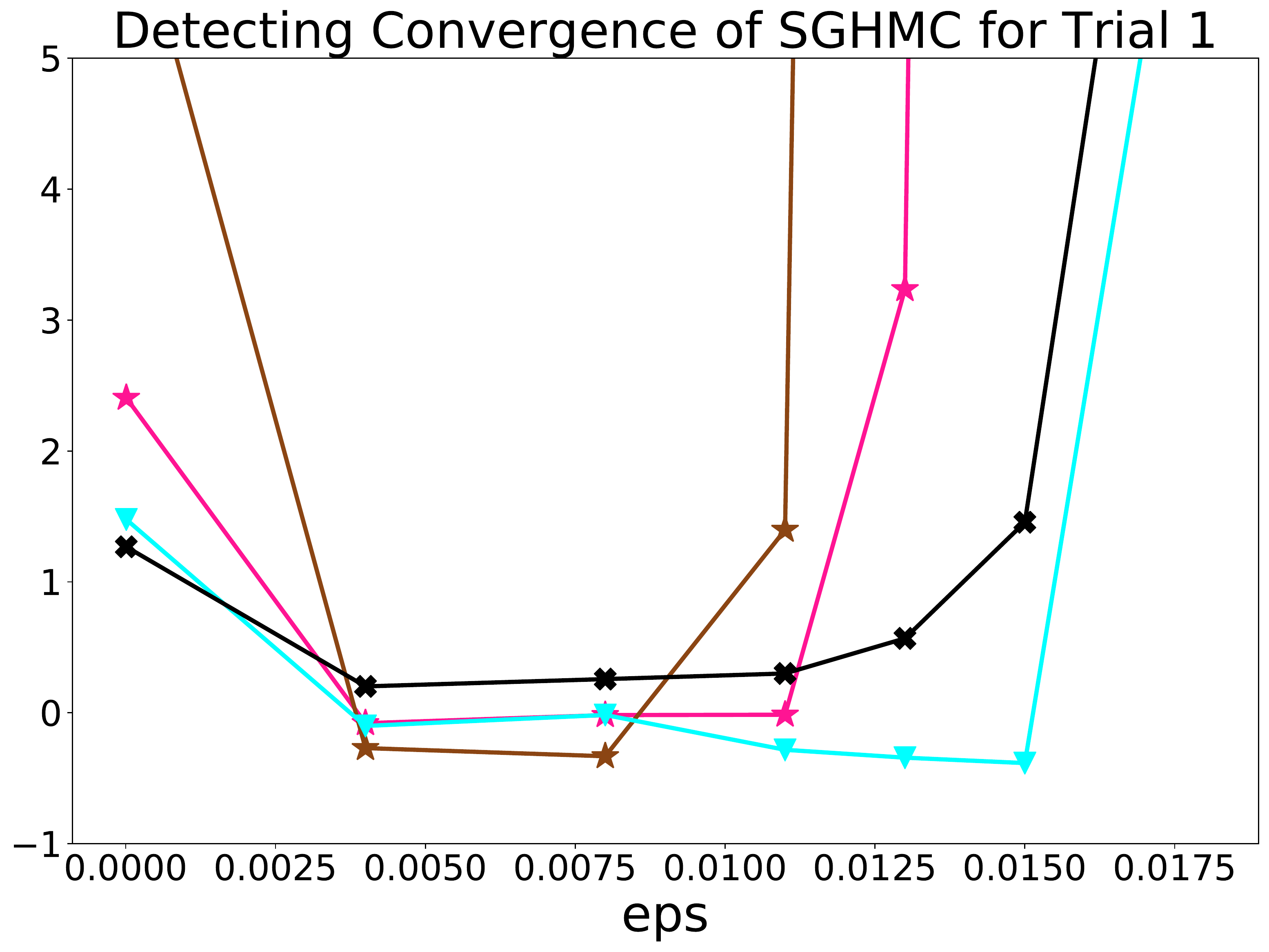}}
    \subfloat[Random seed 2]{\includegraphics[scale=0.145]{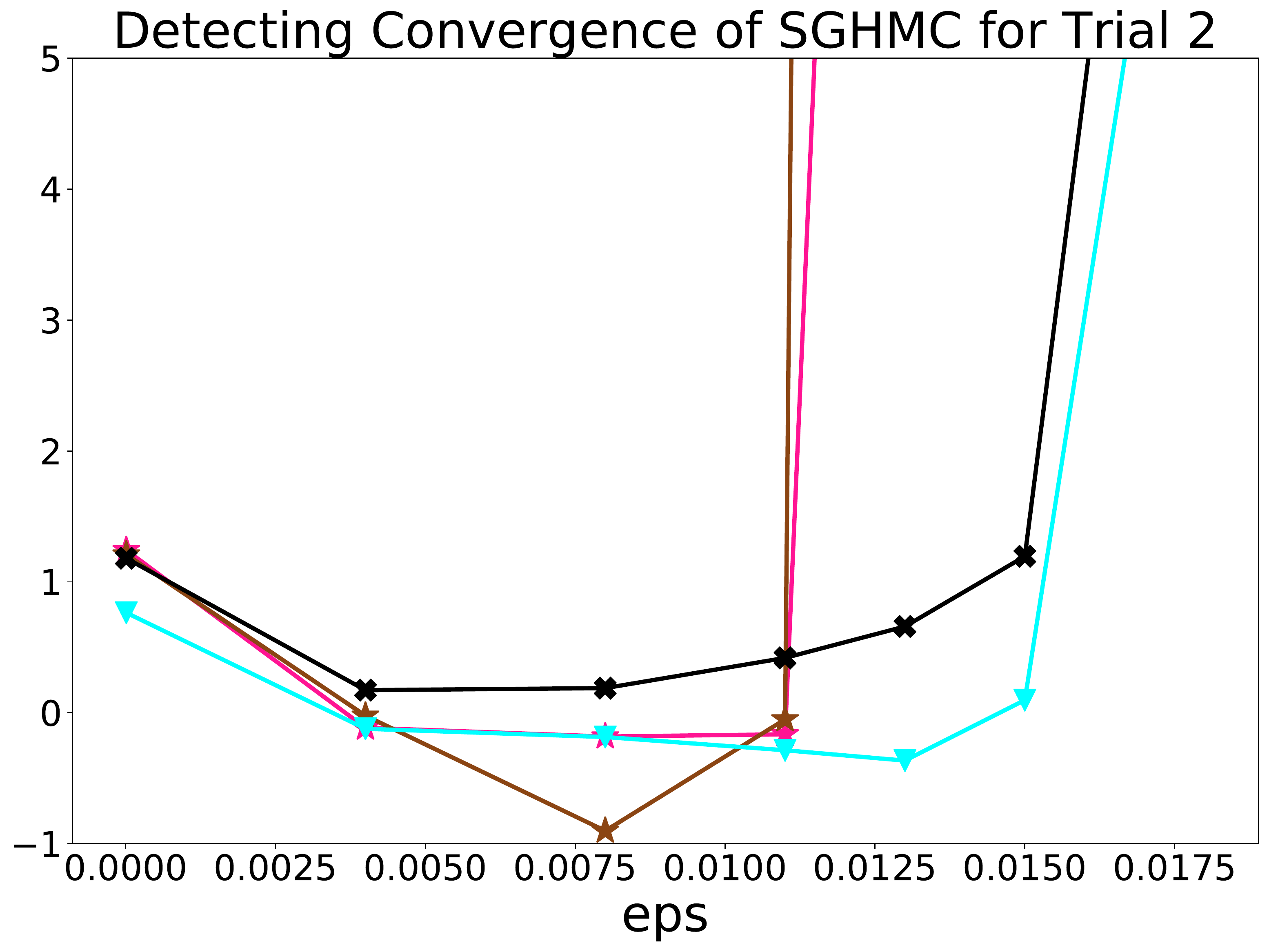}}
    \caption{Discrepancy value for different random seed. The x-axis indicates the step size used for SGHMC.}
    \label{fig: sampler convergence}
\end{figure}

Figure \ref{fig: sampler convergence} shows the discrepancy curve with different step sizes and table \ref{tab: SGHMC Convergence} shows some diagnostic statistics. KL based method is used as the 'ground truth' measure. In summary, the step sizes chosen by maxSKSD based methods are more sensible than those selected by KSD. To be specific, take random seed 1 as an example, KSD failed to detect the non-convergence for step size larger than $0.011$ where KL starts to increase. Even worse, KSD achieves the lowest value at step size $0.015$ which is a poor choice indicated by KL divergence. On the other hand, maxSKSD based methods, especially maxSKSD\_rg, can detect the non-convergence and agrees with the trend shown by the KL method. The above also holds true for other random seeds.
\begin{table}[!h]
\centering
\subfloat[]{
\begin{tabular}{l|llll}
\cline{1-4}
\multicolumn{1}{c|}{Method} & \multicolumn{3}{c}{Random seed} &  \\
                            & 1       & 2       & 3       &  \\ \cline{1-4}
KL                          & 0.004   & 0.004   & 0.004   &  \\ \cline{1-4}
KSD                         & 0.015   & 0.015   & 0.013   &  \\
maxSKSD\_rg                     & 0.004   & 0.008   & 0.008   &  \\
maxSKSD\_g                     & 0.013   & 0.004   & 0.008   &  \\ \cline{1-4}

\end{tabular}
\label{tab:eps}}\\
\subfloat[]{\begin{tabular}{llllll}
\hline
\multicolumn{1}{c|}{Method} & \multicolumn{5}{c}{Metric}                   \\
\multicolumn{1}{l|}{}       & step size &KSD    & maxSKSD\_rg &maxSKSD\_g                    & KL    \\ \hline
KSD                         & 0.015&-0.384 & 902&\multicolumn{1}{l|}{66.0}   & 1.46  \\
maxSKSD\_rg                     & 0.008 &-0.0166&-0.332 & \multicolumn{1}{l|}{-0.018} & 0.257 \\ 
maxSKSD\_g                     & 0.004 &-0.0100&-0.269 & \multicolumn{1}{l|}{-0.079} & 0.201\\
\hline
\end{tabular}\label{tab: divergence value}}
\caption{\textbf{Top}: This table shows the step size chosen by different methods. We can observe that those chosen by maxSKSD based methods and KL method are closer compared to KSD. \textbf{Bottom}: The divergence value at the chosen step size for random seed 2. The row indicates the method used to choose the step size and column indicates the corresponding values. We can observe maxSKSD based methods indeed agree more with KL method where the KL value is around $0.2$ at step size chosen by maxSKSD based methods. On the other hand, KSD failed to detect the non-convergence at step size $0.015$ where KL value is already $1.46$}
\label{tab: SGHMC Convergence}
\end{table}
\label{App: Sampler Convergence}
\section{Model Training}
\subsection{Variance Estimation for Gaussian Toy Example}
This experiment is to demonstrate the mode collapse problem of SVGD at high dimensions and the advantage of the proposed S-SVGD.  
\paragraph{Setup} The target distribution is an standard Gaussian distribution $\mathcal{N}(\bm{0},\bm{I})$. $50$, $100$ and $200$ samples are used for SVGD and S-SVGD. For fair comparison, we use the same RBF kernel with median heuristic for both SVGD and S-SVGD. We run $6000$ update steps to make sure they fully converged before estimating the variance. For S-SVGD, to avoid the over-fitting of sliced matrix $\bm{G}$ to small number of samples, we only update the matrix $\bm{G}$ when the samples after the update are far away from the one used for previous update. The initialized particles are drawn from $\mathcal{N}(\bm{2},\bm{2I})$. The evaluation metric is the averaged estimated variance of the resulting samples. Namely, for a $D$ dimensional target distribution, $\text{Var}_{avg}=\frac{1}{D}\sum_{d=1}^D{\text{Var}(\{\bm{x}_d\}_{n=1}^N)}$.

From figure \ref{fig: S_SVGD_Gaussian}, we observe when the sample number is small, the resulting samples tend to collapse to a point in high dimensions (low variance). On the other hand, the proposed S-SVGD correctly recovers the true target variance regardless of the number of samples and dimensions. This mode collapse behavior of SVGD is directly related to the decrease of the repulsive force at high dimensions (for detailed analysis of this behavior, refer to \citep{zhuo2017message}). To verify this, we plot the \textit{particle averaged repulsive force} (PARF) and the averaged estimated mean of the samples in figure \ref{fig: S_SVGD_Gaussian_Stat}. The PARF for SVGD reduces as the dimension increases whereas S-SVGD stays at a constant level. This is because the kernel and repulsive force of S-SVGD are evaluated on the one-dimensional projections instead of the full input $\bm{x}$. This dimensionality reduction side steps the decrease of the repulsive force at high dimensions regardless of the sample number, thus S-SVGD recovers the correct target variance. 
\begin{figure}
    \centering
\includegraphics[scale=0.16]{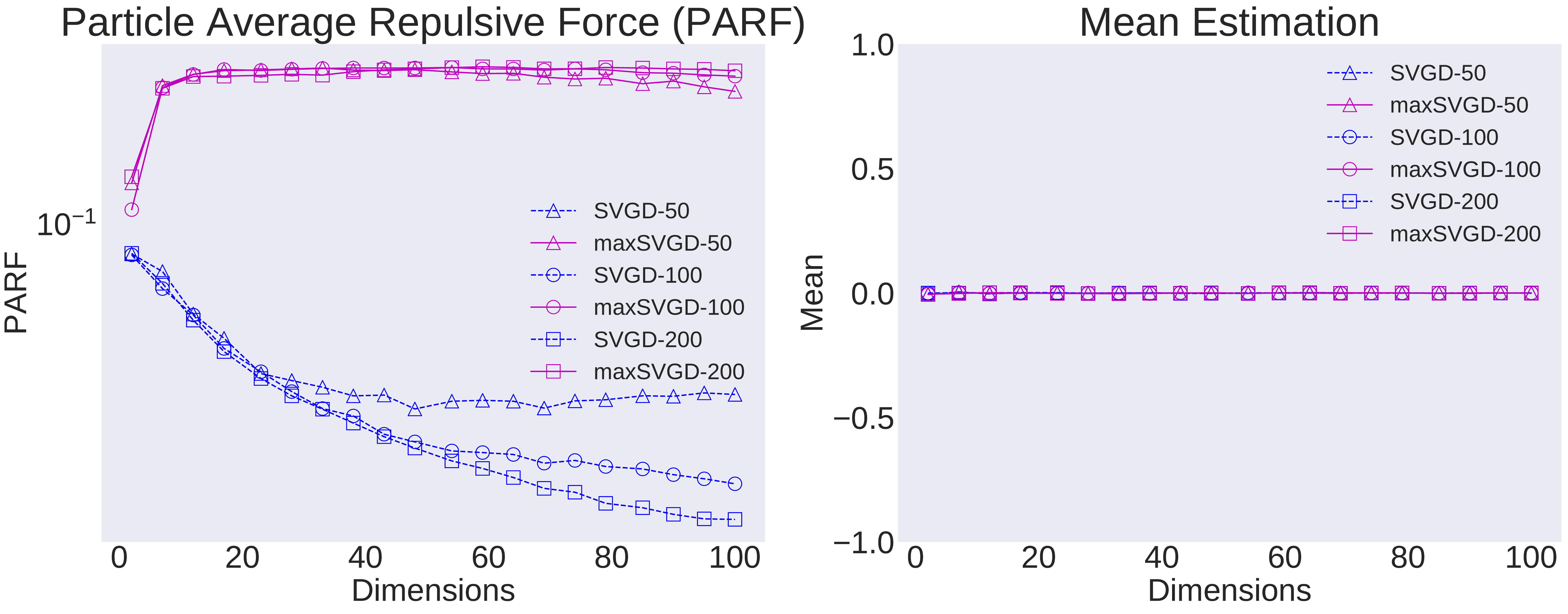}
    \caption{Statistics for SVGD and S-SVGD. The PARF is calculated as $\frac{1}{N}\sum_{n=1}^N||R(\bm{x}_n,\hat{q}_N)||_{\infty}$ \citep{zhuo2017message}, where $R(\bm{x}_n,\hat{q}_N)$ is the repulsive force for $\bm{x}_n$ and $\hat{q}_N$ is the empirical distribution of the samples $\{\bm{x}\}_{n=1}^N$. For SVGD, it is $R(\bm{x},q_y)=\mathbb{E}_{q_y}[\nabla_{\bm{y}}k(\bm{x},\bm{y})]$ and for S-SVGD, the d$^\text{th}$ element of $R(\bm{x},q_y)$ is $R(\bm{x},q_y)_d=\mathbb{E}_{q_y}[g_{d,d}\nabla_{\bm{y}^T\bm{g}_d}k_{g_d}(\bm{x}^T\bm{g}_d,\bm{y}^T\bm{g}_d)]$.}
    \label{fig: S_SVGD_Gaussian_Stat}
\end{figure}
\label{App: S_SVGD_Gaussian}
\subsection{Setup for ICA model training}
We increase the dimensions for ICA from $10$ to $200$ to evaluate their performance in low and high dimensions. We generate the training and test data by using a randomly sampled weight matrix. We use $20000$ training data and $5000$ test data. To make the computation stabler, we follow \citep{grathwohl2020cutting} such that the weight matrix is initialized until its conditional number is smaller than the dimension of the matrix. For LSD, we follow the exact same architecture as the original paper \citep{grathwohl2020cutting}. For KSD, we use the U-statistics with the bandwidth chosen as the median distance (the training for KSD with V-statistic diverges). For maxSKSD, we instead use the V-statistics with $1.5$ times median distance as the bandwidth. We train the ICA model for $15000$ steps using Adam optimizer with $0.001$ learning rate and $\beta_1=0.5,\;\beta_2=0.9$. We use $5$ independent runs and average their results. 
\label{App: ICA Setup}
\subsection{ICA Additional Plots}
\begin{figure}
\centering
    \subfloat[$D=10$]{\includegraphics[scale=0.11]{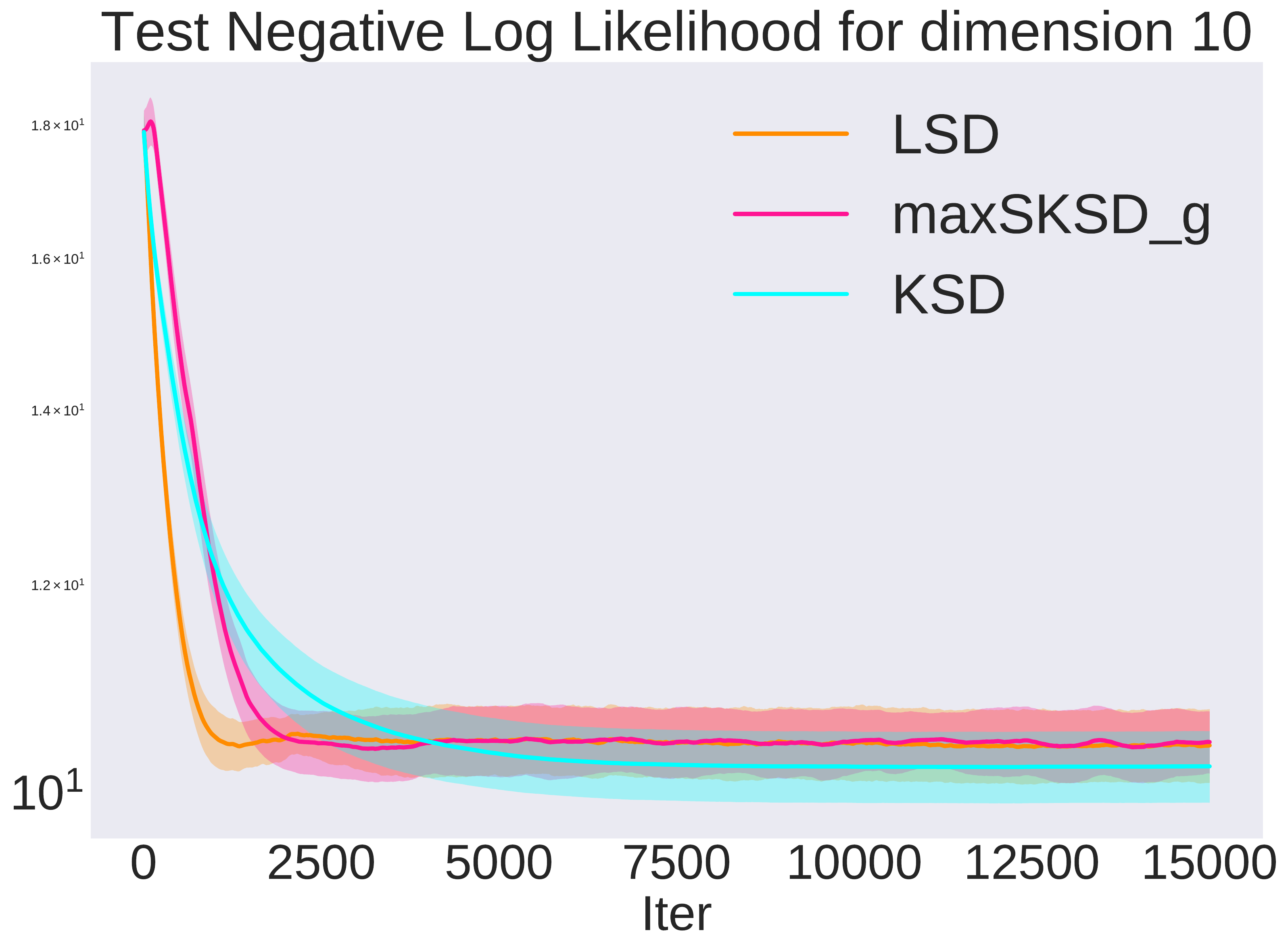}}
    \subfloat[$D=20$]{\includegraphics[scale=0.11]{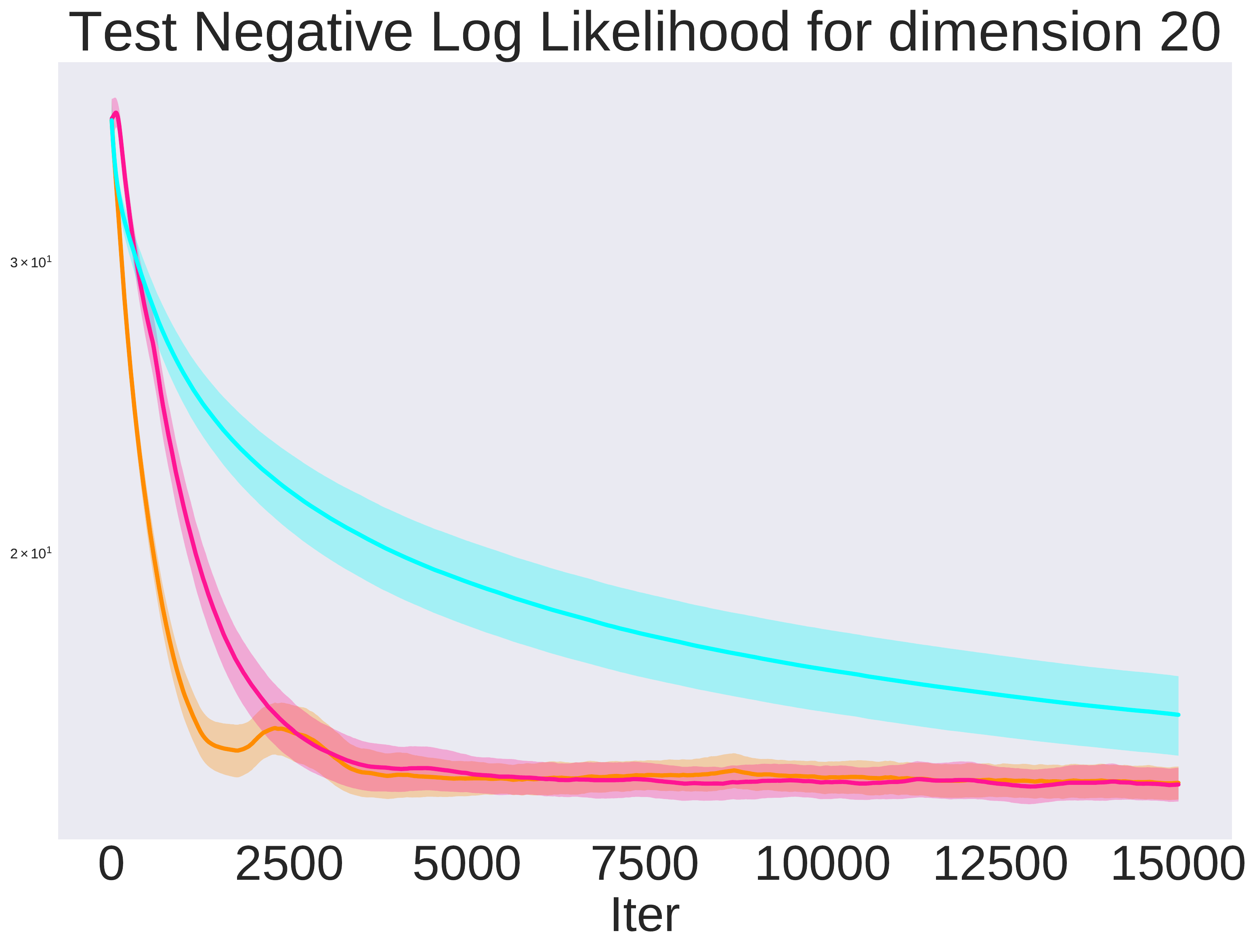}}
    \subfloat[$D=40$]{\includegraphics[scale=0.11]{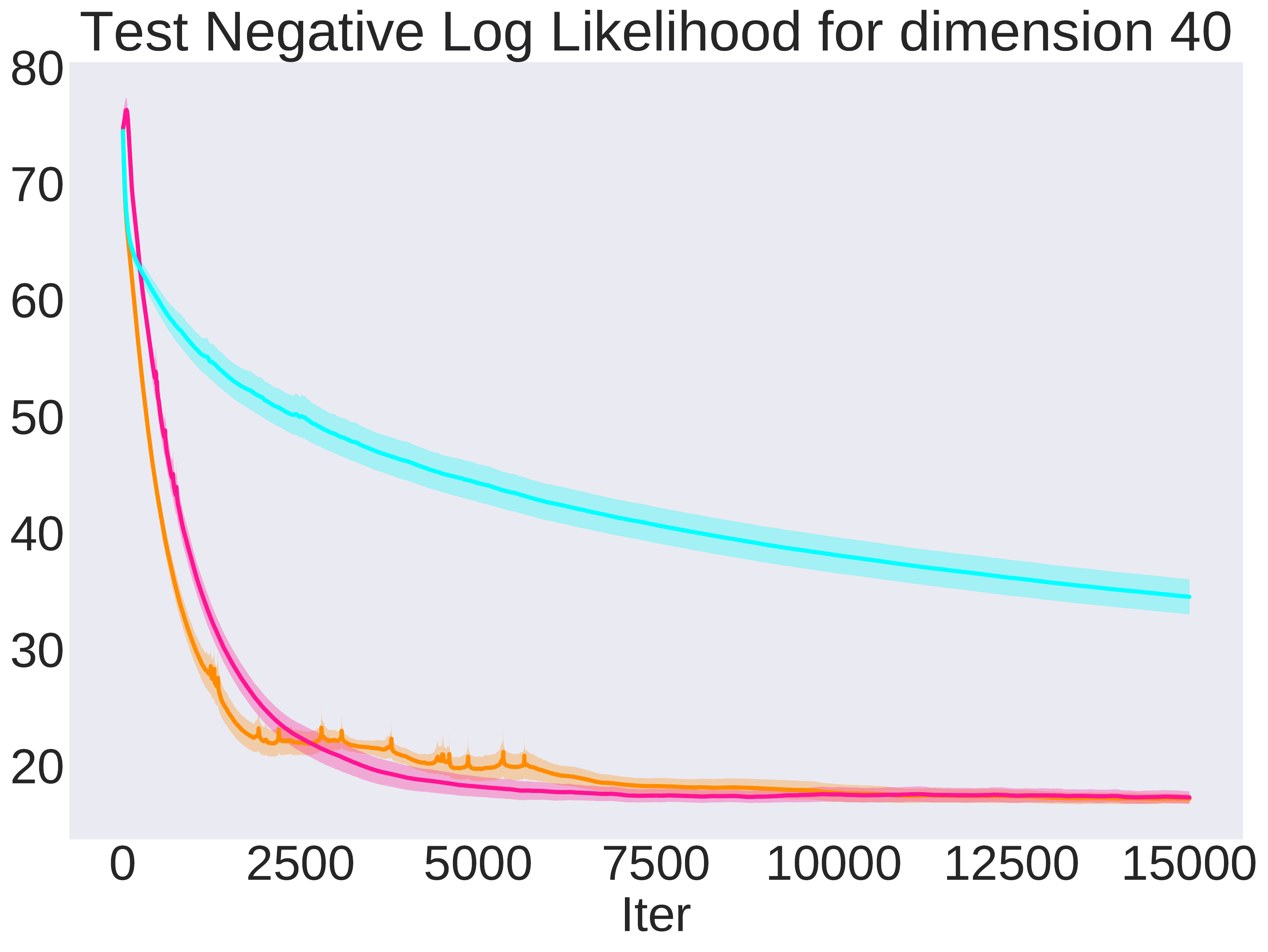}}
    \subfloat[$D=60$]{\includegraphics[scale=0.11]{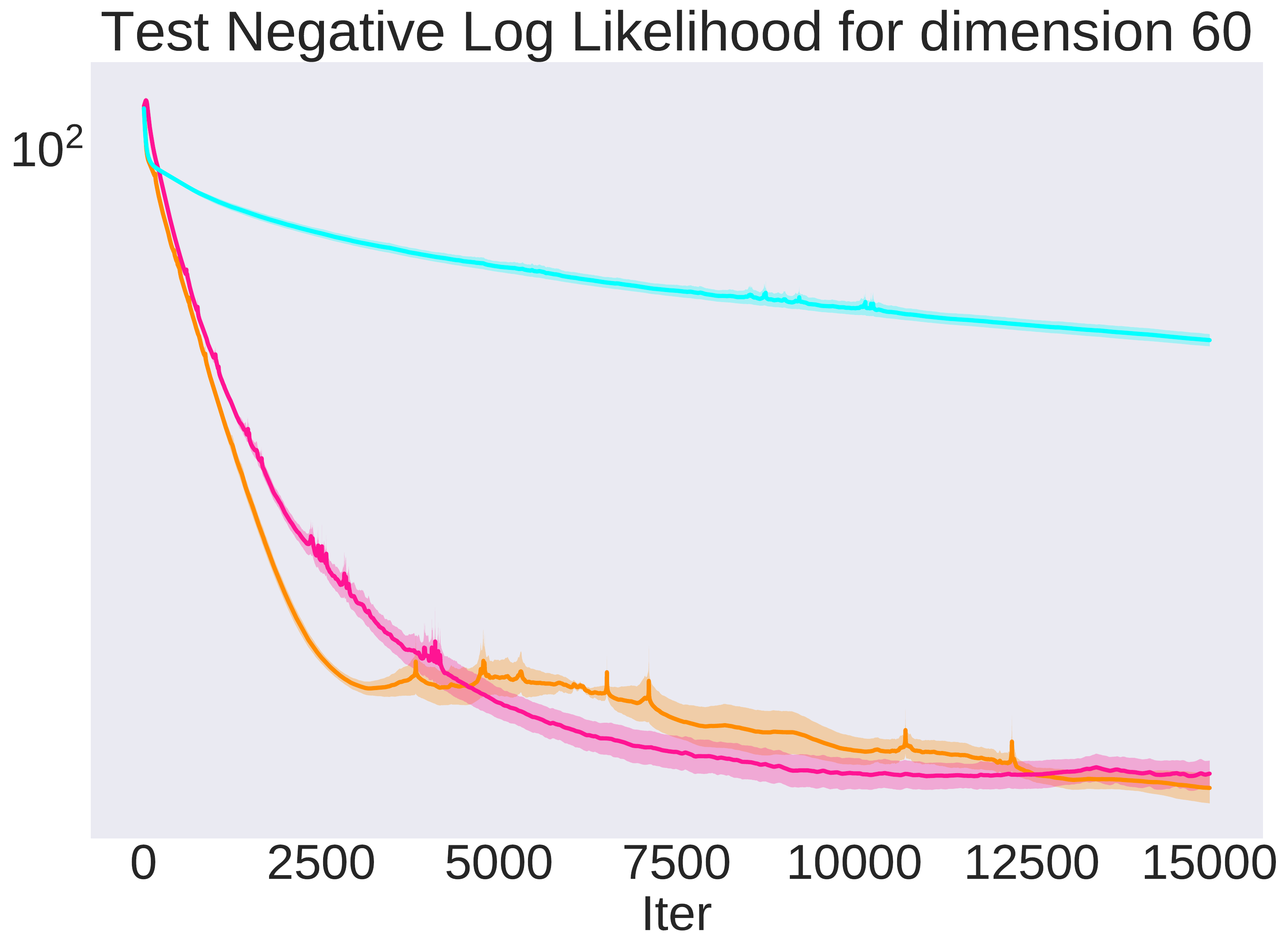}}\\
    \subfloat[$D=80$]{\includegraphics[scale=0.11]{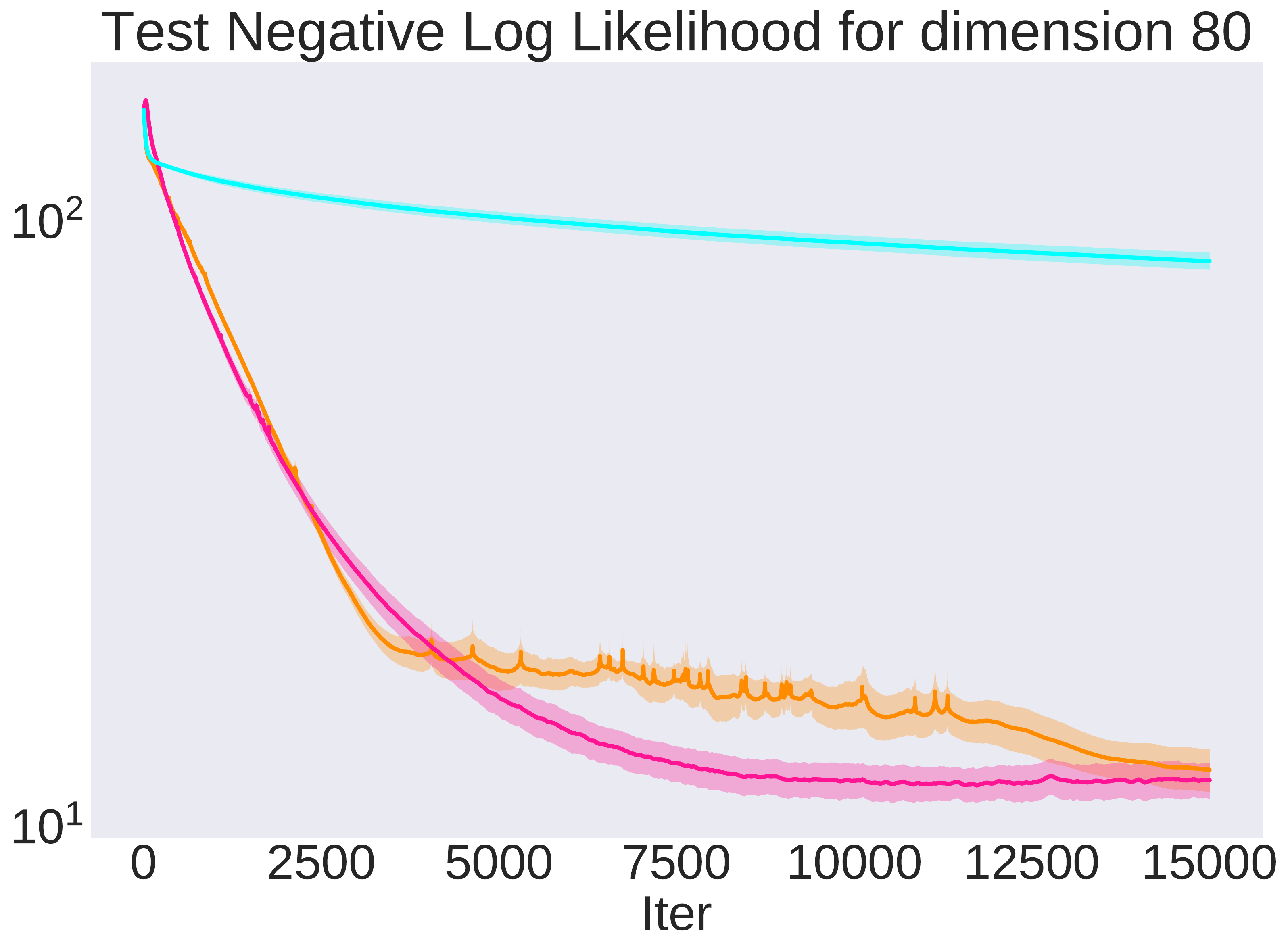}}
    \subfloat[$D=100$]{\includegraphics[scale=0.11]{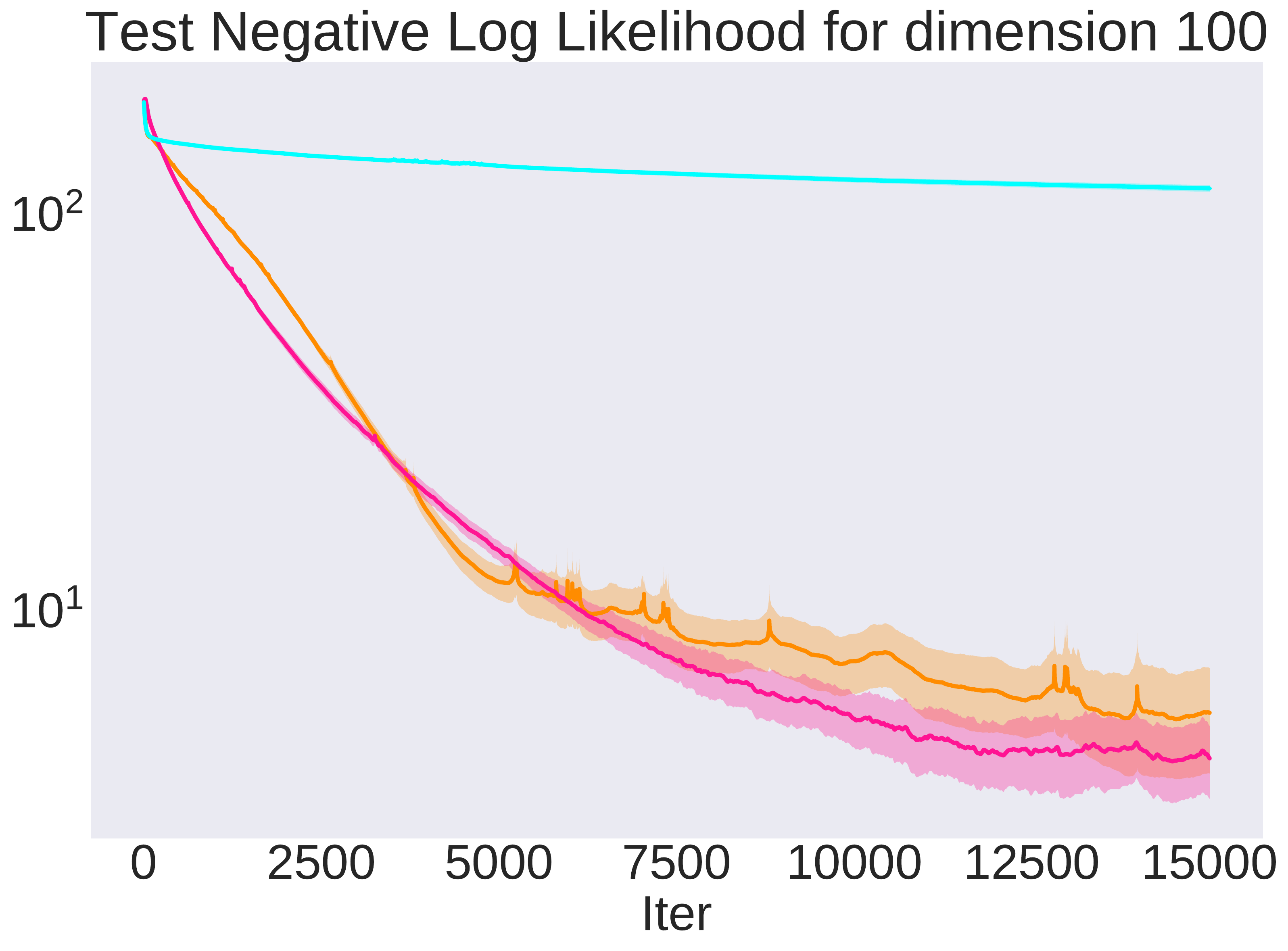}}
    \subfloat[$D=200$]{\includegraphics[scale=0.11]{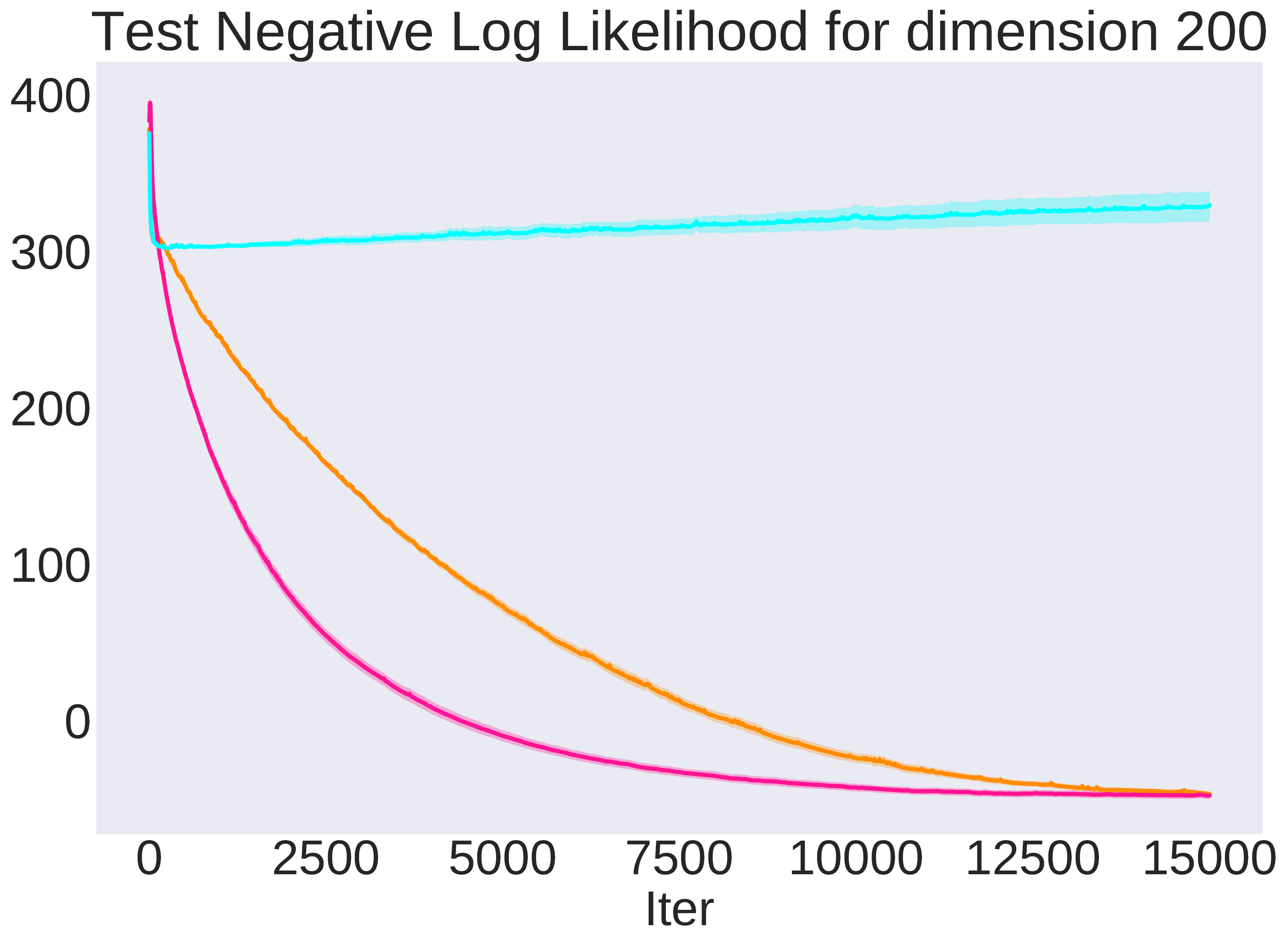}}
    \caption{Training curve of different methods for ICA problems. The y-axis indicates the NLL of test data.}
    \label{fig: ICA Training Curve}
\end{figure}
From the figure \ref{fig: ICA Training Curve}, we observe at low dimensions ($D=10$), LSD converges fastest and KSD is the slowest. However, as the dimension increases, the convergence speed of maxSKSD catches up with LSD and becomes faster after $D=60$, whereas KSD starts to slow down and even diverges at $D=200$. 
\label{App: ICA Additional Plots}
\subsection{Amortized SVGD}
Algorithm \ref{alg: Amortized SVGD} shows the training framework of amortized SVGD. For experiment details, we use fully connected neural network with ReLU activations and 2 hidden layers for encoder and decoder ($[300,200]$ and $[200,300]$ respectively). For decoder output, we use sigmoid activation function and binary cross-entropy for the decoder loss. For the implicit encoder, the input is simply a concatenation of the image and Gaussian noise with the same dimension as the latent space. We also use dropout with probability $0.3$ for each layer of the encoder. For SVGD and S-SVGD, we use $0.1$ for step size and only run 1 update of the latent samples before we update the encoder. The kernel bandwidth is chosen by the median heuristic. We update the sliced matrix $\bm{G}$ for S-SVGD once per each encoder update. $50$ latent samples are used for both encoder and decoder updates. We use Adam optimizer\citep{kingma2014adam} with $0.001$ learning rate and $100$ for batch size.  

For evaluation, the log likelihood is computed using \textit{Hamiltonian annealed importance sampling} (HAIS) \citep{wu2016quantitative}. Specifically, we use $1000$ annealed steps and $10$ leapfrog update per step. We tune the HAIS step size to maintain $0.65$ acceptance rate. 

For imputation, we follow \citep{rezende2014stochastic} to use approximate Gibbs sampler with $D=32$ latent space. Specifically, with missing and observed pixels denoted as $\bm{x}_m$ and $\bm{x}_o$, encoder distribution $q_\phi$ and decoder $p_\theta$, we iteratively applies the following procedure: (1) generate latent samples $\bm{z}\sim q_\phi(\bm{z}|\bm{x}_o,\bm{x}_m)$ (2) reconstruction $\bm{x}^*\sim p_\theta(\bm{x}^*|\bm{z})$ (3) Imputation $\bm{x}_m\leftarrow \bm{x}^*_m$. To compute label entropy and accuracy, $200$ parallel samplers are used for each image with $500$ steps to make sure they fully converged. The imputation label is found by the nearest neighbour method in training data. Label entropy is computed by the its empirical probability and the accuracy is the percentage of the correct ones among all imputed images. 

\begin{algorithm}[H]
\SetKwInOut{Input}{Input}
\SetKwInOut{Output}{Output}

\SetAlgoLined
\Input{Total training step $T$,Adam learning rate $\epsilon_O$, SVGD/S-SVGD step size $\epsilon_S$, latent sample size $N$, encoder network $f_q$, decoder network $f_d$ and decoder loss $\mathcal{L}$}

\For{t $\leq$ T}{
Generate $N$ initial latent samples using encoder $\{\bm{z}_i\}_{i=1}^N=f_q(\bm{x})$\;
Update the samples $\{\bm{z}^*_i\}_{i=1}^N$ based on $\{\bm{z}_i\}_{i=1}^N$ using SVGD or S-SVGD (algorithm \ref{alg: S_SVGD}) with step size $\epsilon_S$\;
Compute the encoder MSE loss between $\{\bm{z}^*_i\}_{i=1}^N$ and $\{\bm{z}_i\}_{i=1}^N$ and update encoder $f_q$ using $\text{Adam}(f_q,\epsilon_O)$\;
Compute decoder loss $\mathcal{L}(\bm{x},\{\bm{z}^*_i\}_{i=1}^N)$ and update decoder using $\text{Adam}(f_d,\epsilon_O)$\;
}
 \caption{Amortized SVGD}
 \label{alg: Amortized SVGD}
\end{algorithm}
\begin{figure}
\centering
    \subfloat[Vanilla VAE]{\includegraphics[scale=0.7]{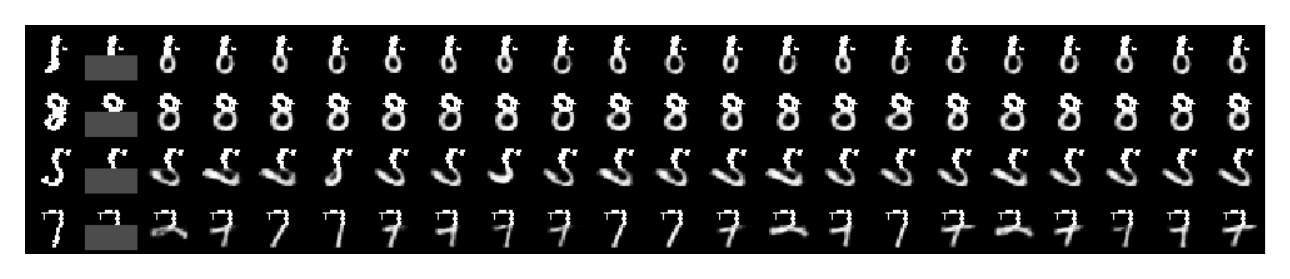}}\\
    \subfloat[Amortized SVGD]{\includegraphics[scale=0.7]{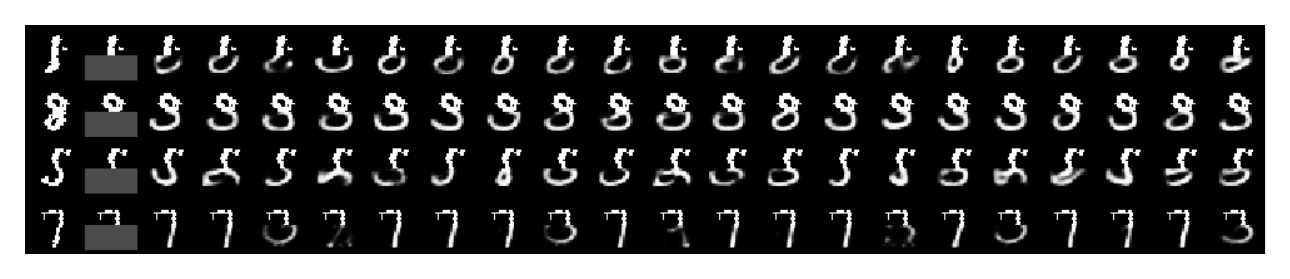}}\\
    \subfloat[Amortized S-SVGD]{\includegraphics[scale=0.7]{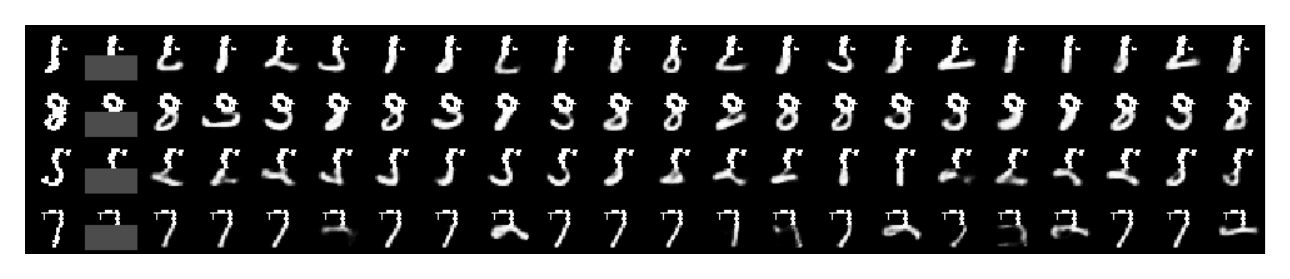}}
    \caption{Imputation images after 500 Gibbs steps. Those images are generated by parallel pseudo-Gibbs sampler. The first column shows the original images. The second column represents the masked images. The rest of the columns are the imputed images. }
    \label{fig: Imputed Images}
\end{figure}
Figure \ref{fig: Imputed Images} shows some of the resulting imputed images after 500 Gibbs steps. We can clearly observe that the S-SVGD generated more diverse images compared to Vanilla VAE (e.g. digit '8' and digit '5'), where it only captures a single mode. Compared to amortized SVGD, the diversity of generated images are similar, but the imputed images of S-SVGD seems to be closer to the original image (e.g. digit '8' and the first digit '5'). This explains the high accuracy value in table \ref{tab: Label entropy}. Although vanilla VAE also generates images that are close to the original one,  it 
may fail to capture the correct mode and get stuck at the wrong one (e.g. first digit '5'). This explains the slightly worse accuracy compared to amortized S-SVGD.
\label{App: Amortized_SVGD_Setup}
\subsection{Bayesian Neural Network Regression}
We also compare our proposed S-SVGD algorithm with the baseline SVGD in high dimensional Bayesian neural network inference. We follow the same settings in \citep{liu2016stein} to use a fully connected one-hidden-layer neural network with ReLU activation and 50 hidden units. The dataset are randomly split into $90\%$ training and $10\%$ test data. Batch size $100$ is used for all data sets. Each results are averaged over $15$ random trials, except for Protein where $5$ trails are conducted. AdaGrad is used for both SVGD and S-SVGD. For SVGD, the bandwidth is selected in the same way as \citep{liu2016stein}. For S-SVGD, we use the same way to select the bandwidth except we multiply a coefficient $0.15$ in front of the bandwidth. $50$ samples are used for both SVGD and S-SVGD. We initialize the particles to be closed to each other. For small datasets like \textit{Boston Housing}, \textit{Yacht} and \textit{Energy}, we apply a small coefficient for the initial repulsive force of S-SVGD, and it gradually increases to $1$ after $500$, $1000$ and $500$ epochs respectively. This is to avoid the over-dominance of the repulsive force at the beginning. For other datasets, we do not tune the repulsive force. For \textit{Boston Housing}, \textit{Concrete}, and \textit{Energy}, we train the network for $2000$ epochs. We use $500$ and $50$ epochs for \textit{Wine} and \textit{Protein} respectively. For the rest of the data set, we use $200$ epochs.

We evaluate the performance through the log likelihood and root mean squared error (RMSE) of the test set, together with the particle-sum distance $\sum_{1\leq i< j\leq N}{dist(\bm{x}_i,\bm{x}_j)}$ to examine the spread of the resulting particles. Table \ref{Tab: BNN LL} shows the performance of BNN trained using SVGD and S-SVGD on 9 UCI data sets. We can clearly observe S-SVGD outperforms SVGD on $7$ out of $9$ data sets. From the particle-sum distance, the resulting particles from S-SVGD are more spread out than SVGD to prevent mode collapse. This behavior can indeed bring benefits especially when dealing with small data set where uncertainty quantification is important. To be specific, SVGD achieves better result only on the large \textit{Protein} data set where the epistemic uncertainty is low compared to small data set. Therefore, the mode collapse of SVGD does not affect the performance too much. This can be partially verified by examining other smaller datasets. \textit{Boston Housing}, \textit{Concrete}, \textit{Energy} and \textit{Yacht} are very small data sets with quite noisy features. Thus, S-SVGD significantly outperforms SVGD on those datasets due to its better uncertainty estimation. For the remaining data set, e.g. \textit{Combined}, \textit{Naval} and \textit{kin8nm}, their data set sizes are between the aforementioned small set and \textit{Protein}. Thus, S-SVGD still achieves better results but the difference is less significant. One exception is \textit{Wine}, a small data set, where S-SVGD has similar performance as SVGD. This is because \textit{Wine} has relatively easy prediction targets.  

\begin{table}[]
\caption{BNN results on UCI regression benchmarks, comparing SVGD and S-SVGD. See main text for details.}
\resizebox{\columnwidth}{!}{
\begin{tabular}{l|ll|ll|ll}
\hline
\multicolumn{1}{c|}{\multirow{2}{*}{Dataset}} & \multicolumn{2}{c|}{RMSE}            & \multicolumn{2}{c|}{test LL}                & \multicolumn{2}{c}{Dist}                        \\
\multicolumn{1}{c|}{}                         & SVGD             & S-SVGD            & SVGD               & S-SVGD            & SVGD            & S-SVGD                         \\ \hline
Boston                                        & $2.937\pm 0.173$ & {$\bm{2.87\pm0.163}$}    & $-2.533 \pm 0.092$ & $\bm{-2.507\pm0.086}$  & $23272\pm 986$  & $49550\pm 6250$                \\
Concrete                                      & $5.189\pm 0.115$ & $\bm{4.880\pm 0.082}$  & $-3.076\pm 0.024$  & $\bm{-3.004\pm 0.023}$ & $24650\pm 1367$ & $62680\pm 1090$                \\
Combined                                      & $3.979\pm 0.040$ & $\bm{3.914 \pm 0.041}$ & $-2.802\pm 0.010$   & $\bm{-2.786 \pm 0.010}$ & $7148\pm 245$   & $33090\pm 430$                 \\
Naval                                         & $0.0030\pm 0$    & $\bm{0.0029\pm 0}$     & $4.368\pm 0.014$   & $\bm{4.411\pm 0.010}$  & $61838\pm 2450$ & $231600\pm 2980$               \\
Wine                                          & $0.607\pm 0.009$ & $\bm{0.603 \pm 0.009}$ & $-0.924\pm 0.015$  & $\bm{-0.914\pm 0.015}$ & $12534\pm 982$  & $35280\pm 2470$                \\
Energy                                        & $1.353\pm 0.049$ & $\bm{1.132\pm 0.048}$  & $-1.736\pm 0.040$  & $\bm{-1.540\pm 0.044}$  & $16476\pm 719$  & $50850\pm 1570$ \\
kin8nm                                        & $0.082\pm 0.001$ & $\bm{0.079\pm 0}$      & $1.084\pm 0.012$   & $\bm{1.104\pm 0.006}$  & $55715\pm 2276$ & $117700\pm 902$                \\
Yacht                                         & $0.714\pm 0.078$ & $\bm{0.613\pm 0.064}$  & $-1.277\pm 0.155$  & $\bm{-0.999\pm 0.087}$ & $15530\pm 1079$ & $47290\pm 2100$                \\
Protein                                       & $\bm{4.543\pm 0.010}$  & $4.587\pm 0.009$   & $\bm{-2.932\pm 0.003}$  & $-2.942\pm 0.002$ & $62370\pm 2143$ & $102600\pm 2335$               \\ \hline
\end{tabular}}
\label{Tab: BNN LL}
\end{table}
\label{App: Bayesian NN}
\end{document}